\documentclass{article} 
\usepackage{times}
\usepackage[accepted]{icml2024}

\usepackage{amsmath,amsfonts,bm}
\usepackage{graphicx}

\newcommand{\mcs}{\mathcal{S}}
\newcommand{\mca}{\mathcal{A}}

\usepackage[utf8]{inputenc} 
\usepackage[T1]{fontenc}    
\usepackage{hyperref}       
\usepackage{url}            
\usepackage{booktabs}       
\usepackage{amsfonts}       
\usepackage{nicefrac}       
\usepackage{microtype}      
\usepackage{xcolor}         
\usepackage{amsmath,amsfonts}
\usepackage{algorithmic}
\usepackage{array}
\usepackage[caption=false,font=normalsize,labelfont=sf,textfont=sf]{subfig}
\usepackage{textcomp}
\usepackage{stfloats}
\usepackage{url}
\usepackage{verbatim}
\usepackage{graphicx}
\usepackage{hyperref}
\usepackage{cleveref}

\usepackage{microtype}      
\usepackage{xcolor}         
\usepackage{algorithm}
\usepackage{mathtools}
\usepackage{amsthm}
\allowdisplaybreaks
\usepackage{amssymb}
\usepackage{mathtools} 
\usepackage{multirow}
\usepackage{multicol}

\newtheorem{assumption}{\textbf{Assumption}}

\newtheorem{lemma}{\textbf{Lemma}}
\newtheorem{theorem}{\textbf{Theorem}}

\newtheorem{remark}{\textbf{Remark}}

\newtheorem{proposition}{\textbf{Proposition}}
\newcommand{\nn}{\nonumber}
\newcommand{\D}{\text{KL}}

\newcommand{\sa}{s_{t+1},a_{t+1}}

\newcommand{\thebar}{\bar{\theta}^*_t}
\newcommand{\thebara}{\bar{\theta}^*_{t+1}}
\newcommand{\thestar}{ \theta^*_t}
\newcommand{\thestara}{\theta^*_{t+1}}
\newcommand{\E}{\mathbb{E}}

\newcommand{\gradj}{\nabla J(\omega_t)}
\newcommand{\myineq}[2]{\overset{(#2)}{#1}}

\newcommand{\twonorm}[1]{ \left\| #1 \right\|_2 }
\newcommand{\ftwonorm}[1]{ \left [\left\| #1 \right\|_2\right] }
\newcommand{\twonormsq}[1]{ \left\| #1 \right\|_2^2 }
\newcommand{\ftwonormsq}[1]{\left [ \left\| #1 \right\|_2^2\right ] }
\newcommand{\myceil}[1]{ \left\lceil #1 \right\rceil }
\newcommand{\tvnorm}[1]{ \left\| #1 \right\|_{\mathcal{TV}} }
\newcommand{\mynorm}[1]{ \left\| #1 \right\| }
 
\newcommand{\lbrac}[1]{ \left| #1 \right| }
\newcommand{\lnorm}[1]{ \left| #1 \right| }
\newcommand{\lnormsq}[1]{ \left( #1 \right)^2 }
\newcommand{\flnormsq}[1]{\left [ \left( #1 \right)^2\right ] }
\newcommand{\varbrac}[1]{ \left\{ #1 \right\} }
\newcommand{\vbrac}[1]{ \left\langle #1 \right\rangle }
\newcommand{\fvbrac}[1]{ \left[\left\langle #1 \right\rangle\right] }
\newcommand{\brac}[1]{ \left( #1 \right) }
\newcommand{\Fbrac}[1]{ \left [ #1 \right] }
\newcommand*{\eqenv}[1]{\begin{alignb} #1 \end{alignb} }
\newcommand*{\eqenvnonum}[1]{\begin{align*} #1 \end{align*} }

\newcommand{\Prob }{\mathbb{P} }
\newcommand{\myarrow}[1]{\xrightarrow{#1}}

\newcommand{\tit }{_{t+1} }

\usepackage{color}

\def\eqref#1{equation~\ref{#1}}

\def\1{\bm{1}}

\DeclareMathAlphabet{\mathsfit}{\encodingdefault}{\sfdefault}{m}{sl}
\SetMathAlphabet{\mathsfit}{bold}{\encodingdefault}{\sfdefault}{bx}{n}

\usepackage{hyperref}
\usepackage{url}
\NewDocumentEnvironment{alignb}{b}{
  \begin{align*}
  \refstepcounter{equation} #1 \tag{\theequation}
  \end{align*}
}

\icmltitlerunning{Non-Asymptotic Analysis for Single-Loop \\(Natural) Actor-Critic with Compatible Function Approximation}

\usepackage[toc,page,header]{appendix}
\usepackage{minitoc}

\icmltitlerunning{Non-asymptotic analysis for single-loop AC/NAC with compatible function approximation  }
\begin{document}

\doparttoc 
\faketableofcontents
\twocolumn[
\icmltitle{Non-Asymptotic Analysis for Single-Loop \\(Natural) Actor-Critic with Compatible Function Approximation}

\icmlsetsymbol{equal}{*}

\begin{icmlauthorlist}
\icmlauthor{Yudan Wang}{ub}
\icmlauthor{Yue Wang}{ucf}
\icmlauthor{Yi Zhou}{utah}
\icmlauthor{Shaofeng Zou}{ub,xxx}
\end{icmlauthorlist}

\icmlaffiliation{ub}{Electrical Engineering, University of Buffalo}
\icmlaffiliation{ucf}{Electrical and Computer Engineering, University of Central Florida}
\icmlaffiliation{utah}{Electrical and Computer Engineering, University of Utah}
\icmlaffiliation{xxx}{Computer Science \& Engineering, University at Buffalo}

\icmlcorrespondingauthor{Shaofeng Zou}{szou3@buffalo.edu}

\icmlkeywords{Machine Learning, ICML}

\vskip 0.3in
]
\printAffiliationsAndNotice{}
\begin{abstract}
Actor-critic (AC) is a powerful method for learning an optimal policy in reinforcement learning, where the critic uses algorithms, e.g., temporal difference (TD) learning with function approximation, to evaluate the current policy and the actor updates the policy along an approximate gradient direction using information from the critic. This paper provides the \textit{tightest} non-asymptotic convergence bounds for both the AC and natural AC (NAC) algorithms. Specifically, existing studies show that AC converges to an $\epsilon+\varepsilon_{\text{critic}}$ neighborhood of stationary points with the best known sample complexity of $\mathcal{O}(\epsilon^{-2})$ (up to a log factor), and NAC converges to an $\epsilon+\varepsilon_{\text{critic}}+\sqrt{\varepsilon_{\text{actor}}}$ neighborhood of the global optimum with the best known sample complexity of $\mathcal{O}(\epsilon^{-3})$, where $\varepsilon_{\text{critic}}$ is the approximation error of the critic and $\varepsilon_{\text{actor}}$ is the approximation error induced by the insufficient expressive power of the parameterized policy class.  This paper analyzes the convergence of both AC and NAC algorithms with compatible function approximation. Our analysis eliminates the term $\varepsilon_{\text{critic}}$ from the error bounds while still achieving the best known sample complexities. Moreover, we focus on the challenging single-loop setting with a single Markovian sample trajectory. Our major technical novelty lies in analyzing the stochastic bias due to policy-dependent and time-varying compatible function approximation in the critic, and handling the non-ergodicity of the MDP due to the single Markovian sample trajectory. Numerical results are also provided in the appendix.
    
\end{abstract}

\section{Introduction}\label{sec:intro}

Actor-Critic (AC) \citep{Barto1983,konda2003onactor} is a reinforcement learning algorithm that combines the advantages of actor-only methods and critic-only methods by alternatively performing policy gradient update (actor) and action-value function estimation (critic) in an online fashion. Specifically, the critic uses a parameterized function to estimate the value function of the current policy, e.g., temporal difference (TD) \citep{sutton1988learning} and Q-learning \citep{watkins1992q}. Then the actor updates the policy along an approximate gradient direction based on the estimate from the critic using approaches such as policy gradient \citep{sutton1999policy} and natural policy gradient \citep{kakade2001natural}. In contrast to critic-only methods, AC methods, which are gradient based, usually have desirable convergence properties when combined with the approach of function approximation. However, critic-only methods may not converge or even diverge when applied together with function approximation \citep{baird1995residual,gordon1996chattering}. Moreover, AC methods also enjoy a reduced variance due to the  critic, and thus their convergence is typically more stable and faster than actor only methods.

While the asymptotic convergence for AC and NAC has been well understood in the literature, e.g., \citep{bhatnagar2009natural,kakade2001natural,konda2003onactor,suttle2023beyond}, its non-asymptotic convergence analysis has been largely open until very recently. The non-asymptotic analysis is of great practical importance as it answers the questions that how many samples are needed and how to appropriately choose the different learning rates for the actor and the critic. Existing studies show that AC converges to an $\epsilon+\varepsilon_{\text{critic}}$ neighborhood of stationary points with the best known sample complexity of $\mathcal{O}\brac{\epsilon^{-2}}$ \citep{chen2021closing,olshevsky2022small,xu2020improving}, and NAC converges to an $\epsilon+\varepsilon_{\text{critic}}+\sqrt{\varepsilon_{\text{actor}}}$ neighborhood of the global optimum with the best known sample complexity of $\mathcal{O}(\epsilon^{-3})$ \citep{chen2022finite,xu2020improving}, where $\varepsilon_{\text{critic}}$ is the approximation error of the critic and $\varepsilon_{\text{actor}}$ is the approximation error induced by the insufficient expressive power of the parameterized policy class. In this paper, when presenting sample complexity, we omit the log factors.
In these studies, the critic employs a fixed class of parameterized functions (typically linear function approximation with fixed feature), which may not satisfy the compatible condition \citep{sutton1999policy} (see \Cref{sec:objective} for details). This will result in a non-diminishing bias in the policy gradient estimate, and therefore, an additional error term $\varepsilon_{\text{critic}}$ is incurred in the overall error bound. Several works \citep{cayci2022finite,wang2019neural} propose to use overparameterized neural networks in the critic to mitigate this issue, where $\varepsilon_{\text{critic}}$ diminishes as the network size increases. However, the convergence of the critic requires stringent conditions that are hard to verify \citep{cayci2022finite,wang2019neural}, and large neural network introduces expensive computational and memory costs. Actually, if the critic employs the approach of compatible function approximation, which is \textit{linear}, then $\varepsilon_{\text{critic}}$ vanishes without introducing additional computational and memory costs \citep{sutton1999policy} (see details in \Cref{sec:objective}). 
Moreover, for NAC applied with fixed function approximation in the critic, one needs to explicitly estimate the Fisher information matrix and compute its inverse, which will be computationally and memory expensive. Another advantage of compatible function approximation when applied with NAC is that the inverse of the Fisher information in the natural gradient will cancel out with the policy gradient (see \Cref{prop:npg}), and thus there is no need to estimate the Fisher information matrix anymore. 
\subsection{Challenges and Contributions}
\begin{table*}
  \caption{{Comparison of sample complexity of AC}}
  \label{table1}
  \centering
  \begin{tabular}{c|c|c|c|l}
    \hline
   Reference     & Single-loop    & Sample size& Error& Comments \\
    \hline
    \citep{wang2019neural}&$\times$ & $\mathcal{O}\brac{\epsilon^{-6}}$& $\epsilon+ \varepsilon_{\text{critic}}$& Critic: neural\\
     \cline{1-5}  
    \citep{zhou2022single}&$\surd$ &  $\mathcal{O}\brac{\epsilon^{-1}}$& $\epsilon$ &LQR \\
\cline{1-3}  
    \citep{chen2023global}&$\surd$ &  $\mathcal{O}\brac{\epsilon^{-2.5}}$& & \\
    \cline{1-5} 
    \citep{zhang2020provably}&$\surd$ & Asymptotic& &\\
 \cline{1-3}

   \citep{qiu2021finite}& $\times$& $\mathcal{O}\brac{\epsilon^{-4}}$ 
 &  &Actor: \\
  \cline{1-3}

   \citep{kumar2023sample} & $\times$ &  $\mathcal{O}(\epsilon^{-3})$& & non-linear, smooth \\
  \cline{1-3}
\citep{kumar2023sample}
 \citep{xu2020non}&$\times$ & $\mathcal{O}\brac{\epsilon^{-2.5}}$& 
 $\epsilon+ \varepsilon_{\text{critic}} $
 & Critic: linear
 \\
  \cline{1-3}
     \citep{xu2020improving,suttle2023beyond}& $\times$ & $\mathcal{O}\brac{\epsilon^{-2}}$ 
    & &  function approx.\\
    \cline{1-3}
 \citep{barakat2022analysis}
 \citep{wu2020finite}&$\surd$ & $\mathcal{O}\brac{\epsilon^{-2.5}}$&  \\
  \cline{1-3}
\citep{olshevsky2022small}
 \citep{chen2021closing}&$\surd$ & $\mathcal{O}\brac{\epsilon^{-2}}$&&\\
 \cline{1-4}

Our Work & $\surd$ &  $\mathcal{O}(\epsilon^{-2})$&$\epsilon$\\
    \hline
  \end{tabular}
\end{table*}

Though AC and NAC with compatible function approximation enjoy no approximation error from the critic and no need of estimating the Fisher information matrix (for NAC), their non-asymptotic convergence analyses are much more challenging than the ones with fixed function approximation. To the best of the authors' knowledge, this paper develops the tightest non-asymptotic error bounds for AC and NAC algorithms, and our analyses are for the challenging case of single Markovian sample trajectory. We prove that AC with compatible function approximation converges to an $\epsilon $ stationary point with sample complexity $\mathcal O(\epsilon^{-2})$, and NAC with compatible function approximation converges to an $\epsilon+\sqrt{\varepsilon_{\text{actor}}}$ neighborhood of the globally optimal policy with sample complexity $\mathcal O(\epsilon^{-3})$. Our non-asymptotic error bounds outperform the best known AC and NAC bounds in the literature by a constant $\varepsilon_{\text{critic}}$ and achieve the same sample complexity: $\mathcal O(\epsilon^{-2})$ for AC and $\mathcal O(\epsilon^{-3})$ for NAC (see \Cref{table1,table2}). We note that this constant $\varepsilon_{\text{critic}}$ is due to the approximation error of the function class used by the critic, and does not diminish with time. 

One of the biggest challenges in the analysis is due to the time-varying critic feature function. Specifically, the critic with compatible function approximation employs an $\omega$-dependent linear function class, where $\omega$ is the policy parameter. As the actor updates the policy, the feature function of the critic also changes with $\omega$. Therefore, the critic is using a linear function with time-varying $\omega$-dependent feature to track the value function of the current policy $\pi_\omega$, which is also time varying. This makes the analysis of the tracking error, i.e., the error between the ideal limit of the critic given the current policy and its current estimate, challenging. In this paper, we design a novel approach to explicitly bound this error. The central idea is to construct an auxiliary eligibility trace with fixed feature to approximate the eligibility trace with time-varying feature (in the critic, we use $k$-step TD with compatible function approximation). 


In this paper, we focus on the challenging single-loop setting with a single Markovian sample trajectory. 
Some studies tried to decouple the updates of the actor and the critic using approaches, e.g., nested loop \citep{qiu2021finite,agarwal2021theory,chen2022finite,xu2020improving,suttle2023beyond}, and to further develop the non-asymptotic analysis.  Specifically, after the actor updates the policy, then the policy is fixed and the critic starts an inner loop to iterate sufficient number of steps until it gets a perfect evaluation of the current policy. This decoupling approach makes it easier to analyze as there is no need to analyze the interaction between the actor and the critic. However, this decoupling approach does not enjoy benefits from the two time-scale structure in the original AC and NAC algorithms \citep{konda2003onactor,bhatnagar2009natural}, e.g., algorithmic simplicity and statistical efficiency, and techniques therein cannot be generalized to analyze the single-loop single-trajectory two time-scale AC and NAC algorithms. Moreover, analyses therein require some kind of i.i.d.\ assumptions or require trajectories starting from any arbitrary state, which might be difficult to guarantee in practice. 
To develop the tightest bound, we develop a novel approach that bounds the tracking error as a function of the policy gradient norm (for AC) and the optimality gap (for NAC). We also note that our analysis for NAC does not need the smoothness assumption on the parameterized policy, which is typically required in existing NAC and AC analyses \citep{chen2021closing,olshevsky2022small}.

\subsection{Related Work}
In this section, we review recent relevant works on non-asymptotic analyses on reinforcement learning algorithms with function approximation.  We provide a detailed comparison between our results and existing studies on AC and NAC in \Cref{table1,table2}. The "Sample complexity" in the table is the one needed to guarantee the gradient norm/optimality gap less than or equal to the "Error".

\begin{table*}
  \caption{{Comparison of sample complexity of NAC}}
  \label{table2}
  \centering
  \begin{tabular}{c|c|c|c|l}
    \hline
   Reference     & Single-loop    & Sample size& Error& Comments \\
    \hline
    \citep{khodadadian2022finite}&$\surd$ & $\mathcal{O}\brac{\epsilon^{-6}}$& &\\
     \cline{1-3}
    \citep{khodadadian2021finite}&$\times$ & $\mathcal{O}\brac{\epsilon^{-3}}$& $\epsilon$ 
     & Tabular case  \\
     \cline{1-5}
    \citep{wang2019neural}&$\times$ & $\mathcal{O}\brac{\epsilon^{-6}}$&$\epsilon+ \varepsilon_{\text{critic}}$\\
 \cline{1-3}
    \citep{cayci2022finite}&$\times$ & $\mathcal{O}\brac{\epsilon^{-3}}$&$+\sqrt{\varepsilon_{\text{actor}}} $& Critic: neural\\
 \cline{1-5}
    \citep{agarwal2021theory}
& $\times$&$\mathcal{O}\brac{\epsilon^{-6}}$ 
    &  & Actor:  \\
    \cline{1-3}
    \citep{xu2020improving}& $\times$ & $\mathcal{O}\brac{\epsilon^{-3}}$ 
    &  $\epsilon+ \varepsilon_{\text{critic}} $
    &  non-linear, smooth \\
    \cline{1-3}
 \citep{xu2020non}&$\times$ & $\mathcal{O}\brac{\epsilon^{-4}}$& 
 $+\sqrt{\varepsilon_{\text{actor}}} $
 & Critic: linear
 \\
 \cline{1-3}
 \citep{chen2022finite}    & $\times$ & $\mathcal{O}\brac{\epsilon^{-3}}$
 &&    function approx.\\
  \cline{1-4}
Our Work & $\surd$ &  $\mathcal{O}(\epsilon^{-3})$&$\epsilon+\sqrt{\varepsilon_{\text{actor}}}$\\
    \hline
  \end{tabular}
\end{table*}

\textbf{Actor-critic analyses.}
We list recent works on non-asymptotic analyses for AC in \Cref{table1}. Based on whether the updates of actor and critic are decoupled, the results can be grouped into "single-loop" and "nested-loop/decoupling" approaches. For a general MDP, the best known sample complexity for both single-loop and nested-loop approaches is $\mathcal O(\epsilon^{-2})$ \citep{chen2021closing,olshevsky2022small,xu2020improving,suttle2023beyond}. The only exception is \citep{zhou2022single}, which is due to  the special structure of the LQR problem. These studies all use a fixed function class in the critic, and therefore,  the convergence error consists of a non-diminishing constant term of $\varepsilon_{\text{critic}}$. In this paper, we analyze the AC with compatible function approximation, and we obtain a strictly tighter error bound without $\varepsilon_{\text{critic}}$. Our analysis is also much more challenging than the ones in the literature, which is mainly due to that the function class in the critic varies with the policy in the actor.

\textbf{Natural actor-critic analyses.}
We list recent works on non-asymptotic analyses for NAC in \Cref{table2}. The best sample complexity for single-loop NAC is $\mathcal{O}\brac{\epsilon^{-6}}$ and it is for the tabular case \citep{khodadadian2022finite}, whereas the best sample complexity for nested-loop/decoupling NAC is $\mathcal{O}\brac{\epsilon^{-3}}$ with an error of $\epsilon+\varepsilon_{\text{critic}}+\sqrt{\varepsilon_{\text{actor}}}$ \citep{chen2022finite,xu2020improving}. There exists a gap of $\mathcal{O}\brac{\epsilon^{-3}}$ between these two approaches, which is mainly due to the challenge in bounding the tracking error for NAC in the single-loop setting.  In this paper, we close this gap and show that NAC in the single-loop setting can also achieve the sample complexity of $\mathcal{O}(\epsilon^{-3})$, and more importantly with a reduced error of $\epsilon+\sqrt{\varepsilon_{\text{actor}}}$. 

\textbf{Actor/critic only analyses.}
Non-asymptotic analyses for critic only methods have been extensively studied recently, e.g., TD \citep{srikant2019finite,lakshminarayanan2018linear,bhandari2018finite,cai2019neural,sun2019finite,xu2020finite}, SARSA \citep{zou2019finite}, gradient TD (GTD) method \citep{dalal2018finite,xu2019two,wang2021non,wang2017finite,liu2015finite,gupta2019finite,kaledin2020finite,ma2020variance,ma2021greedy,wang2020finite}. There are also non-asymptotic analyses for actor only method, e.g., \citep{bhandari2021linear,bhandari2019global,agarwal2021theory,mei2020global,li2021softmax,laroche2021dr,zhang2021global,cen2021fast,zhang2020variational,xiao2022On}. In this paper, we focus on AC and NAC algorithms, where  how the errors in the actor and the critic affects the other  needs to be analyzed.

\section{Preliminaries}\label{sec:objective}
\textbf{Markov Decision Processes}
Consider a general reinforcement learning setting, where an agent interacts with a stochastic environment modeled as a Markov decision process (MDP). An MDP can be represented by a tuple 
$\vbrac{\mathcal{S}, \mathcal{A}, P, R}$, 
where $\mathcal{S}$ denotes the state space, $\mathcal{A}$ denotes the discrete finite action space,
$R(\cdot,\cdot): \mathcal{S}\times \mathcal{A}\to [0, R_{\max}]$ is the reward function. 
The transition kernel $P\brac{\cdot|s, a}$ denotes the distribution of the next state if taking action $a$ at state $s$, $\forall s\in\mcs,a\in\mca$.

A stationary policy $\pi$ maps a state $s\in \mathcal{S}$ to a probability distribution $\pi\brac{\cdot|s}$ over the action space $\mathcal{A}$. 
Then 
the expected long term average reward for a policy $\pi$ is defined as follows:
\vspace{-0.2cm}
\begin{align}
     J\brac{\pi}&= \lim_{N\to \infty }\frac{1}{N}\E\bigg[\sum _{t=0}^{N-1}  R(s_t,a_t)\large|\pi\bigg ]\nn\\&=\E_{s\sim d_\pi, a\sim \pi(\cdot|s)}\Fbrac{ R(s,a)}
,\nn
\end{align}
where we denote by $d_{\pi}$ the stationary distribution 
$$d_{\pi}(s)=\lim_{N\to \infty}\frac{1}{N} \sum _{t=0}^{N-1} \Prob\brac{s_t=s|  \pi}.$$ Denote by $D_\pi= d_\pi \times \pi $  the state-action stationary distribution. 
We rewrite $R_t:= R(s_t,a_t)$.
For a given policy $\pi$ and an initial  state
 $s$, the relative value function is defined as
 \vspace{-0.1cm}
$$V^\pi(s)=\E\bigg[\sum_{t=0}^\infty  R_t -J(\pi)|s_0=s, \pi \bigg] ,\forall s\in \mathcal{S}.$$
Given initial state $s$ and action $a$, the  relative action value function ($Q$ function) for a given policy $\pi$ is defined as 
\begin{align*}
    Q^\pi(s,a)=\E\bigg[\sum_{t=0}^\infty  R_t-J(\pi)|&s_0=s,a_0=a, \pi  \bigg], \\ &\forall (s,a) \in \mathcal{S}\times \mathcal{A}.\nonumber
\end{align*}
The relative advantage function is defined as 
\eqenvnonum{
A^\pi(s,a)=Q^\pi(s,a)-V^\pi(s),\forall (s,a) \in \mathcal{S}\times \mathcal{A}.\nonumber
}
The goal is to find the optimal policy $\pi^*$ that maximizes the long term average reward: 
$
\max_\pi J(\pi).$

\textbf{(Natural) Actor-Critic with Compatible Function Approximation}
Consider a parameterized policy class $\Pi_\omega=\varbrac{\pi_\omega:\omega\in \mathcal{W}}$, where $\mathcal{W}\subseteq \mathbb{R}^d$. Then the problem in \Cref{sec:objective} can be solved by optimizing over the parameter space $\mathcal{W}$. Specifically, the actor updates the policy via the approach of (natural) policy gradient, where the policy gradient is given by \citep{sutton1999policy}
\begin{align}
    \nabla J(\pi)&= \E_{D_{\pi_\omega}}\Fbrac{Q^{\pi_\omega}\brac{s,a}\phi_\omega(s,a)}
,\label{eq:gradient}
\end{align}
where $\phi_\omega (s,a)=\nabla_\omega \log \pi_\omega(a|s)$. We further let $\Phi_\omega$ denote the feature matrix, which is the stack of all feature vectors. Specifically, $\Phi_\omega\in \mathbb{R}^{|\mathcal{S}||\mathcal{A}|\times d}$ and the $(s,a)$-row of $\Phi_\omega$ is $\phi_\omega^\top(s,a)$.
On the other hand, the critic estimates the $Q$ function in \Cref{eq:gradient} via the approach of TD learning, and the $Q$ function is usually parameterized using linear function approximation in the existing literature, i.e., $\mathcal Q=\{Q_\theta(s,a)=\phi(s,a)^\top \theta, \theta\in\Theta\}$ where $\phi$ denotes the feature vector and $\Theta\subseteq\mathbb{R}^d$. However, as summarized in \Cref{table1,table2}, using a fixed $\phi$ introduces an additional 
non-vanishing error term $\varepsilon_{\text{critic}}$ to the gradient estimate. 

To avoid the critic's function approximation error,  \citep{sutton1999policy,konda2003onactor} proposed a smart idea of compatible function approximation, which 
uses the compatible feature vector $\phi_\omega$ that depends on the policy parameter $\omega$. To explain, 
in order to approximate the value function $Q^{\pi_\omega}$ associated with policy $\pi_\omega$, we can set the feature vector as $\phi_\omega (s,a):=\nabla_\omega \log \pi_\omega(a|s)$ and solve for the best linear approximation parameter $\bar\theta^*_\omega$ via the following optimization problem.
\begin{align}\label{eq:thetabar}
    \bar\theta^*_\omega \in \arg \min _{\theta} \E_{D_{\pi_\omega} }\Fbrac{\brac{Q^{\pi_\omega}(s,a)-\phi_\omega^\top(s,a)\theta}^2}.
\end{align}

\begin{proposition}[\cite{sutton1999policy}]\label{prop:pg}
    With compatible function approximation, the policy gradient $\nabla J(\pi_\omega)$ can be rewritten as:
    \eqenv{\label{eq:eq5}
 \nabla J(\pi_\omega) &=\E_{D_{\pi_\omega}}\Fbrac{ \nabla \log \pi_\omega(a|s)Q^{\pi_\omega}(s,a)}\\&=
 \E_{D_{\pi_\omega}}\Fbrac{\phi_\omega(s,a)(\phi^\top_\omega\brac{s,a}\bar\theta_\omega^*)}.
} 
\end{proposition}

This implies that as long as we can solve the finite dimensional problem \Cref{eq:thetabar}, linear function approximation with the compatible feature $\phi_\omega$ and parameter $\bar\theta^*_\omega$ does not induce any function approximation error. 
This approach is referred to as compatible function approximation \citep{sutton1999policy}, i.e., estimating $Q^{\pi_\omega}$ using an $\omega$-dependent linear function class: $\mathcal Q_\omega=\{\phi^\top_\omega(s,a) \theta, \theta\in\Theta\}.$ 
%
%
%
%
To solve \Cref{eq:thetabar} for the compatible function approximation parameter, we use the $k$-step  TD algorithm with compatible feature $\phi_\omega$
 \citep{sutton1999policy}.


The actor can also use the following natural policy gradient to update the policy \citep{kakade2001natural}:
$$
    \widetilde{\nabla} J(\pi_\omega)=F_\omega ^{-1}\nabla J(\pi_\omega), 
$$
where the matrix $F_\omega$ denotes the Fisher information matrix:
$$
    F_\omega=\E_{D_{\pi_\omega}}\Fbrac{ \nabla  \log \pi_\omega(a|s)\brac{\nabla \log \pi_\omega(a|s)}^\top }.
$$

 \begin{proposition}[\cite{peters2008natural}]\label{prop:npg}
 With compatible function approximation, natural policy gradient is reduced to: 
$$
    \widetilde{\nabla} J(\pi_\omega)=\bar \theta^*_\omega. $$
 \end{proposition}

That is, there is no need to estimate the Fisher information matrix and compute its inverse, which is typically computationally expensive.

\section{Main Results}
\begin{algorithm*}[tb]
   \caption{(Natural) Actor-Critic with Compatible Function Approximation}
   \label{alg:AC}
\begin{algorithmic}[1]
   \STATE {\bfseries Initialization:}  $k$, $\eta_0$, $ \omega_0, \pi_0=\pi_{\omega_0},\theta_0,$ $\phi_0=\nabla \log\pi_{0},$
   $s_0,a_0\sim \pi_0(\cdot|s_0),z_0=0$
   \FOR{$t=0,...,T-1$}
   \STATE Observe $R_t$; Sample $s_{t+1}\sim  P(\cdot|s_t,a_t)$;
    $a\tit\sim \pi_t(\cdot|s\tit) $
   \STATE $\phi_t(s,a)=\nabla_\omega \log\pi_{t}(a|s)$ \hfill{\verb+/*Compatible function approximation*/+}
   \STATE \textbf{Critic}: $\delta_t(\theta_t)=R_t-\eta_t+ \phi_t^\top(\sa)\theta_t-\phi_t^\top(s_t,a_t)\theta_t $\hfill{\verb+/*TD error*/+}
   \STATE \qquad\quad $z_t=\sum_{j=t-k}^t \phi_j(s_j,a_j)$\hfill{\verb+/*eligibility trace*/+}
   \STATE  \qquad\quad $\eta_{t+1}=\eta_t+ \gamma_t(R_t-\eta_t)$ \hfill{\verb+/*average reward update*/+}
   \STATE \qquad\quad $\theta_{t+1}=\Pi_{2,B}(\theta_t+\alpha_t \delta_t(\theta_t)z_t)$ \hfill{\verb+/*TD update*/+}
   \STATE \textbf{Option I}: $\omega_{t+1}=\omega_t+\beta_t \phi_t^\top(s_t,a_t)\theta_t \phi_t(s_t,a_t)$ \hfill{\verb+/*Actor update in AC*/+}
   \STATE \textbf{Option II}: $\omega_{t+1}=\omega_t+\beta_t \theta_t $ \hfill{\verb+/*Actor update in NAC*/+}
   \ENDFOR
\end{algorithmic}
\end{algorithm*}

\vspace{0.05cm}
The detailed AC and NAC algorithms with compatible function approximation is summarized in \Cref{alg:AC}.
In the critic update, $\alpha_t$ is the stepsize, and denote by 
$\Pi_{2,B}(v)=\arg\min_{\twonorm{\omega}\leq B}\twonorm{v-\omega} $ for any $v\in\mathbb R^d$ the project operator, and $B$ is the radius. 
Next, we present the non-asymptotic bounds for the AC and NAC with compatible function approximation in \Cref{alg:AC}. 

\begin{assumption}\label{asm:ergodic}
(Uniform Ergodicity) Consider the MDP with policy $\pi_\omega$ and transition kernel $P$, there exists constants $m>0$, and 
$\rho\in (0,1)$ such that 
$$
\sup_{s\in\mcs}\tvnorm{\mathbb{P}\brac{s_t\in \cdot|s_0=s}-D_{\pi_\omega}(\cdot)}\leq m \rho^{t}.
$$
\end{assumption}
\vspace{-0.1cm}
Here $\tvnorm{\cdot}$ denotes the total variation distance between two distributions. 
Assumption \ref{asm:ergodic} is widely used in the literature to handle the Markovian noise, e.g., \citep{srikant2019finite,zou2019finite,bhandari2018finite}. 
We further assume that the $d$ feature functions, $\phi_{\omega,i}, i=1,...,d$, are linearly independent, i.e., the feature matrix $\Phi_\omega$ is full rank when $|\mcs||\mca|\geq d$. This is also commonly used in the literature of analyzing RL algorithms with linear function approximation \citep{srikant2019finite,zou2019finite,bhandari2018finite}. 
\vspace{0.1cm}
    

\subsection{Critic:$k$-step TD}\label{sec:3.1}


Consider the critic update, where the TD method is used to learn the relative value function under the average-reward setting. It is known that the feature function needs to satisfy certain condition (Assumption 2 in \citep{tsitsiklis1999average}) so that the limit of the TD method is unique.   In the following proposition, we show that compatible function approximation automatically satisfy the assumption needed in \citep{tsitsiklis1999average}, and therefore guarantees the convergence of the critic without the need of any additional assumptions. 

\begin{proposition}\label{prop:e}
For any $\omega\in \mathcal{W}$ and $\theta \in \Theta$, 
$\Phi_\omega \theta \neq \mathbf{e},$
where  $\mathbf e\in \mathbb R^d$ is an all-one vector.
\end{proposition}

We note that the results in \citep{wu2020finite} use a different assumption from the one in \citep{tsitsiklis1999average} to guarantee the convergence of the critic in the average-reward setting (Assumption 4.1 in \citep{wu2020finite}): the matrix 
$\mathbb E[\phi(s)(\phi(s')-\phi(s))^{\top}]$ is negative
definite, where $\phi$ is the fixed feature function, $s$ is the current state and $s'$ is the subsequent state. 

As discussed in \Cref{sec:objective}, we would like the critic to find the solution of \Cref{eq:thetabar}. However, the objective in \Cref{eq:thetabar} requires the knowledge of $Q^{\pi_\omega}$, which is unavailable. Therefore, in the critic, we propose to use the method of $k$-step TD, so that as $k$ enlarges, the solution from the $k$-step TD converges to the solution of \Cref{eq:thetabar}. We present the $k$-step TD algorithm in \Cref{alg:ksetptd}. Here, the AC and NAC algorithms in \Cref{alg:AC} are single-loop, single sample trajectory and two time-scale. We introduce the $k$-step TD algorithm in \Cref{alg:ksetptd} only to illustrate the basic idea.

Based on \Cref{prop:e}, Assumption \ref{asm:ergodic}, and the assumption that $\Phi_\omega$ is full rank, from \citep[Theorem 1]{tsitsiklis1999average}, we can show that the $k$-step TD algorithm in \Cref{alg:ksetptd} has a unique solution, denoted by $\theta_{\omega}^*$:
\begin{flalign}\label{eq:kstep}
  {\E}_{D_{\pi_{\omega}}}
\Fbrac{\phi_{\omega}^\top(s,a)\brac{\mathcal{T}_{\pi_\omega}^{(k)} \brac{\phi^\top_{\omega}(s,a) \theta^*_{\omega}} -\phi^\top_{\omega}(s,a) \theta^*_{\omega}}}
=\mathbf{0},
\end{flalign}
where 
$ \mathcal{T}_{\pi_\omega}^{(k)}(Q (s,a))=\E\big[\sum_{j=0}^{k-1} (R_j- J(\pi_\omega)) + Q (s_k,a_k) |s_0=s,a_0=a,\pi_\omega \big].$

\vspace{0.08cm}
Assume that $\E_{D_w}\Fbrac{\phi_w(s,a)\phi^\top_w(s,a)}$ is positive definite with the minimum eigenvalue $\lambda_{\min}>0$. This is to guarantee that the solution to \Cref{eq:thetabar} is unique. We can remove this assumption by adding a regularizer $\lambda\|\theta\|_2^2$ to \Cref{eq:thetabar} to guarantee the solution to the regularized \Cref{eq:thetabar} is unique, and bounding the difference. 

Then we bound the difference between the solution to \Cref{eq:thetabar} and the solution to the $k$-step TD algorithm in the following proposition.

\begin{proposition}\label{prop:bound}
   For any $\omega\in \mathcal{W}$, the difference between $\theta_{\omega}^*$ and $\bar{\theta}_{\omega}^*$ can be bounded as follows:
$$
   \twonorm{\theta_{\omega}^*-\bar{\theta}_{\omega}^*}\leq 
   \frac{C_{\text{gap}}m\rho^k}{\lambda_{\min}},
   $$ where $C_{\text{gap}}$ is a constant defined in \Cref{sec:a2}. 
\end{proposition}
It can be seen that the bound diminishes exponentially with $k$. Therefore by picking a large $k$, the $k$-step TD is expected to solve \Cref{eq:thetabar} to a desired accuracy.

\begin{algorithm}[tb]
   \caption{Compatible $k$-step TD Algorithm}
   \label{alg:ksetptd}
\begin{algorithmic}
   \STATE {\bfseries Initialization:} $k$,  $\eta, \theta_0,$ $\phi=\nabla \log\pi_{\omega},$
   $s_0,a_0\sim \pi_\omega(\cdot|s_0),z_0=0$
   \FOR{$t=0,...,T-1$}
   \STATE Observe $R_t$ 
   \STATE $s_{t+1}\sim  P(\cdot|s_t,a_t)$;
    $a\tit\sim \pi_\omega(\cdot|s\tit) $
   \STATE  $\delta(\theta_t)=R_t-\eta+ \phi^\top(\sa)\theta_t-\phi^\top(s_t,a_t)\theta_t $
   \STATE\hfill{\verb+/*TD error*/+}
   \STATE $z_t=\sum_{j=t-k}^t \phi(s_j,a_j)$\STATE\hfill{\verb+/*eligibility trace*/+}
   \STATE   $\eta=\eta+ \gamma_t(R_t-\eta)$ \STATE\hfill{\verb+/*average reward update*/+}
   \STATE $\theta_{t+1}=\Pi_{2,B}(\theta_t+\alpha_t \delta(\theta_t)z_t)$ \hfill{\verb+/*TD update*/+}
   \ENDFOR
\end{algorithmic}
\end{algorithm}


\subsection{Non-asymptotic Bound for AC}\label{sec:3.2}
\vspace{0.1cm}

\begin{assumption}\label{asm:smooth}
(Smoothness and Boundedness) 
For any $\omega,\omega'\in \mathbb{R}^d$ and any state-action pair $(s,a)\in \mathcal{S}\times \mathcal{A}$, 
there exist positive constants 
$L_\phi,C_\phi,C_\pi$ and $L_\delta$ such that 
\begin{align}
 &   \text{1)} \twonorm{\phi_\omega(s,a)-\phi_{\omega'}(s,a)}\leq L_\phi \twonorm{\omega-\omega'};\nonumber
\\&\text{2)}\tvnorm{\pi_\omega(\cdot|s)-\pi_{\omega'}(\cdot|s)}\leq C_\pi \twonorm{\omega-\omega'};\nonumber
\\& \text{3)} \twonorm{\phi_\omega(s,a)}\leq C_\phi;\nonumber
\\&\text{4)}
\twonorm{\nabla^2 \pi_\omega(a|s) }\leq C_\delta;\nonumber
\\
&\text{5)}
\lbrac{\partial_{\omega_i}\partial_{\omega_j}\partial_{\omega_l}\pi_{\omega}(a|s) }\leq L_\delta \text{, for } 1\leq i,j,l\leq n.\qquad\qquad\nn
\end{align}
\end{assumption}

The first three assumptions in Assumption \ref{asm:smooth} assume the policy and feature function $\phi_\omega$ is smooth and bounded.
The fourth and fifth assumptions in Assumption \ref{asm:smooth} are only needed for the AC analysis. For the NAC analysis, it is not necessary.
We note that these assumptions can be easily satisfied by choosing a proper policy parameterization. 
For example,
if the policy is parameterized using neural network, then these assumptions can be satisfied \citep{du2019gradient,miyatospectral,neyshabur2017implicit} if the activation functions are analytic functions and have bounded each-order derivative, (e.g. logistic, hyperbolic tangent and softplus). With these proper policy parameterizations, the fifth one in Assumption \ref{asm:smooth} can be deduced by $\twonorm{\nabla^2 \pi_\omega(a|s)-\nabla^2 \pi_{\omega'}(a|s)} \leq L_{\delta} \twonorm{\omega-\omega'}$.

We first present the bound on the tracking error, which measures how the critic tracks its ideal limit: 
\vspace{-0.1cm}
$$\frac{1}{T}\sum _{t=0}^{T-1} \E\Fbrac{\twonormsq{\theta^*_t-\theta_t}}.$$  
Here, $\theta_t$ is the critic parameter at time $t$ of \Cref{alg:AC}, and we rewrite  $\theta^*_t=\theta^*_{\omega_t}$ {and $J(\omega_t)=J(\pi_{\omega_t})$}   for convenience. In the AC algorithm,
we set $\alpha_t=\alpha$,  
$\beta_t=\beta$, $\gamma_t=\gamma$, and $k = \mathcal{O}\brac{\log T}$ such that $\gamma\geq  \alpha\geq   \beta\geq m \rho^k.$
Note that we use a projection in Line 8 in \Cref{alg:AC}.  In order for convergence and optimality, we require that all $\|\theta_\omega^*\|\leq B$. A sufficient condition to guarantee this is to set $B=\frac{m R_{\max}C_\phi}{(1-\rho)\brac{\lambda_{\min}-C_\phi^2 d m \rho^k }}$ (see Appendix \ref{sec:basic} for the proof). 

\vspace{0.1cm}
\begin{proposition}\label{thm1}
 The tracking error of the AC algorithm in  \Cref{alg:AC} can be bounded as follows:
\begin{align}
&\frac{1}{T}\sum _{t=0}^{T-1} \E\Fbrac{\twonormsq{\theta^*_t-\theta_t}} \nn
\\ &\leq  \brac{\frac{c_\alpha\beta}{\alpha}+\frac{c_\eta \beta}{\gamma}}\frac{1}{T}\sum _{t=0}^{T-1} \E \ftwonormsq{\gradj}
+\mathcal{O}\brac{\frac{1}{T \alpha}}\nn
\\ & +\mathcal{O}\brac{\frac{\log^2 T}{T \gamma}}\nn
+\mathcal{O}\brac{\alpha\log^2 T }
+\mathcal{O}\brac{\beta\log^3 T }
\\& + \mathcal{O}\brac{\gamma\log^3 T }+\mathcal{O}\brac{\frac{\beta^2\log^2 T}{\alpha}}
+\mathcal{O}\brac{\frac{\beta^2\log T}{\gamma }}\nn
\\& +\mathcal{O}\brac{\frac{\beta}{\alpha}(m\rho^k)}
+\mathcal{O}\brac{(m\rho^k)\log^2 T}
+\mathcal{O}\brac{\frac{(m\rho^k)^2}{\alpha\beta}}\nn,
\end{align}
 where $c_\alpha$ and $ c_\eta$ is a positive constant defined in 
\Cref{sec:proofac}. 
Set $\gamma=\mathcal{O}(\frac{1}{\sqrt{T}})$,  $  \alpha=\mathcal{O}(\frac{1}{\sqrt{T}\log ^2 T})$, $ \beta =\mathcal{O}(\frac{1}{\sqrt{T}\log ^2 T})$, we have 
\begin{align}
    \frac{1}{T} &\sum_{t=0}^{T-1}\E[\twonormsq{\theta^*_t-\theta_t}]\nn\\&\leq 
\frac{1}{4C^4_\phi T}\sum_{t=0}^{T-1}\E [\twonormsq{\gradj}]+\mathcal{O}\brac{\frac{{\log^3 T}}{\sqrt{T}} }.
\end{align}
\end{proposition}

For simplicity, we only present the order of the bound here, and the detailed non-asymptotic bound can be found in the \Cref{sec:proofac}. The key novelty in the analysis is that we bound the tracking error as a function of the policy gradient, and we also bound the policy gradient as a function of the tracking error. By applying the bound recursively, we get a tight bound on the tracking error in \Cref{thm1}. Many existing studies in the two time-scale analysis upper bound the policy gradient in the tracking error using its maximum norm, which is constant-level. However, as we see in the following theorem, the policy gradient shall also decrease to zero. Therefore, the above approach does not obtain the tightest bound, and leads to a higher-order sample complexity.  



\begin{theorem}\label{thm:ac}
    Consider the AC algorithm in  \Cref{alg:AC}. It can be shown that
    \begin{align}
    \frac{1}{T}& \sum _{t=0}^{T-1} \E \Fbrac{\twonormsq{\gradj}} 
     \leq  \frac{2 C_\phi^4}{T} \sum _{t=0}^{T-1} \E \Fbrac{\twonormsq{\thestar-\theta_t}}\nn
    \\&+\mathcal{O}\brac{\frac{1}{T \beta}}
    +\mathcal{O}\brac{\beta \log^2 T }+\mathcal{O}(m\rho^k).\nonumber
    \end{align}
    Set $\gamma=\mathcal{O}(\frac{1}{\sqrt{T}}),  \alpha,  \beta =\mathcal{O}(\frac{1}{\sqrt{T}\log^2 T})$, then
$$
        \frac{1}{T} \sum _{t=0}^{T-1} \E \Fbrac{\twonormsq{\gradj}}
    \leq\mathcal{O}\brac{\frac{{\log^3 T}}{\sqrt{T}}}.
$$
\end{theorem}
\vspace{-0.1cm}
\Cref{thm:ac} implies that the AC algorithm with compatible function approximation converges to an $\epsilon$-stationary point with sample complexity $\epsilon^{-2}$. This improves the best known error bound by a constant $\varepsilon_{\text{critic}}$ 
\citep{wang2019neural,zhang2020provably,qiu2021finite,kumar2023sample,xu2020non,barakat2022analysis,wu2020finite,chen2021closing,olshevsky2022small,xu2020improving}, and matches the best known sample complexity \citep{chen2021closing,olshevsky2022small,xu2020improving,suttle2023beyond}. 




\subsection{Non-asymptotic Bound for NAC}
In this section, we present the non-asymptotic bound for the NAC algorithm in \Cref{alg:AC}. 
It was shown in \citep{agarwal2021theory} that due to the parameter invariant property of the natural  policy gradient update, natural  policy gradient is able to converge to the globally optimal policy with a gap that depends on the capacity of the policy class. Define the compatible linear function approximation error
$$
    \varepsilon_{\text{actor}}=\max_{\omega\in \mathcal{W}}\varbrac{\min_{\theta}\E_{D_{\pi_\omega}}\Fbrac{\twonormsq{A^{\pi_\omega}(s,a)-\phi_\omega ^\top(s,a) \theta}}}.
$$

This  error represents the approximation error due to the insufficient expressive power of the policy parameterization, and shall decrease if a large neural network is used.

Using the same idea as the one in AC, we can also develop a tight bound on the tracking error: $\mathcal{O}\brac{{T^{-\frac{1}{3}}}}$, where now we bound the tracking error as a function of the optimality gap instead of the gradient norm. We then also develop bound of the optimality gap as a function of the tracking error. Applying them recursively, we obtain the tightest bound on the tracking error and the tightest bound on the optimality gap in the following theorem.
We set $\alpha_t=\alpha$,  
$\beta_t=\beta$, $\gamma_t=\gamma$, and $k=\mathcal{O}\brac{\log T}$ such that $\gamma\geq \alpha\geq  \beta\geq m \rho^k.$

\begin{assumption}
    \label{asm:explore}
    There exist a constant $C_\infty< \infty$ such that 
    $
    \sup_{\omega\in\mathcal{W}}  \mynorm{\frac{D_{\pi^*}(s,a)}{D_{\pi_\omega}(s,a)}}_\infty\leq C_\infty.
    $
\end{assumption}

Assumption \ref{asm:explore} guarantees that the policy is sufficiently exploratory, and is commonly used in NAC analyses, e.g.,  \citep{cayci2022finite, xu2020improving,agarwal2021theory}. Approaches to guarantee this assumption were also studied in \citep{khodadadian2021finite,khodadadian2022finite}. 

\begin{theorem}\label{thm:nac}
     Consider the NAC algorithm in  \Cref{alg:AC}. Then,
    we have that 
   \eqenv{
   &\min_{t\leq T } \E\Fbrac{J(\pi^*)-J(\omega_t) }
     \leq \mathcal{O}\brac{\frac{\log ^2T}{T\alpha}}
     +\mathcal{O}\brac{\frac{\log T}{T\beta}}
     \\&+ \mathcal{O}\brac{\frac{\gamma\sqrt{\log T} }{\sqrt{\alpha}}}
     +\mathcal{O}\brac{\frac{\beta }{\sqrt{\alpha}}}
    + \mathcal{O}\brac{\sqrt{\alpha\log ^{3} T }}
     \\&+ \mathcal{O}\brac{\frac{\sqrt{\log ^3T}}{T\sqrt{\alpha\beta}}}
     +\mathcal{O}\brac{\sqrt{\frac{\beta \log T}{T\alpha}}}
    +\mathcal{O}\brac{\sqrt{\frac{\log T}{T}} }
     \\&+ \mathcal{O}\brac{\frac{\sqrt{\gamma\beta}\log T }{\sqrt{\alpha}}} 
     + \mathcal{O}\brac{\sqrt{\beta\log ^{3} T} }
    \\& + \mathcal{O}\brac{\sqrt{\frac{(m\rho^k)\beta}{\alpha}} }+ \mathcal{O}\brac{\frac{m\rho^k}{\sqrt{\alpha\beta}} }
     + \mathcal{O}\brac{\sqrt{\frac{(m\rho^k)\gamma}{\alpha}} }
     \\&+\mathcal{O}\brac{\sqrt{(m\rho^k) \log T} }
     +
     \mathcal{O}\brac{ 
    \sqrt{\varepsilon_{\text{actor}}}}.
   }
 If we set $\gamma=\mathcal{O}(T^{-\frac{2}{3}}\log T),  \alpha=\mathcal{O}(T^{-\frac{2}{3}}\log^{-1}T),  \beta =\mathcal{O}(T^{-\frac{2}{3}}\log^{-1}T)$, we have 
    \begin{align}
    \min_{t\leq T } \E[J(\pi^*)-J(\omega_t)]\leq \mathcal{O}\brac{{T^{-\frac{1}{3}}}\log ^3T}
    + \mathcal{O}\brac{\sqrt{\varepsilon_{\text{actor}}}}.\nonumber
    \end{align}
\end{theorem}

\begin{remark}
    Unlike the results for AC in \Cref{thm:ac}, \Cref{thm:nac} for NAC only needs the first three assumptions in Assumption \ref{asm:smooth}. This is one advantage of using compatible function approximation in NAC. As we can see from Line 11 in \Cref{alg:AC} and \Cref{prop:npg}, the inverse of the Fisher information matrix is cancelled out. Therefore, there is no stochastic noise from using $\phi_{\omega}(s_t,a_t)\phi^\top_{\omega}(s_t,a_t)$ in the analysis of NAC. However, in AC, we need to handle this noise, and therefore, the fourth assumption in Assumption \ref{asm:smooth} is needed for the AC algorithm.
\end{remark}
\Cref{thm:nac} implies that NAC with compatible function approximation  converges to an $\epsilon+\sqrt{\varepsilon_{\text{actor}}}$-neighborhood of the globally optimal policy $\pi^*$ with sample complexity $\mathcal O(\epsilon^{-3})$. 
Compared to existing studies, our work eliminate the approximation error of the critic, $\varepsilon_{\text{critic}}$, from the overall error bound \citep{wang2019neural,cayci2022finite,agarwal2021theory,xu2020improving,xu2020non,chen2022finite}. Moreover, as summarized in \Cref{table2},  the best known sample complexity of NAC is $\epsilon^{-3}$, which however is for the nested-loop NAC variant \citep{xu2020improving,chen2022finite}. Our results achieves this sample complexity, and is for the challenging single-loop NAC algorithm with a single Markovian sample trajectory.

Here we provide a proof sketch for the NAC algorithm to highlight major challenges and our technical novelties. The analysis of NAC contains of most major technical novelty in the AC analysis.
\begin{proof}[Proof sketch]
For simplicity of presentation, we set $\hat t=\mathcal{O}\brac{\frac{\log T}{\alpha}}$ and 
$\widetilde T=\hat t\lceil \frac{T}{\hat t \log T} \rceil$.
 We denote by $$M_t=\mathbb E[||\theta_{t}-\theta^*_{t}||^2_2]+ \mathbb E[(\eta_t-J(\omega_t))^2]$$ the sum of the tracking error and the estimation error of the average reward. Denote by $\D(\omega_t)=KL(\pi^*|\pi_t)$ the KL divergence between policy $\pi^*$ and $\pi_t$.

\textbf{Step 1 (Error decomposition):}
According to the smoothness property of $\D(\omega)$ with respect to $\omega$, we bound the performance gap between the current policy and the optimal policy (optimality gap) as follows:
\begin{align}
    &\frac{1}{\widetilde T}\sum_{j=t}^{t+\widetilde T-1}\mathbb E[J(\pi^*)-J(\omega_j)]\leq \frac{\D(\omega_{t+\widetilde T})-\D(\omega_t)}{\widetilde T \beta }\nn\\
     &+ \mathcal{O}\left(\sqrt{\frac{1}{\widetilde T}\textstyle\sum_{j=t}^{t+\widetilde T-1} M_j}\right)+\mathcal{O}(C_\infty \sqrt{\varepsilon_{\text{actor}}}+ \beta+m\rho^k).\nn
\end{align}

\textbf{Step 2 (Estimation error in the average reward):} In this step, we analyze estimation error in the average reward:
 $\eta_t-J(\omega_t)$. We provide a tight characterization of this error:
\begin{align}
   \mathbb E& [(\eta_{t+1}-J(\omega_{t+1}))^2] \nn\\&\leq (1-\gamma)\mathbb E[(\eta_{t}-J(\omega_{t}))^2] 
     +\mathcal{O}(\beta \mathbb E[||\nabla J(\omega_t)||_2^2])
    \nn\\&\quad+ \mathcal{O}( m\rho^k\gamma +k^2 \gamma^2 +k^2 \gamma\beta +\beta^2).\nn
\end{align}
One of our key novelties lies in that we bound this estimation error using the gradient norm $\mathbb E[||\nabla J(\omega_t)||_2^2]$. The above bound itself is tighter than the existing one in (Wu et al., 2020).

\textbf{Step 3 (Tracking error):}
In this step, we bound the tracking error in the critic: $||\theta_t-\theta_t^*||_2^2$. 
By the TD error step in Algorithm 1, we decompose the term $||\theta_{t+1}-\theta_{t+1}^*||^2_2$ as follows:
\begin{align}
    ||\theta&_{t+1}-\theta_{t+1}^*||^2_2 \nn\\&\leq || \theta_t-\theta_{t}^* ||^2_2+|| \theta_{t}^*-\theta_{t+1}^*||_2^2 + \alpha^2 ||\delta_t z_t||^2_2 \nn\\&\quad+ 2 \alpha\langle\theta_t-\theta_{t}^*, \delta_t z_t \rangle+ 2\alpha \langle \theta_{t}^* - \theta_{t+1}^* , \delta_t z_t\rangle
 \nn\\&\quad+2 \langle\theta_t-\theta_{t}^*, \theta_{t}^*-\theta_{t+1}^* \rangle.\nn
\end{align}
Another key challenge lies in  how to bound the term 
$\mathbb E[\langle \theta_t-\theta_t^*, \delta_t z_t\rangle ] $.  We develop a novel technique of auxiliary Markov chain to decompose this error into two parts: 1) error due to time-varying feature function and 2) error due to time-varying policy. Specifically, consider the first Markov chain generated from the algorithm:
$$s_0,a_0 \overset{\pi_0\times P}{\to} s_1,a_1\to...\to s_t,a_t \overset{\pi_t\times P}{\to}s_{t+1},a_{t+1},$$ where at each time $j$, the action is chosen according to $\pi_j$ and the transition kernel is $P$. Here $z_t=\sum_{j=t-k}^t \phi_j(s_j,a_j)$ is the eligibility trace used in the algorithm. It can be seen that in $z_t$, the feature $\phi_j$ changes with $j$, and the distribution of $s_j,a_j$ depends on the time-varying policy $\pi_j$. We then design an auxiliary eligibility trace $\hat z_t=\sum_{j=t-k}^t \phi_t(s_j,a_j)$, where the feature is fixed to be $\phi_t$, and only the the distribution of $s_j,a_j$ depends on the time-varying policy $\pi_j$. To further handle the time-varying distribution of  $s_j,a_j$, we design an auxiliary Markov chain (denoted by $A1$) as follows:
$$A1: ( s_0,\widetilde  a_0 )\sim \pi_t\overset{\pi_t\times P}{\to} \widetilde s_1,\widetilde a_1{\to}...\widetilde  s_t,\widetilde a_t \overset{\pi_t\times P}{\to}\widetilde s_{t+1},\widetilde a_{t+1},$$ where the action at each time $j$ is always chosen according to a fixed policy $\pi_t$. Based on this auxiliary Markov chain, we introduce another auxiliary eligibility trace $\widetilde  z_t=\sum_{j=t-k}^t \phi_t(\widetilde s_j,\widetilde a_j)$, where it uses a fixed feature $\phi_t$, and samples from this auxiliary Markov chain. Lastly, we design a second auxiliary Markov chain (denoted by $A2$):
$$A2: (\bar s_0,\bar  a_0)\sim D_t \overset{\pi_t\times P}{\to} \bar s_1,\bar a_1{\to}...\bar  s_t,\bar a_t \overset{\pi_t\times P}{\to}\bar s_{t+1},\bar a_{t+1},$$
where the only difference between A2 and A1 lies in the initial state distribution. Then we define the last auxiliary eligibility trace as $\bar  z_t=\sum_{j=t-k}^t \phi_t(\bar s_j,\bar a_j)$.

The difference between $z_t $ and $\hat z_t$ measures the error due to the time-varying compatible feature function. We bound this error using the Lipschitz continuity of the feature function.  The difference between $\hat z_t $ and $\widetilde z_t$ measures the error due to the  time-varying sampling policy.   The difference between $\widetilde z_t $ and $\bar z_t$ measures the error due to the difference between the stationary distribution and the actual distribution of the samples, which can be bounded based on Assumption 1. By such a error decomposition, we can show that
\begin{align}
    \mathbb E&[||\theta_{t+1}-\theta^*_{t+1}||^2_2] \leq (1-\bar \lambda_{\min}\alpha /2)\mathbb E[||\theta_{t}-\theta^*_{t}||^2_2]
   \nn\\
   &+\mathcal{O}(k^2 \alpha \mathbb E[(\eta_t-J(\omega_t))^2] )+\mathcal{O}(\beta\mathbb E [||\nabla J(\omega_t)||_2^2])\nn\\
   &+ \mathcal{O}\brac{k^3 \alpha^2+k^3\alpha\beta+\beta^2+m\rho^k\alpha+\frac{(m\rho^k)^2}{\beta} } .\nn
\end{align}
\textbf{Step 4 (Bound on gradient):} As we can see from Steps 2 and 3, we bound the estimation error of the average reward and the tracking error using the gradient norm $\|\nabla J(\omega_t)\|_2^2$. Therefore,  in order to derive the tightest bound, we further develop a novel bound on the gradient norm $\|\nabla J(\omega_t)\|_2^2$. Note that the idea is novel as it serves as a pivotal link connecting the analysis of the tracking error/estimation error in the average reward and the optimality gap. Specifically, we bound the gradient norm using 
the estimation error in the average reward and tracking error. By the smoothness of $J(\omega)$, we have that
\vspace{-0.1cm}
\begin{align}
\sum_{j=t}^{t+\widetilde T-1}&\frac{\mathbb E [ ||\nabla J(\omega_j) ||_2^2]}{\widetilde T}\leq 2C^2_\phi \frac{J(\omega^*)- E [J(\omega_{t})]}{\beta \widetilde T}\nn
\\& + \mathcal O\bigg(\frac{1}{\widetilde T}\sum_{j=t}^{t+\widetilde T-1}\mathbb E[||\theta_j-\theta_j^*||_2^2]\bigg) +\mathcal{O}(m\rho^k+\beta).\nn
\end{align}
We also note that we bound the gradient norm using the optimality gap, and this is of great importance to establish the tight bound in this paper. In previous works, this  term $\mathbb E [J(\omega_{t+\widetilde T})]- E [J(\omega_{t})] $ is bounded by a constant, and thus the overall complexity is not as tight. 

\textbf{Step 5:} Combining steps 1-4, we conclude the proof.
\end{proof}
\vspace{-0.2cm}
\section{Conclusion}
In this paper, we develop the tightest non-asymptotic convergence bounds for 
both the AC and NAC algorithms with compatible function approximation. For the AC algorithm, our results achieve the best sample complexity of $\epsilon^{-2}$ with a reduced error from $\epsilon+\varepsilon_{\text{critic}}$ to $\epsilon$, where $\varepsilon_{\text{critic}}$ is a non-diminishing constant. For the NAC algorithm, our results is the first one in the literature that analyze the single-loop NAC with a single Markovian trajectory, and we achieve the best known sample complexity of $\epsilon^{-3}$ also with a reduced error of $\epsilon+\sqrt{\varepsilon_{\text{actor}}}$. Our results demonstrate the advantage of compatible function approximation when applied in AC and NAC algorithms, including relaxed technical condition to guarantee convergence, no need of estimating Fisher information matrix, and no approximation error from the critic. Our technical novelty lies in analyzing the error due to use of a  time-varying and policy dependent feature in the critic.




\section*{Impact Statement }This paper presents work whose goal is to advance the field of Reinforcement Learning. There are many potential societal consequences of our work, none which we feel must be specifically highlighted here.

\section*{Acknowledgments}
The work of Yudan Wang and Shaofeng Zou is supported by the National Science Foundation under Grants CCF-2007783, CCF-2106560 and ECCS-2337375 (CAREER). Yi Zhou’s work is supported by the National Science Foundation under grants CCF-2106216, DMS-2134223, ECCS-2237830 (CAREER).

This material is based upon work supported under the AI Research Institutes program by National Science Foundation and the Institute of Education Sciences, U.S. Department of Education through Award \# 2229873 - National AI Institute for Exceptional Education. Any opinions, findings and conclusions or recommendations expressed in this material are those of the author(s) and do not necessarily reflect the views of the National Science Foundation, the Institute of Education Sciences, or the U.S. Department of Education.

\bibliography{icml2024}
\bibliographystyle{icml2024}



\newpage

\newpage
\onecolumn
\appendix\label{appendix}

\addcontentsline{toc}{section}{Appendix} 
\part{Appendix} 
\parttoc 


\section{Supporting Lemmas and Proofs for  \Cref{prop:e,prop:bound}}\label{sec:basic}
In this section, we provide a number of supporting lemmas, and proofs for \Cref{prop:pg,prop:npg,prop:e,prop:bound}. In the following proofs, $\|x\|_2$ denotes the $\ell_2$ norm if $x$ is a vector; and $\|X\|_2$ denotes the operator norm if $X$ is a matrix.

\subsection{Supporting Lemmas}
For convenience, we denote $J(\omega)=J(\pi_\omega)$. We first prove a lemma showing that both $J(\omega)$ and $\nabla J(\omega)$ are Lipschitz in $\omega$.
\begin{lemma}
\label{Prop:lip}
Under Assumptions  \ref{asm:ergodic} and \ref{asm:smooth}, for any $\omega, \omega'\in\mathcal W$, we have that
\eqenv{\label{cons:cj}
\twonorm{\nabla J(\omega)}\leq C_J,
}
             where $C_J=   C^2_\phi \brac{B+ \frac{C_{\text{gap}} m \rho^k}{\lambda_{\min}}}$, and 
\eqenv{\label{cons:lj}
\twonorm{\nabla J(\omega)-\nabla J(\omega')}\leq L_J \twonorm{\omega-\omega'}
,} where $L_J= \frac{m R_{\max}}{1-\rho}\brac{4L_\pi C_\phi+L_\phi}$ and 
$L_\pi=
\frac{1}{2}C_\pi 
\brac{1+\myceil{\log m^{-1}}+\frac{1}{1-\rho}}$.
\end{lemma}

Recall \Cref{eq:thetabar}. The solution $\bar \theta_\omega^*$ given the feature function satisfies that
     \eqenv{
     \bar \theta^*_\omega=\arg\min_\theta \E_{D_\omega}\ftwonormsq{Q^{\pi_\omega}(s,a)-\phi^\top _\omega(s,a)\theta}.
     }
We show that the solution $\bar \theta^*_\omega$ is   Lipschitz in $\omega$ in the following lemma.
\begin{lemma}\label{lm:acnacinnersmooth}
For any  $\omega, \omega' \in\mathcal{W}$, it holds that
\eqenv{\label{cons:ctheta}
\twonorm{\bar \theta^*_\omega-\bar \theta^*_{\omega'}}\leq C_\Theta\twonorm{\omega-\omega'},
} 
    where $C_\Theta=\frac{C_J}{\lambda^2_{\min}}
         \brac{2 C_\phi L_\phi+C^2_\phi L_\pi }+ \frac{L_J}{\lambda_{\min}}$. 
\end{lemma}

For any $\omega\in \mathcal{W}$, let 
\begin{align}\label{eq:aw}
    H_{\omega}(s,a)&=\E\Fbrac{\phi_{\omega}(s_0,a_0)\brac{\phi_{\omega}(s_k,a_k)-\phi_{\omega}(s_0,a_0) }^\top |s_0=s,a_0=a,\pi_\omega},\nn\\
    H_\omega&=\E_{D_{\pi_\omega}}\Fbrac{H_\omega(s,a)}.
\end{align}
\begin{lemma}\label{lm:eig}
 For $ k>\myceil{\frac{\log( md C^2_\phi)-\log \lambda_{\min} }{1-\rho}}$, it holds that
    $$\lambda_{\max}\brac{\frac{H_\omega+H_\omega^\top}{2}}\leq C_\phi^2 d m\rho^k-\lambda_{\min}=-\bar \lambda_{\min}<0, $$ 
    where $\lambda_{\max}\brac{X}$ is the largest eigenvalue of symmetric matrix $X$.
\end{lemma}
When $ k>\myceil{\frac{\log( md C^2_\phi)-\log \lambda_{\min} }{1-\rho}}$, we have that
\eqenv{ 
    \bar \lambda_{\min}&=\lambda_{\min}-C_\phi^2 d m\rho^k 
    \\& > \lambda_{\min}-C_\phi^2 d m e^{-k(1-\rho)}
    \\&\geq \lambda_{\min}-C_\phi^2 d m e^{-\log\brac{\frac{mdC^2_\phi}{\lambda_{\min}}}}=0,
}
and therefore $\bar \lambda_{\min}$ is positive.

The following lemma bounds the distance between the stationary distribution induced by $\pi_t$ and the distribution of $s_t,a_t$ in \Cref{alg:AC}. Define $\mathcal{F}_j$ to be $\sigma$-field generated by all the randomness until the $j$-th time-step. For simplicity, we write $D_{\pi_t}$ as $D_t$.
\begin{lemma}
    \label{lm:tv}
   For any $0\leq k\leq t$, it can be shown that
    \eqenv{
    {\tvnorm{\Prob\brac{s_t,a_t|\mathcal{F}_{t-k}}-D_t}}\leq
    C_\pi \sum_{j=t-k}^{t-1} \twonorm{\omega_t-\omega_j}+ m\rho^k.
    }
\end{lemma}

We rewrite $\thestar=\theta^*_{\omega_t} $, where $\theta^*_{\omega}$ is the solution to \Cref{eq:kstep}. 

\begin{lemma}\label{lm:consterm}
    Consider the term $\E \fvbrac{\theta_t-\thestar, \delta_t z_t} $. It can be shown that
    \eqenv{
\E\fvbrac{\theta_t-\thestar, \delta_t z_t}
\leq& - \frac{\bar \lambda_{\min}}{2} \E \ftwonormsq{\theta_t-\thestar}
\\& +\frac{(k+1)^2 C^2_\phi}{2\bar\lambda_{\min}}\E\flnormsq{J(\omega_t)-\eta_t}+G^\delta_t
, \label{eq:lm:ac:teI}
}
 where $U_\delta =R_{\max}+2 C_\phi B$. For AC, 
 \eqenv{\label{vari:gdeltat1}
 G^\delta_t &=2B^2 C^2_\phi U_\delta C_\pi \sum_{j=t-k}^{t} \sum_{i=j}^{t-1} \beta_i
+4  B C_\phi U_\delta \sum_{j=t-k}^t\brac{ 
 B C^2_\phi  C_\pi \sum_{i=j-k}^{j-1} \sum_{\iota=i}^{j-1} \beta_\iota  + m\rho^k+ BC^2_\phi L_\pi\sum_{i=j}^{t-1} \beta_i  }
 \\&+  2(k+1) C_\phi U_\delta\brac{(k+1) C_\phi U_\delta \sum_{j=t-2k}^{t-1} \alpha_j+ C_\Theta C^2_\phi B \sum_{j=t-2k}^{t-1} \beta_j+\frac{2C_{\text{gap}}m \rho^k}{\lambda_{\min}}},
 }
and for NAC, 
 \eqenv{\label{vari:gdeltat2}
 G^\delta_t &=2B^2 U_\delta C_\pi \sum_{j=t-k}^t \sum_{i=j}^{t-1} \beta_i+4  B C_\phi U_\delta \sum_{j=t-k}^t\brac{ 
 B C_\pi \sum_{i=j-k}^{j-1} \sum_{\iota=i}^{j-1} \beta_\iota   + m\rho^k+B L_\pi\sum_{i=j}^{t-1} \beta_i  }
\\& +2(k+1) C_\phi U_\delta\brac{(k+1) C_\phi U_\delta \sum_{j=t-2k}^{t-1} \alpha_j+ C_\Theta  B \sum_{j=t-2k}^{t-1} \beta_j+\frac{2C_{\text{gap}}m \rho^k}{\lambda_{\min}}}
.
 }

\end{lemma}

In the following, we prove that $\theta^*_\omega$ defined in 
\Cref{eq:kstep}  is bounded.
\begin{lemma}\label{lm:lm6}
    The solution $\theta^*_\omega$ to \Cref{eq:kstep}  is bounded:
    \begin{align}
        \|\theta^*_\omega\|_2\leq \frac{1}{\lambda_{\min}-d C_\phi^2 m \rho^k }\frac{m C_\phi R_{\max}}{1-\rho}=\frac{m C_\phi R_{\max}}{\bar \lambda_{\min}\brac{1-\rho}}.
    \end{align}
\end{lemma}

\begin{lemma}\label{lm:acinnersmooth}
%
Under Assumption \ref{asm:smooth} and \ref{asm:ergodic}, for any  $\omega, \omega'\in\mathcal{W}$,
\eqenv{
    \twonorm{\nabla^2 J(\omega)-\nabla^2J(\omega')}\leq L_\Theta\twonorm{\omega-\omega'},
}
    where $L_\Theta= d^2 \brac{ \frac{6 C^3_\phi m^3 e^{4R_{\max}}}{(1-\rho)^3}+\frac{6 m^2 C_\phi C_\delta e^{3R_{\max}}}{(1-\rho)^2}+\frac{m L_\delta e^{2R_{\max}}}{1-\rho}}$. 
\end{lemma}

The proof of above Lemmas could be found in \Cref{sec:secd}.

\subsection{Proofs for \Cref{prop:pg,prop:npg,prop:e,prop:bound}}\label{sec:a2}
We include the proof of \Cref{prop:pg} and \Cref{prop:npg} for completeness. 
\begin{proof}
By the \Cref{eq:thetabar}, $\bar \theta^*_\omega$ satisfies that 
\eqenv{
\E_{D_{\pi_\omega}}\Fbrac{(Q^{\pi_\omega}(s,a)-\phi^\top_\omega(s,a)\bar \theta_\omega^*)\phi_\omega(s,a) }=0.
}
Since $ \phi_\omega^\top(s,a) \bar\theta^*_\omega$ is a scalar, we can get that
\eqenv{
\E_{D_{\pi_\omega}}\Fbrac{Q^{\pi_\omega}(s,a)\phi_\omega(s,a) }=\E_{D_{\pi_\omega}}\Fbrac{\phi_\omega(s,a)\phi^\top_\omega(s,a)\bar \theta_\omega^* } 
}

For the policy gradient $\nabla J(\pi_\omega)$, we get that
\eqenv{
 \nabla J(\pi_\omega) =\E_{D_{\pi_\omega}}\Fbrac{ \nabla \log \pi_\omega(a|s)Q^{\pi_\omega}(s,a)}=
 \E_{D_{\pi_\omega}}\Fbrac{\phi_\omega(s,a)(\phi^\top_\omega\brac{s,a}\bar\theta_\omega^*)}.
}
Furthermore, we have that
\begin{align}
    \widetilde{\nabla} J(\pi_\omega)&= F^{-1}_\omega
\E_{D_{\pi_\omega}}\Fbrac{\phi_\omega(s,a)\phi_\omega(s,a)^\top \bar \theta^*_\omega  }
     \nn\\
     &=F^{-1}_\omega
     \E_{D_{\pi_\omega}}\Fbrac{\phi_\omega(s,a)\phi_\omega(s,a)^\top   }\bar \theta^*_\omega=\bar \theta^*_\omega. \end{align}
     This conclude the proof.
     \end{proof}

We present the proof of \Cref{prop:e}. 
\begin{proposition}(Restatement of \Cref{prop:e})
For any $\omega\in \mathcal{W}$ and $\theta \in \Theta$, 
$\Phi_\omega \theta \neq \mathbf{e},$
where  $\mathbf e\in \mathbb R^{|\mathcal{S}||\mathcal{A}|}$ is an all-one vector.
\end{proposition}
\begin{proof}
Assume that there exists $\theta_c\in\Theta$ such that $\Phi_\omega \theta_c=\mathbf{e}$, then $\E_{D_{\pi_\omega}}\Fbrac{\phi_\omega^\top(s,a)\theta_c}=1$. 

However, note that
\eqenv{
    \E_{D_{\pi_\omega}}\Fbrac{\phi_\omega^\top(s,a)\theta_c}&=
    \sum_{s}d_{\pi_\omega}(s)\sum_{a}\pi_\omega(a|s) \phi_\omega^\top(s,a)\theta_c
   \\ &  =
    \sum_{s}d_{\pi_\omega}(s)\sum_{a}\pi_\omega(a|s) \nabla \log\pi_{\omega}(a|s)^\top\theta_c
    \\ &  =
    \sum_{s}d_{\pi_\omega}(s)\sum_{a}\pi_\omega(a|s) \frac{\nabla_{\omega} \pi_\omega(a|s)^\top}{\pi_\omega(a|s)}\theta_c
    \\&  =
    \sum_{s}d_{\pi_\omega}(s)\sum_{a}  \nabla_{\omega}\pi_\omega (a|s)^\top \theta_c\\
    &=\sum_{s}d_{\pi_\omega}(s)  \nabla_{\omega}\left(\sum_{a}  \pi_\omega (a|s) \right)^\top\theta_c\\
    &=0,
    }
    where the last equation is from the fact that $\sum_a \pi_\omega(a|s)=1$, and hence the gradient of it is $0$. 
    This hence results in a contradiction, which completes the proof.
\end{proof}

We then present the proof of \Cref{prop:bound}. 
\begin{proposition}(Restatement of \Cref{prop:bound})\label{lm:cgap}
   For any $\omega\in \mathcal{W}$, denote the fixed point of $k$-step TD operator by $\theta_{\omega}^*$, and the solution to \Cref{eq:thetabar} by $\bar{\theta}_{\omega}^*$, then
   \eqenv{
   \twonorm{\theta_{\omega}^*-\bar{\theta}_{\omega}^*}\leq 
   \frac{C_{\text{gap}}m\rho^k}{\lambda_{\min}},
   } where $C_{\text{gap}}=C^2_\phi B  + \frac{C_\phi R_{\max}}{1-\rho}$.
\end{proposition}
\begin{proof}
   From the definition, it holds that
    \eqenv{
\theta^*_{\omega}=\brac{\E_{D_{\pi_\omega}}\Fbrac{\phi_{\omega}(s,a)\phi^\top_{\omega}(s,a)}}^{-1}\brac{\E_{D_{\pi_\omega}}\Fbrac{\phi_{\omega}(s,a)\brac{\mathcal{T}_{\pi_\omega}^{(k)}\phi_{\omega}^\top(s,a)\theta^*_{\omega}}}},
    }
    and
        \eqenv{
\bar\theta^*_\omega=\brac{\E_{D_{\pi_\omega}}\Fbrac{\phi_{\omega}(s,a)\phi^\top_{\omega}(s,a)}}^{-1}\brac{\E_{D_{\pi_\omega}}\Fbrac{\phi_{\omega}(s,a)Q^{\pi_{\omega}}(s,a)}}.
    }
Thus, we have that
\eqenv{
&\twonorm{\theta^*_\omega-\bar\theta^*_\omega}\\&=\twonorm{\brac{\E_{D_{\pi_\omega}}\Fbrac{\phi_{\omega}(s,a)\phi^\top_{\omega}(s,a)}}^{-1}\brac{\E_{D_{\pi_\omega}}\Fbrac{\phi_{\omega}(s,a)\brac{\mathcal{T}_{\pi_\omega}^{(k)}\phi_{\omega}^\top(s,a)\theta^*_{\omega}-Q^{\pi_{\omega}}(s,a)}}}}
\\&{\leq} \frac{1}{\lambda_{\min}}
\twonorm{\E_{D_{\pi_\omega}}\Fbrac{\phi_{\omega}(s,a)\brac{\mathcal{T}_{\pi_\omega}^{(k)}\phi_{\omega}^\top(s,a)\theta^*_{\omega}-Q^{\pi_{\omega}}(s,a)}} }
\\& = \frac{1}{\lambda_{\min}}
\Bigg \|\E_{D_{\pi_\omega}}\Bigg [\phi_{\omega}(s,a)\Bigg (\E\Fbrac{
\sum_{j=0}^{k-1} R_j-J(\omega)+\phi_{\omega}(s_k,a_k)^\top \theta^*_{\omega}|s_0=s,a_0=a, \pi_{\omega}
}
-Q^{\pi_{\omega}}(s,a)\Bigg )\Bigg ] \Bigg \|
\\& = \frac{1}{\lambda_{\min}}
\Bigg \| \E_{D_{\pi_\omega}}\Bigg [\phi_{\omega}(s,a)\Bigg (\E\Fbrac{
\sum_{j=0}^{k-1} R_j-J(\omega)+\phi_{\omega}(s_k,a_k)^\top \theta^*_{\omega}|s_0=s,a_0=a, \pi_{\omega}
}
\\ & \quad  -\E\Fbrac{
\sum_{j=0}^\infty R_j-J(\omega)|s_0=s,a_0=a, \pi_{\omega}
}\Bigg )\Bigg ]\Bigg \|_2
\\ & = \frac{1}{\lambda_{\min}}
\twonorm{\E_{D_{\pi_\omega}}\Fbrac{\phi_{\omega}(s,a)\brac{\E\Fbrac{
\sum_{j=k}^{\infty}R_j-J(\omega)+\phi_{\omega}(s_k,a_k)^\top \theta^*_{\omega} | s_0=s,a_0=a, \pi_{\omega}
}}}}
\\ & \myineq{\leq}{a}
\frac{1}{\lambda_{\min}}
\twonorm{\E_{D_{\pi_\omega}}\Fbrac{\phi_{\omega}(s,a)\brac{\E\Fbrac{
\sum_{j=k}^{\infty}R_j-J(\omega)\bigg |(s_k,a_k)\sim D_{\pi_\omega},\pi_\omega
}}}}
\\&\quad +\frac{1}{\lambda_{\min}}
\twonorm{\E_{D_{\pi_\omega}}\Fbrac{\phi_{\omega}(s,a)\brac{\E\Fbrac{
\phi_{\omega}(s_k,a_k)^\top \theta^*_{\omega}\big | (s_k,a_k)\sim D_{\pi_\omega}
}}}}
\\& \quad +\frac{1}{\lambda_{\min}}
C^2_\phi B \E_{D_{\pi_\omega}}\Fbrac{\tvnorm{\Prob\brac{s_k,a_k|s_0=s,a_0=a,\pi_{\omega}}-D_{\pi_\omega}} }
\\& \quad +\frac{1}{\lambda_{\min}}
C_\phi\sum_{j=k}^\infty R_{\max} 
\E_{D_{\pi_\omega}}\Fbrac{\tvnorm{\Prob(s_j,a_j|s_0=s,a_0=s,\pi_{\omega})-D_{\pi_\omega}}}
\\ & \myineq{\leq}{b} \frac{1}{\lambda_{\min}}
\twonorm{\E_{D_{\pi_\omega}}\Fbrac{\phi_{\omega}(s,a)\brac{\E\Fbrac{
\sum_{j=k}^{\infty}R_j-J(\omega)\bigg |(s_k,a_k)\sim D_{\pi_\omega},\pi_\omega
}}}}
\\&\quad +\frac{1}{\lambda_{\min}}
\twonorm{\E_{D_{\pi_\omega}}\Fbrac{\phi_{\omega}(s,a)\brac{\E\Fbrac{
\phi_{\omega}(s_k,a_k)^\top \theta^*_{\omega}\big | (s_k,a_k)\sim D_{\pi_\omega}
}}}}
\\& \quad +\frac{C_\phi}{\lambda_{\min}}\brac{
C_\phi B m \rho^k +\sum_{j=k}^\infty R_{\max} m \rho^j
}
\\ & \myineq{\leq}{c} 
\frac{1}{\lambda_{\min}}
\twonorm{\E_{D_{\pi_\omega}}\Fbrac{\phi_{\omega}(s,a)\brac{\E_{D_{\pi_\omega}}\Fbrac{
\phi_{\omega}(s,a)^\top \theta^*_{\omega}}}}}
+\frac{1}{\lambda_{\min}}
{C_\phi\brac{
C_\phi B m \rho^k +\sum_{j=k}^\infty R_{\max} m \rho^j
}}
\\ & \myineq{\leq}{d}
\frac{1}{\lambda_{\min}}\brac{C^2_\phi B m \rho^k + C_\phi R_{\max}\frac{m \rho^k}{1-\rho}}\\
&
=  \frac{C_{\text{gap}}m \rho^k}{\lambda_{\min}},
}
where $C_{\text{gap}}= C^2_\phi B  + C_\phi R_{\max}\frac{1}{1-\rho}$, $(a)$ follows from the triangular inequality and the fact that for any probability distribution $P_1 $ and $P_2$, and any random variable $X$, s.t. 
$\lbrac{X}\leq X_{\max}$,  $ \lbrac{\E_{P_1}\Fbrac{X}-\E_{P_2}\Fbrac{X}}\leq X_{\max} \tvnorm{ P_1-P_2}$, $(b)$ follows from Assumption \ref{asm:ergodic}, $(c)$ follows from $J(\omega)=\E _{D_{\pi_\omega}}[R(s,a)]$,  and $(d)$ follows from $\E_{D_{\pi_\omega}}\Fbrac{\phi_{\omega}(s,a)\brac{\E_{D_{\pi_\omega}}\Fbrac{
\phi_{\omega}(s,a)^\top \theta^*_{\omega}}}}=0 $.
\end{proof}

\section{AC Sample Complexity Analysis}\label{sec:proofac}
In this section, we provide the sample complexity analysis for our single-loop AC algorithm. 

\subsection{Bound on Gradient Norm in AC}
In this section, we first present a preliminary bound on the gradient norm $\|\nabla J(\omega)\|$. 

\begin{lemma}\label{lm:acgrad}
It holds that
\eqenv{\label{eq:graditera}
\frac{\beta_t}{2}&\E \ftwonormsq{\gradj}
 \leq \E\Fbrac{J(\omega\tit)}- \E\Fbrac{J(\omega_t)}
+C_\phi^4 \beta_t \E\Fbrac{\twonormsq{\theta_t-\thestar}}+ G^\omega_t, 
}
  where
\eqenv{\label{sym:Gomegat}
G^\omega_t=&\frac{L_J C_\phi^4 B^2 \beta_t^2}{2} +\brac{\frac{C_\phi^4 C^2_{\text{gap}}m\rho^k}{\lambda^2_{\min}}+C_J C^2_\phi B}\beta_t m \rho^k
+ 2 C^4_\phi B^2 L_J \beta_t\sum_{j=t-k}^{t-1} \beta_j
\\&+ C_J C^4_\phi B^2C_\pi  \beta_t\sum_{j=t-k}^{t-1} \sum_{i=j}^{t-1} \beta_i.
}
\end{lemma}
\begin{proof}
Recall that in the update of AC algorithm,  $\omega\tit-\omega_t=\beta_t\phi^\top_t(s_t,a_t)\theta_t\phi_t(s_t,a_t)$.  Following \Cref{Prop:lip}, it can be shown that
\eqenv{
J&(\omega\tit)\geq J(\omega_t)+ \vbrac{\gradj, \omega\tit -\omega_t}-\frac{L_J}{2} \twonormsq{\omega\tit-\omega_t}
\\& = J(\omega_t)+ \beta_t\vbrac{\gradj, \phi_t^\top(s_t,a_t)\theta_t\phi_t(s_t,a_t)}- \frac{L_J\beta_t^2}{2} \twonormsq{\phi_t^\top(s_t,a_t)\theta_t\phi_t(s_t,a_t)}
\\& = J(\omega_t)+ \beta_t\vbrac{\gradj, \phi_t^\top(s_t,a_t)\theta_t\phi_t(s_t,a_t)
-\E_{D_t}\Fbrac{\phi^\top_t(s,a)\theta_t\phi_t(s,a)}}
\\& \quad+\beta_t \vbrac{\gradj, \E_{D_t}\Fbrac{\phi^\top_t(s,a)
\brac{\theta_t-\thestar}\phi_t(s,a)}}
+\beta_t \vbrac{\gradj,\gradj}
\\&\quad+\beta_t \vbrac{\gradj, \E_{D_t}\Fbrac{\phi_t^\top(s,a)\thestar\phi_t(s,a)}-\gradj}
- \frac{L_J\beta_t^2}{2} \twonormsq{\phi_t^\top(s_t,a_t)\theta_t\phi_t(s_t,a_t)}
\\& \geq J(\omega_t)+ \beta_t\vbrac{\gradj, \phi_t^\top(s_t,a_t)\theta_t\phi_t(s_t,a_t)
-\E_{D_t}\Fbrac{\phi^\top_t(s,a)\theta_t\phi_t(s,a)}}
\\& \quad +\beta_t \twonormsq{\gradj}
-\frac{\beta_t}{4} \twonormsq{\gradj} 
-\beta_t \twonormsq{ \E_{D_t}\Fbrac{\phi^\top_t(s,a)
\brac{\theta_t-\thestar}\phi_t(s,a)}}
\\&\quad-\frac{\beta_t}{4}\twonormsq{\gradj} -\beta_t \twonormsq{ \E_{D_t}\Fbrac{\phi_t^\top(s,a)\thestar\phi_t(s,a)}-\gradj}
- \frac{L_J C^4_\phi B^2}{2} \beta_t^2
\\ & \myineq{=}{a} 
J(\omega_t)+ \beta_t\underbrace{\vbrac{\gradj, \phi_t^\top(s_t,a_t)\theta_t\phi_t(s_t,a_t)
-\E_{D_t}\Fbrac{\phi^\top_t(s,a)\theta_t\phi_t(s,a)}}}_{\text{part I}}
\\& \quad 
+\frac{\beta_t}{2} \twonormsq{\gradj} 
-\beta_t \twonormsq{ \E_{D_t}\Fbrac{\phi^\top_t(s,a)
\brac{\theta_t-\thestar}\phi_t(s,a)}}- \frac{L_J C^4_\phi B^2}{2} \beta_t^2
\\&\quad -\beta_t \twonormsq{ \E_{D_t}\Fbrac{\phi_t^\top(s,a)(\thestar-\thebar)\phi_t(s,a)}},\label{eq:ac:grad}
} where we write $ \bar \theta^*_{\omega_t}$ as $\thebar$ for convenience, $(a)$ follows from \Cref{eq:eq5} that 
$\nabla J(\omega_t)= \E_{D_t}\Fbrac{\phi_t^\top(s,a)\thebar\phi_t(s,a)} $.

We then bound part I in \Cref{eq:ac:grad}. Note that
\eqenv{
\big|\E &\fvbrac{\gradj, \phi_t^\top(s_t,a_t)\theta_t\phi_t(s_t,a_t)
-\E_{D_t}\Fbrac{\phi^\top_t(s,a)\theta_t\phi_t(s,a)}}\big|
\\& \leq \lbrac{\E \fvbrac{\gradj-\nabla J(\omega_{t-k}), \phi_t^\top(s_t,a_t)\theta_t\phi_t(s_t,a_t)
-\E_{D_t}\Fbrac{\phi^\top_t(s,a)\theta_t\phi_t(s,a)}} }
\\&\qquad+\lbrac{\E \fvbrac{\nabla J(\omega_{t-k}), \phi_t^\top(s_t,a_t)\theta_t\phi_t(s_t,a_t)
-\E_{D_t}\Fbrac{\phi^\top_t(s,a)\theta_t\phi_t(s,a)}} }
\\ &\leq 
\E\Fbrac{\twonorm{\gradj-\nabla J(\omega_{t-k})} \twonorm{\phi_t^\top(s_t,a_t) \theta_t\phi_t(s_t,a_t)-\E_{D_t}\Fbrac{\phi^\top_t(s,a)\theta_t\phi_t(s,a)}}}
\\ & \qquad+ \lbrac{\E \Fbrac{ \nabla^\top J(\omega_{t-k})\E \Fbrac{ \phi_t^\top(s_t,a_t)\theta_t\phi_t(s_t,a_t)
-\E_{D_t}\Fbrac{\phi^\top_t(s,a)\theta_t\phi_t(s,a)}|\mathcal{F}_{t-k} }}}
\\ & \myineq{\leq}{a} 
2 C^2_\phi B L_J \E\ftwonorm{\omega_t-\omega_{t-k}}
+ C_J \E\Fbrac{\tvnorm{\Prob\brac{s_t,a_t|\mathcal{F}_{t-k}}-D_t}
} C^2_\phi B
\\ & \myineq{\leq}{b} 2 C^4_\phi B^2 L_J \sum_{j=t-k}^{t-1} \beta_j
+ C_J C^2_\phi B\brac{C_\pi \sum_{j=t-k}^{t-1} \E\ftwonorm{\omega_t-\omega_j}+ m\rho^k }
\\& \leq 2 C^4_\phi B^2 L_J \sum_{j=t-k}^{t-1} \beta_j
+ C_J C^2_\phi B\brac{C_\pi C^2_\phi B \sum_{j=t-k}^{t-1} \sum_{i=j}^{t-1} \beta_i+ m\rho^k },\label{eq:ac:grad2}
} where $(a)$ is from the $L_J$-smoothness of $J$, and $(b)$ is from \Cref{lm:tv}.
On the other hand, from \Cref{prop:bound}, we can show that
\eqenv{
\twonormsq{ \E_{D_t}\Fbrac{\phi^\top_t(s,a)
\brac{\thebar-\thestar}\phi_t(s,a)}}
\leq C_\phi^4 \twonormsq{\thestar-\thebar} \leq C_\phi^4 \brac{\frac{C_{\text{gap}}m\rho^k}{\lambda_{\min}}}^2.\label{eq:ac:grad3}
}
Thus, combining \Cref{eq:ac:grad},\Cref{eq:ac:grad2} and \Cref{eq:ac:grad3} completes the proof,
\eqenv{
\E\Fbrac{J(\omega\tit)}&\geq \E\Fbrac{J(\omega_t)}
+ \frac{\beta_t}{2}\E \ftwonormsq{\gradj}
-\frac{L_J C_\phi^4 B^2}{2}  \beta_t^2-C_\phi^4 \beta_t \E\Fbrac{\twonormsq{\theta_t-\thestar}}
\\& -
C_\phi^4 \beta_t \brac{\frac{C_{\text{gap}}m\rho^k}{\lambda_{\min}}}^2
-2 C^4_\phi B^2 L_J \beta_t\sum_{j=t-k}^{t-1} \beta_j- C_J C^2_\phi B\beta_t\brac{C_\pi C^2_\phi B \sum_{j=t-k}^{t-1} \sum_{i=j}^{t-1} \beta_i+ m\rho^k }. 
}\end{proof}


\subsection{Bound on $|\eta_t-J(\omega_t)|$ in AC}\label{sec:ac:av}
In this section, we bound the error between $\eta_t$ and $J(\omega_t)$, where $J(\omega_t)=\lim_{N\to\infty }\E\Fbrac{\frac{1}{N} \sum_{t=0}^{N-1}R_t|\pi_{\omega_t}}$ is the average-reward for policy $\pi_{\omega_t}$.   
\begin{lemma}\label{lm:aceta}
If $\gamma_t-\gamma^2_t\geq  {\beta_t} $, then it holds that
    \eqenv{\label{eq:etaitera}
\E&\flnormsq{\eta\tit-J(\omega_{t+1})}
\leq \brac{1-\gamma_t}\E\flnormsq{\eta_t-J(\omega_t)}
+{C_\phi^4 B^2} \beta_t\E\ftwonormsq{\gradj}+ G^\eta_t,
} 
where \eqenv{\label{eq:33}
G^\eta_t& = 2\gamma_t \brac{R_{\max}^2 C_\pi C^2_\phi B  \sum_{j=t-k}^{t-1} \sum_{i=j}^{t-1} \beta_j 
+  R^2_{\max} m \rho^k+ R_{\max}^2 \sum_{j=t-k}^{t-1}\gamma_j
+{ R_{\max}C_J C_\phi^2 B \beta_t}+  R_{\max}C_J C_\phi^2 B \sum_{j=t-k}^{t-1}\beta_j}
\\& + R_{\max}^2 \gamma^2_t + C_J^2 C_\phi^4 B^2 \beta_t^2
+ 2R_{\max}L_J C^4_\phi B^2 \beta_t^2.
} 
\end{lemma}

\begin{proof}
Recall the update rule in \Cref{alg:AC}. Then  we have that
\eqenv{
\eta\tit -J(\omega_{t+1}) =\eta_t+ \gamma_t \brac{R_t-\eta_t}-J(\omega_{t})
+J(\omega_t) -J(\omega_{t+1}).
}

It then follows that
\eqenv{
\lnormsq{\eta\tit-J(\omega_{t+1}) }
&=\lnormsq{\brac{1-\gamma_t}\brac{\eta_t-J(\omega_t)}
+\gamma_t(R_t-J(\omega_t))+J(\omega_t)-J(\omega_{t+1})  }
\\  \leq& \brac{1-\gamma_t}^2\lnormsq{\eta_t-J(\omega_t)}
+\gamma_t^2 \lnormsq{R_t-J(\omega_t)} +\lnormsq{J(\omega_t)-J(\omega_{t+1})}
\\ +2&\gamma_t\underbrace{\brac{R_t-J(\omega_t)}\brac{J(\omega_t)-J(\omega_{t+1})}}_{\text{I}}
+ 2\gamma_t\brac{1-\gamma_t}
\underbrace{\brac{\eta_t-J(\omega_t)}\brac{ R_t-J(\omega_t)}}_{\text{II}}
\\ +2&\brac{1-\gamma_t}
\underbrace{\brac{\eta_t-J(\omega_t) }\brac{J(\omega_t)-J(\omega_{t+1})}}_{\text{III}}.\label{eq:ac:averr}
}
The term $\brac{J(\omega_t)-J(\omega_{t+1})}^2$ can be bounded by \Cref{Prop:lip}: 
\eqenv{
\lnorm{J(\omega_t)-J(\omega_{t+1})}
 \leq C_J\twonorm{\omega_t-\omega\tit}
 \leq C_J C^2_\phi B \beta_t .
}

Term I in \Cref{eq:ac:averr} can be bounded as follows:
\eqenv{
\lbrac{\E\Fbrac{\brac{R_t-J(\omega_t)}\brac{J(\omega_t)-J(\omega_{t+1})}}}
 &\leq \E\Fbrac{\lnorm{R_t-J(\omega_t)}\lnorm{J(\omega_t)-J(\omega_{t+1})}}
\\& \leq R_{\max}C_J\E\ftwonorm{\omega\tit-\omega_t}\\&
\leq  R_{\max}C_J C_\phi^2 B \beta_t.\label{eq:ac:averr1}
}

Term II in \Cref{eq:ac:averr} can be bounded as follows: 
\eqenv{\label{eq:ac:averr2}
&\lbrac{\E\Fbrac{\brac{\eta_t-J(\omega_t)}\brac{R_t-J(\omega_t)}}}
\\& \leq
\lbrac{\E\Fbrac{\brac{\eta_{t-k}-J(\omega_{t-k})}\brac{ R_t-J(\omega_t)}}}
+\lbrac{\E\Fbrac{\brac{\eta_t-\eta_{t-k}-J(\omega_t)+J(\omega_{t-k})}\brac{ R_t-J(\omega_t)}}}
\\& \myineq{\leq}{a} \lbrac{\E\Fbrac{\E\Fbrac{\brac{\eta_{t-k}-J(\omega_{t-k})}\brac{R_t-J(\omega_t)}|\mathcal{F}_{t-k}}}-\E_{D_t}\Fbrac{\brac{\eta_{t-k}-J(\omega_{t-k})}\brac{R(s,a)-J(\omega_t)} }}
 \\ &\qquad+ \E\Fbrac{\lnorm{ \eta_t-\eta_{t-k}-J(\omega_t)+J(\omega_{t-k})}\lnorm{R_t-J(\omega_t)}}
\\& \myineq{\leq}{b}  R^2_{\max }\E\Fbrac{\tvnorm{\mathbb{P}\brac{s_t,a_t|\mathcal{F}_{t-k}},D_t}}  
 + R_{\max}\E\Fbrac{\lnorm{\eta_t-\eta_{t-k} }+\lnorm{J(\omega_t)-J(\omega_{t-k}) }} 
\\ & \myineq{\leq}{c} 
{R_{\max}^2\brac{C_\pi\sum_{j=t-k}^{t-1}\E\ftwonorm{\omega_t-\omega_{j}}+m\rho^k } } 
+ R_{\max}\brac{R_{\max}\sum_{j=t-k}^{t-1}\gamma_j+C_J \E\ftwonorm{\omega_t-\omega_{t-k}} }
\\& \leq R_{\max}^2 C_\pi C^2_\phi B  \sum_{j=t-k}^{t-1} \sum_{i=j}^{t-1} \beta_j 
+  R^2_{\max} m \rho^k+ R_{\max}^2\sum_{j=t-k}^{t-1}\gamma_j
+  R_{\max}C_J C_\phi^2 B\sum_{j=t-k}^{t-1}\beta_j,
}
 where $(a)$ follows from $\E_{D_t}\Fbrac{R(s,a)-J(\omega_t)}=0 $, $(b)$ follows from that $0\leq\eta_t\leq R_{\max}$, $0\leq J(\omega_t)\leq R_{\max}$,  $0\leq R_t\leq R_{\max}$ and
 $(c)$ follows from  \Cref{lm:tv}.

Term III in \Cref{eq:ac:averr} can be bounded as follows: 
\eqenv{
|&\E\Fbrac{ \brac{\eta_t-J(\omega_t)}\brac{J(\omega_t)-J(\omega_{t+1})}}|
\\&  \myineq{\leq}{a}
\lbrac{\E\Fbrac{ \brac{\eta_t-J(\omega_t)}\brac{\nabla^\top J(\omega_t) (\omega\tit-\omega_t) }}}
\\& \qquad\qquad+
\lbrac{\E\Fbrac{\brac{\eta_t-J(\omega_t)}\brac{\omega\tit-\omega_t}^\top  \frac{\nabla^2J(\hat \omega_t)}{2}\brac{\omega\tit-\omega_t} }}
\\&  =\beta_t \lbrac{\E\Fbrac{\brac{\eta_t-J(\omega_t)} \nabla^\top J(\omega_t) (\phi_t^\top(s_t,a_t) \theta_t \phi_t(s_t,a_t)) }}
\\ &\qquad\qquad+
\beta^2_t\lbrac{\E\Fbrac{\brac{\eta_t-J(\omega_t)}\brac{\phi_t^\top(s_t,a_t) \theta_t \phi_t(s_t,a_t)}^\top \frac{\nabla^2J(\hat \omega_t)}{2}\brac{\phi_t^\top(s_t,a_t) \theta_t \phi_t(s_t,a_t)}  }}
\\ &  \myineq{\leq}{b}
\frac{\beta_t}{2}\E\flnormsq{\eta_t-J(\omega_t)}+
\frac{ C^4_\phi B^2}{2}\beta_t\E\ftwonormsq{\gradj}
+ R_{\max}L_J C^4_\phi B^2 \beta_t^2 ,\label{eq:ac:averr3}
}
where $(a)$ follows from the Lagrange's Mean Value Theorem and \Cref{Prop:lip} 
for some $ \hat \omega_t=\lambda \omega_t+(1-\lambda)\omega_{t+1}$ with $\lambda\in [0,1]$;
 $(b)$ follows from $\langle a, b\rangle \leq \frac{\|a\|^2+\|b\|^2}{2}$ and \Cref{Prop:lip}.

Combining \Cref{eq:ac:averr}, \Cref{eq:ac:averr1}, \Cref{eq:ac:averr2} and \Cref{eq:ac:averr3} implies
\eqenv{\label{eq:ac_eta}
\E&\flnormsq{\eta\tit-J(\omega_{t+1})}
 \leq \brac{(1-\gamma_t)^2+ \beta_t}\E\flnormsq{\eta_t-J(\omega_t)}
+\beta_t C_\phi^4 B^2 \E\ftwonormsq{\gradj}
\\ &
+ 2 \gamma_t \brac{R_{\max}^2 C_\pi C^2_\phi B  \sum_{j=t-k}^{t-1} \sum_{i=j}^{t-1} \beta_i 
+  R^2_{\max} m \rho^k+ R_{\max}^2\sum_{j=t-k}^{t-1}\gamma_j
+ R_{\max}C_J C_\phi^2 B \beta_t}
\\& + R_{\max}^2 \gamma^2_t + C_J^2 C_\phi^4 B^2 \beta_t^2
+ 2 { R_{\max}C_J C_\phi^2 B \gamma_t\sum_{j=t-k}^{t-1}\beta_j}
+ 2 R_{\max}L_J C^4_\phi B^2 \beta_t^2,
}


which completes the proof.
\end{proof}


\subsection{Tracking Error Analysis of AC}
In this section, we bound the tracking error $\|\theta_t-\thestar\|_2$. 
Recall that we write $\theta^*_t=\theta^*_{\omega_t} $ and 
$ \theta^*_\omega $ is the solution to \Cref{eq:kstep}, i.e., 
     \eqenv{
     \E_{D_{\pi_{\omega}}}&
\Fbrac{\phi_{\omega}^\top(s,a)\brac{\mathcal{T}_{\pi_\omega}^{(k)}\brac{ \phi^\top_{\omega}(s,a) \theta^*_{\omega}} -\phi^\top_{\omega}(s,a) \theta^*_{\omega}}}
=\mathbf{0}.
     }

We then present a recursive bound on the tracking error in the following lemma.
\begin{lemma}\label{lm:actrackingerror}
Set the step sizes such that
\eqenv{\label{cons:barlambdamin}
\frac{\bar \lambda_{\min} \alpha_t}{2}\geq \frac{B C^3_\phi\brac{C_\phi L_\pi+2L_\phi }+\lambda_{\min}(L_J+ 2 C_\phi)+2\lambda^2_{\min}}{\lambda^2_{\min}}\beta_t,} then it holds that
\eqenv{\label{eq:thetaitera}
\E \ftwonormsq{ \theta\tit-\thestara}
&\leq \brac{1-\frac{\bar \lambda_{\min}\alpha_t}{2}}
  \E \ftwonormsq{\theta_t-\thestar}
  +\frac{L_J\lambda_{\min}+B C^3_\phi\brac{C_\phi L_\pi+2L_\phi }}{\lambda^2_{\min}}\beta_t \E\ftwonormsq{\gradj}
    \\& \qquad+ \frac{(k+1)^2C^2_\phi }{\bar\lambda_{\min}}\alpha_t \E\flnormsq{\eta_t-J(\omega_t) }
   +G_t^\theta,
}
where 
\eqenv{\label{vari:gthetat}
G^\theta_t&= (k+1)^2 U_\delta^2 C^2_\phi \alpha_t^2+ \frac{4(k+1) B C^3_\phi U_\delta L_J}{\lambda_{\min}}\beta_t\sum_{j=t-k}^{t-1}\alpha_t 
    +2(k+1) C^3_\phi B U_\delta C_\Theta \alpha_t\beta_t 
    \\& + \brac{2B C_\phi L_J(C_\phi L_\pi+2L_\phi)+ 2 B L_\Theta \lambda_{\min}+ C_\Theta \lambda_{\min}  }\frac{4 B^2 C^4_\phi}{\lambda^2_{\min}}\beta_t\sum_{j=t-k}^t \beta_j+ 2 C^4_\phi B^2 C_\Theta^2 \beta_t^2
     \\& +\frac{4B C^2_\phi L_J \brac{ 3 C_{\text{gap}}+  \lambda_{\min} }}{\lambda^2_{\min}} m \rho^k\beta_t  + \frac{4(k+1) C_\phi U_\delta C_{\text{gap}}m\rho^k}{\lambda_{\min}}\alpha_t+\brac{\frac{2}{\beta_t}+8}\brac{\frac{C_{\text{gap}}m\rho^k}{\lambda_{\min}} }^2
     \\& +\frac{4B^3 C_\pi C^4_\phi L_J}{\lambda_{\min}}\beta_t \sum_{j=t-k}^{t-1} \sum_{i=j}^{t-1} \beta_i+ 2 \alpha_t G^\delta_t.
}
\end{lemma}

\begin{proof}
From \Cref{alg:AC}, it holds that
\eqenv{
\twonormsq{\theta\tit-\thestara}&=
\twonormsq{\Pi_B\brac{\theta_t+\alpha_t \delta_t z_t}-\thestara}
\\&\myineq{\leq}{a}\twonormsq{ \theta_t+\alpha_t \delta_t z_t-\thestara}
\\ & = \twonormsq{\theta_t+\alpha_t \delta_t z_t-\thestar+\thestar-\thestara}
\\ & = \twonormsq{\theta_t-\thestar}+ \alpha_t^2 \twonormsq{\delta_t z_t}+\twonormsq{ \thestar-\thestara}
+ 2 \alpha_t\vbrac{\theta_t-\thestar, \delta_t z_t}
\\&\quad + 2\alpha_t \vbrac{\delta_t z_t, \thestar-\thestara }
+ 2 \vbrac{\theta_t-\thestar, \thestar-\thestara},
}
where $(a)$ follows from the fact $\twonorm{\Pi_B(x)-y}\leq \twonorm{x-y} $ when $\twonorm{y}\leq B $ and $\twonorm{\theta^*_{t+1}}\leq B$.

Taking expectations on both sides further implies that
\eqenv{\label{eq:72}
\E&\ftwonormsq{\theta\tit-\thestara}  \leq \E\ftwonormsq{\theta_t-\thestar}+ \alpha_t^2 \E\ftwonormsq{\delta_t z_t}+\E\ftwonormsq{ \thestar-\thestara}
\\&\quad+ 2 \alpha_t\underbrace{\E\fvbrac{\theta_t-\thestar, \delta_t z_t}}_{\text{I}}
 + 2\alpha_t \underbrace{\E \fvbrac{\delta_t z_t, \thestar-\thestara }}_{\text{II}}
+ 2 \underbrace{\E\fvbrac{\theta_t-\thestar, \thestar-\thestara}}_{\text{III}}.
}

Firstly, we can get that
\eqenv{
\alpha_t^2 \E\ftwonormsq{\delta_t z_t}\leq \alpha_t^2 (k+1)^2 C_\phi^2 U^2_\delta.
}
The term $\twonorm{\thestar-\thestara}$ can be bounded as follows: 
\eqenv{
\twonorm{\thestar-\thestara}
& =\twonorm{\thebar-\thebara+ \thestar-\thebar-\thestara+\thebara}
\\& \leq \twonorm{\thebar-\thebara}+\twonorm{ \thestar-\thebar}+ \twonorm{\thestara-\thebara}
\\& \myineq{\leq}{a} \twonorm{\thebar-\thebara}+\frac{C_{\text{gap}}m\rho^k}{\lambda_{\min}}
+\frac{C_{\text{gap}} m\rho^k}{\lambda_{\min}}
\\& \myineq{\leq}{b} C_\Theta\twonorm{\omega_t-\omega\tit}+\frac{2C_{\text{gap}}m\rho^k}{\lambda_{\min}}
\\ & = \beta_t C_\Theta \twonorm{\phi_t^\top(s_t,a_t)\theta_t \phi_t(s_t,a_t)  }+\frac{2C_{\text{gap}}m\rho^k}{\lambda_{\min}}
\\ & \leq C_\Theta C_\phi^2 B \beta_t+\frac{2C_{\text{gap}}m\rho^k}{\lambda_{\min}},
\label{eq:ac:te1}
}
where $(a)$ follows from \Cref{prop:bound}; and $(b)$ follows from Lemma \ref{lm:acnacinnersmooth}.
\Cref{eq:ac:te1} further implies that
\eqenv{
\E\ftwonormsq{\thestar-\thestara }
\leq  2C^2_\Theta C_\phi^4 B^2 \beta^2_t+\frac{8 C^2_{\text{gap}}m^2\rho^{2k}}{\lambda^2_{\min}}.
}

By Lemma \ref{lm:consterm}, we can bound  term I in \Cref{eq:72} as follows: 
\eqenv{
&\E\fvbrac{\theta_t-\thestar, \delta_t z_t}
\leq - \frac{\bar \lambda_{\min}}{2} \E \ftwonormsq{\theta_t-\thestar}
+\frac{(k+1)^2 C^2_\phi}{2\bar \lambda_{\min}}\E\flnormsq{J(\omega_t)-\eta_t}+G^\delta_t.\label{eq:ac:teI}
}

For term II in \Cref{eq:72}, we have that
\eqenv{
\E \fvbrac{ \delta_t z_t, \thestar-\thestara}
& \leq \E\Fbrac{\twonorm{\delta_t z_t} \twonorm{ \thestar-\thestara}}
\\& \leq \E\Fbrac{\twonorm{\delta_t z_t} \brac{ \twonorm{ \thebar-\thebara}+\twonorm{ \thestar-\thebar}+\twonorm{ \thestara-\thebara}}}
\\& \leq (k+1) C_\phi U_\delta\brac{C_\Theta C_\phi^2 B \beta_t+\frac{2C_{\text{gap}}m\rho^k}{\lambda_{\min}}  }
\\ &= (k+1) C^3_\phi B U_\delta C_\Theta \beta_t
+ \frac{2(k+1) C_\phi U_\delta C_{\text{gap}}m\rho^k}{\lambda_{\min}},\label{eq:ac:teII}
}
where the last inequality is from \Cref{lm:acnacinnersmooth} and \Cref{prop:bound}.

To convenience, we rewrite $F_t=F_{\omega_t}$. To bound term III in \Cref{eq:72}, note that
\eqenv{
&\E\fvbrac{\theta_t-\thestar, \thestar-\thestara }
 \\& = \E\fvbrac{\theta_t-\thestar, \thebar-\thebara } +\E\fvbrac{\theta_t-\thestar, \thestar-\thebar }
 +\E\fvbrac{\theta_t-\thestar, \thestara-\thebara }
 \\&\leq \E\fvbrac{\theta_t-\thestar, F_t^{-1}\gradj-F_{t+1}^{-1}\nabla J(\omega_{t+1}) }
 +\frac{\beta_t}{2}\E\Fbrac{ \twonormsq{ \theta_t-\thestar}}
 +\frac{1}{2\beta_t}\E\Fbrac{ \twonormsq{ \thestar-\thebar}}
 \\& \qquad +\frac{\beta_t}{2}\E\Fbrac{ \twonormsq{ \theta_t-\thestar}}
 +\frac{1}{2\beta_t}\E\Fbrac{ \twonormsq{ \thestara-\thebara}}
 \\& \myineq{\leq}{a} \E\fvbrac{\theta_t-\thestar,F_{t+1}^{-1}\brac{\nabla J(\omega_t)-\nabla J(\omega_{t+1})} }+\E\fvbrac{\theta_t-\thestar, \brac{F_t^{-1}-F^{-1}_{t+1}}\gradj}
 \\& \qquad +\beta_t \E\ftwonormsq{\theta_t-\thestar} +\frac{1}{\beta_t}\brac{\frac{C_{\text{gap}}m\rho^k}{\lambda_{\min}} }^2
 \\&\myineq{=}{b} \E\fvbrac{\theta_t-\thestar,F_{t+1}^{-1}\nabla^2 J(\hat\omega_t)\brac{\omega_t-\omega_{t+1}} }+\E\fvbrac{\theta_t-\thestar, \brac{F_t^{-1}-F^{-1}_{t+1}}\gradj}
 \\& \qquad +\beta_t \E\ftwonormsq{\theta_t-\thestar} +\frac{1}{\beta_t}\brac{\frac{C_{\text{gap}}m\rho^k}{\lambda_{\min}} }^2
 \\&\myineq{\leq}{c} \E\fvbrac{\theta_t-\thestar,F_{t+1}^{-1}\nabla^2 J(\hat\omega_t)\brac{\omega_t-\omega_{t+1}} }+\frac{1}{\lambda_{\min}^2}\E\Fbrac{\twonorm{\theta_t-\thestar}\twonorm{F_t-F_{t+1}}\twonorm{\gradj}}
 \\& \qquad +\beta_t \E\ftwonormsq{\theta_t-\thestar} +\frac{1}{\beta_t}\brac{\frac{C_{\text{gap}}m\rho^k}{\lambda_{\min}} }^2
 \\& \myineq{\leq}{d}-\beta_t\E\fvbrac{\theta_t-\thestar,F_{t+1}^{-1}\nabla^2 J(\hat\omega_t)\brac{\phi_t^\top(s_t,a_t)\theta_t\phi_t(s_t,a_t)} }
 +\frac{B C^3_\phi\brac{C_\phi L_\pi+2L_\phi }}{2\lambda^2_{\min}}\beta_t\E\ftwonormsq{\theta_t-\thestar } 
 \\&\qquad+ \frac{B C^3_\phi\brac{C_\phi L_\pi+2L_\phi }}{2\lambda^2_{\min}}\beta_t\E\ftwonormsq{\gradj }+\beta_t \E\ftwonormsq{\theta_t-\thestar} +\frac{1}{\beta_t}\brac{\frac{C_{\text{gap}}m\rho^k}{\lambda_{\min}} }^2
 \\& =\underbrace{-\beta_t\E\fvbrac{\theta_t-\thestar,F_{t+1}^{-1}\nabla^2 J(\hat\omega_t)\E_{D_t}\Fbrac{\phi_t^\top(s,a)\thebar\phi_t(s,a)} } }_{(i)}
 \\& \qquad -\underbrace{\beta_t\E\fvbrac{\theta_t-\thestar,F_{t+1}^{-1}\nabla^2 J(\hat\omega_t)\E_{D_t}\Fbrac{\phi_t^\top(s,a)(\thestar-\thebar)\phi_t(s,a)} } }_{(ii)}
 \\&\qquad -\underbrace{\beta_t\E\fvbrac{\theta_t-\thestar,F_{t+1}^{-1}\nabla^2 J(\hat\omega_t)\E_{D_t}\Fbrac{\phi_t^\top(s,a)(\theta_t-\thestar)\phi_t(s,a)} } }_{(iii)}
 \\& \qquad-\underbrace{\beta_t\E\fvbrac{\theta_t-\thestar,F_{t+1}^{-1}\nabla^2 J(\hat\omega_t)\brac{\phi_t^\top(s_t,a_t)\theta_t\phi_t(s_t,a_t)-\E_{D_t}\Fbrac{\phi_t^\top(s,a)\theta_t\phi_t(s,a)} } }}_{(iv)}
 \\& \qquad +\frac{B C^3_\phi\brac{C_\phi L_\pi+2L_\phi }}{2\lambda^2_{\min}}\beta_t\E\ftwonormsq{\gradj }+\frac{B C^3_\phi\brac{C_\phi L_\pi+2L_\phi }+2\lambda^2_{\min}}{2\lambda^2_{\min}}\beta_t \E\ftwonormsq{\theta_t-\thestar} 
 \\& \qquad+\frac{1}{\beta_t}\brac{\frac{C_{\text{gap}}m\rho^k}{\lambda_{\min}} }^2.
\label{eq:ac:teIII}
}
where $(a)$ follows from \Cref{prop:bound}, $(b)$ follows from Lagrange's Mean Value for some $\lambda\in[0,1]$, such that $\hat\omega_t =\lambda \omega_t+(1-\lambda)\omega_{t+1}$, $(c)$ 
follows from the facts that for positive definite matrices $X$ and $Y$, 
    \eqenv{\label{eq:invinequ}
    \twonorm{X^{-1}-Y^{-1} }&\leq 
    \twonorm{X^{-1}\brac{X-Y }Y^{-1} }
    \\& \leq \twonorm{X^{-1}}\twonorm{X-Y }\twonorm{Y^{-1} },
    }
and $(d)$ is from that 
\eqenv{\label{eq:60}
   \twonorm{F_{t+1}-F_t}&=\twonorm{ \E_{D_{t+1}}\Fbrac{\phi_{t+1}(s,a)\phi^\top_{t+1}(s,a)}-\E_{D_{t}}\Fbrac{\phi_{t}(s,a)\phi^\top_{t}(s,a)}}
   \\\leq &\twonorm{ \E_{D_{t+1}}\Fbrac{\phi_{t+1}(s,a)\phi^\top_{t+1}(s,a)}-\E_{D_{t}}\Fbrac{\phi_{t+1}(s,a)\phi^\top_{t+1}(s,a)}}
   \\& \qquad+\twonorm{ \E_{D_{t}}\Fbrac{\phi_{t+1}(s,a)\phi^\top_{t+1}(s,a)}-\E_{D_{t}}\Fbrac{\phi_{t}(s,a)\phi^\top_{t}(s,a)}}
   \\\leq &C^2_\phi\tvnorm{D_{t+1}-D_t }+ \E\Fbrac{\brac{\twonorm{ \phi_t(s,a)}+\twonorm{ \phi_{t+1}(s,a)}}\twonorm{\phi_t(s,a)-\phi_{t+1}(s,a)} }
   \\ \myineq{\leq}{a} &C^2_\phi L_\pi \twonorm{\omega_{t+1}-\omega_t}+ 2C_\phi L_\phi\twonorm{\omega_{t+1}-\omega_t}
   \\ =& \brac{C^2_\phi L_\pi+ 2C_\phi L_\phi } \beta_t\twonorm{\phi^\top_t(s_t,a_t)\theta_t\phi_t(s_t,a_t)}
   \\ \leq & B C^3_\phi\brac{C_\phi L_\pi+2L_\phi }\beta_t
,}where $(a)$ follows from \citep{zou2019finite} and Theorem 1 in \citep{li2021faster}, where $ L_\pi=\frac{1}{2}C_\pi 
\brac{1+\myceil{\log m^{-1}}+\frac{1}{1-\rho}} $, and
    \eqenv{
        \tvnorm{D_{t+1}-D_t }\leq
        L_\pi \twonorm{\omega_{t+1}-\omega_t }.
    }

We then consider the term $(i)$,
\eqenv{
(i)&\leq \beta_t\E \Fbrac{\twonorm{\theta_t-\thestar }\twonorm{F_{t+1}^{-1} }\twonorm{\nabla^2 J(\hat\omega_t) }\twonorm{ \gradj} }
\\& \leq \frac{L_J\beta_t}{\lambda_{\min}} \brac{\frac{1}{2}\E\ftwonormsq{\theta_t-\thestar} +\frac{1}{2}\E \ftwonormsq{\gradj }},
} where the last inequality follows from \Cref{Prop:lip}.

Next, we consider the term $(ii)$, 
\eqenv{
(ii)&\leq \beta_tC^2_\phi \E \Fbrac{\twonorm{\theta_t-\thestar }\twonorm{F_{t+1}^{-1} }\twonorm{\nabla^2 J(\hat\omega_t) }\twonorm{ \thestar-\thebar}  }\\& \leq \frac{2 B C^2_\phi C_{\text{gap}}L_J m\rho^k\beta_t}{\lambda^2_{\min}},
}where the last inequality follows from \Cref{Prop:lip} and \Cref{prop:bound}.

Then, consider the term $(iii)$, 
\eqenv{
    (iii)&\leq \beta_t C^2_\phi\E \Fbrac{\twonorm{\theta_t-\thestar }\twonorm{F_{t+1}^{-1} }\twonorm{\nabla^2 J(\hat\omega_t) }\twonorm{ \theta_t-\thestar}  }
    \\& \leq \frac{C^2_\phi L_J\beta_t}{\lambda_{\min}}\E\ftwonormsq{\theta_t-\thestar}.
}

Consider the term $(iv)$,
\eqenv{\label{eq:eq64}
    &(iv)= -\beta_t\E\fvbrac{(F_{t+1}^{-1}\nabla^2 J(\hat\omega_t) )^\top(\theta_t-\thestar), {\phi_t^\top(s_t,a_t)\theta_t\phi_t(s_t,a_t)}-\E_{D_t}\Fbrac{\phi_t^\top(s,a)\theta_t\phi_t(s,a)} }
    \\& = -\beta_t\E[(F_{t+1}^{-1}\nabla^2 J(\hat\omega_t) )^\top(\theta_t-\thestar)-(F_{t-k}^{-1}\nabla^2 J(\hat\omega_{t-k-1}) )^\top(\theta_{t-k}-\theta^*_{t-k}),
    \\&\qquad\quad {\phi_t^\top(s_t,a_t)\theta_t\phi_t(s_t,a_t)}-\E_{D_t}\Fbrac{\phi_t^\top(s,a)\theta_t\phi_t(s,a)} \rangle]
    \\ & \quad- \beta_t\E\fvbrac{(F_{t-k}^{-1}\nabla^2 J(\hat\omega_{t-k-1}) )^\top(\theta_{t-k}-\theta^*_{t-k}), {\phi_t^\top(s_t,a_t)\theta_t\phi_t(s_t,a_t)}-\E_{D_t}\Fbrac{\phi_t^\top(s,a)\theta_t\phi_t(s,a)} }
    \\ & \leq 2C^2_\phi B\beta_t \E\Fbrac{\twonorm{ (F_{t+1}^{-1}\nabla^2 J(\hat\omega_t) )^\top(\theta_t-\thestar)-(F_{t-k}^{-1}\nabla^2 J(\hat\omega_{t-k-1}) )^\top(\theta_{t-k}-\theta^*_{t-k})} }
    \\&\quad + \beta_t \E[\langle(F_{t-k}^{-1}\nabla^2 J(\hat\omega_{t-k-1}) )^\top(\theta_{t-k}-\theta^*_{t-k}), {\phi_t^\top(s_t,a_t)\theta_t\phi_t(s_t,a_t)}-\E_{D_t}\Fbrac{\phi_t^\top(s,a)\theta_t\phi_t(s,a)}\rangle]
     \\&\myineq{\leq}{a} 2C^2_\phi B \beta_t \E\Fbrac{\twonorm{ (F_{t+1}^{-1}\nabla^2 J(\hat\omega_t) )^\top(\theta_t-\thestar)-(F_{t-k}^{-1}\nabla^2 J(\hat\omega_{t-k-1}) )^\top(\theta_{t-k}-\theta^*_{t-k})}}
    \\& \quad+ \beta_t \E\Fbrac{\frac{2B^2 C^2_\phi L_J}{\lambda_{\min}}\tvnorm{\Prob\brac{s_t,a_t|\mathcal{F}_{t-k} }-D_t} }
    \\& \myineq{\leq}{b}2C^2_\phi B \beta_t \E\Fbrac{\twonorm{ (F_{t+1}^{-1}\nabla^2 J(\hat\omega_t) )^\top(\theta_t-\thestar)-(F_{t-k}^{-1}\nabla^2 J(\hat\omega_{t-k-1}) )^\top(\theta_{t-k}-\theta^*_{t-k})} }
    \\& \quad+  \frac{2 B^2 C^2_\phi L_J\beta_t}{\lambda_{\min}}\brac{ C_\pi \sum_{j=t-k}^{t-1} \E \ftwonorm{\omega_t-\omega_j}+ m\rho^k}
     \\& \leq 2C^2_\phi B \beta_t \E\Fbrac{\twonorm{ (F_{t+1}^{-1}\nabla^2 J(\hat\omega_t) )^\top(\theta_t-\thestar)-(F_{t-k}^{-1}\nabla^2 J(\hat\omega_{t-k-1}) )^\top(\theta_{t-k}-\theta^*_{t-k})}}
    \\& \quad+  \frac{2 B^2 C^2_\phi L_J\beta_t}{\lambda_{\min}}\brac{ C_\pi C^2_\phi B \sum_{j=t-k}^{t-1} \sum_{i=j}^{t-1} \beta_i+ m\rho^k},
}where $(a)$ follows from \Cref{Prop:lip} and $(b)$ follows from \Cref{lm:tv}.

Consider the first term in \Cref{eq:eq64} and  we have that
\eqenv{
    &\twonorm{ (F_{t+1}^{-1}\nabla^2 J(\hat\omega_t) )^\top(\theta_t-\thestar)-(F_{t-k}^{-1}\nabla^2 J(\hat\omega_{t-k-1}) )^\top(\theta_{t-k}-\theta^*_{t-k})}
    \\& \leq \twonorm{\brac{F_{t+1}^{-1}-F_{t-k}^{-1} } \nabla^2 J(\hat\omega_t) )^\top(\theta_t-\thestar)}+\twonorm{F_{t-k}^{-1} \brac{ \nabla^2 J(\hat \omega_t)-\nabla^2 J(\hat \omega_{t-k-1})}(\theta_t-\thestar)}
    \\&\qquad +\twonorm{ (F_{t-k}^{-1}\nabla^2 J(\hat\omega_{t-k-1}) )^\top\brac{\theta_t-\theta_{t-k} - \theta^*_{t}+\theta^*_{t-k}}}
    \\& \leq 2B L_J \twonorm{F_{t+1}^{-1}-F_{t-k}^{-1} }+ \frac{2B}{\lambda_{\min}}\twonorm{\nabla^2 J(\hat \omega_t)-\nabla^2 J(\hat \omega_{t-k-1}) }
    \\& \qquad+\frac{L_J}{\lambda_{\min}}\brac{\twonorm{\theta_t-\theta_{t-k} }+\twonorm{\theta^*_{t}-\theta^*_{t-k} }} 
    \\& \myineq{\leq}{a} \frac{2 B^2 L_J C^3_\phi\brac{C_\phi L_\pi+2L_\phi }}{\lambda^2_{\min}}\sum_{j=t-k}^t\beta_j+ \frac{2B}{\lambda_{\min}}\twonorm{\nabla^2 J(\hat \omega_t)-\nabla^2 J(\hat \omega_{t-k-1}) }
    \\& \qquad+\frac{L_J}{\lambda_{\min}}\brac{\twonorm{\theta_t-\theta_{t-k} }+\twonorm{\bar\theta^*_{t}-\bar\theta^*_{t-k} }+\twonorm{\bar\theta^*_{t}-\theta^*_{t} }+\twonorm{\theta^*_{t-k}-\bar\theta^*_{t-k} } }
    \\& \myineq{\leq}{b} \frac{2 B^2 L_J C^3_\phi\brac{C_\phi L_\pi+2L_\phi }}{\lambda^2_{\min}}\sum_{j=t-k}^t\beta_j+ \frac{2B L_\Theta}{\lambda_{\min}}\twonorm{ \hat\omega_{t}- \hat \omega_{t-k-1} }
    \\& \qquad+\frac{L_J}{\lambda_{\min}}\brac{\twonorm{\theta_t-\theta_{t-k} }+C_\Theta \twonorm{\omega_{t}-\omega_{t-k} }+\frac{2C_{\text{gap}}m\rho^k}{\lambda_{\min}}}
    \\& \leq \frac{2 B^2 L_J C^3_\phi\brac{C_\phi L_\pi+2L_\phi }}{\lambda^2_{\min}}\sum_{j=t-k}^t\beta_j+ \frac{2B^2 C^2_\phi L_\Theta}{\lambda_{\min}} \sum_{j=t-k}^t\beta_t 
    \\& \qquad+\frac{L_J}{\lambda_{\min}}\brac{(k+1) C_\phi U_\delta \sum_{j=t-k}^{t-1}\alpha_j + B C^2_\phi C_\Theta \sum_{j=t-k}^{t-1}\beta_j+\frac{2C_{\text{gap}}m\rho^k}{\lambda_{\min}}}
    \\& = \frac{(k+1) C_\phi U_\delta L_J}{\lambda_{\min}}\sum_{j=t-k}^{t-1}\alpha_j + \frac{2C_{\text{gap}}L_J m\rho^k}{\lambda^2_{\min}}
    \\&\qquad + \frac{B C^2_\phi}{\lambda^2_{\min}}\brac{2B C_\phi L_J(C_\phi L_\pi+2L_\phi)+ 2 B L_\Theta \lambda_{\min}+C_\Theta \lambda_{\min}  }\sum_{j=t-k}^t \beta_j, 
}where $(a)$ follows from \Cref{eq:invinequ} and \Cref{eq:60} and $(b)$ follows from \Cref{lm:acinnersmooth}.

Therefore, the term $(iv)$ can be bounded as:
\eqenv{
 (iv)& \leq  \frac{2(k+1) B C^3_\phi  U_\delta L_J}{\lambda_{\min}}\beta_t\sum_{j=t-k}^{t-1}\alpha_j
 \\& + \frac{2B C^2_\phi L_J\brac{2C_{\text{gap}}+B\lambda_{\min} }}{\lambda^2_{\min}}m\rho^k\beta_t
 +\frac{2B^3 C_\pi C^4_\phi L_J}{\lambda_{\min}}\beta_t \sum_{j=t-k}^{t-1} \sum_{i=j}^{t-1} \beta_i
   \\&  + \frac{2 B^2 C^4_\phi}{\lambda^2_{\min}}\brac{2B C_\phi L_J(C_\phi L_\pi+2L_\phi)+ 2 B L_\Theta \lambda_{\min}+C_\Theta \lambda_{\min}  }\beta_t\sum_{j=t-k}^t \beta_j. 
}

Combining the above bounds on terms $(i),(ii),(iii)$ and $(iv)$, 
we have that
\eqenv{
\E&\fvbrac{\theta_t-\thestar, \thestar-\thestara} 
\leq \frac{B C^3_\phi\brac{C_\phi L_\pi+2L_\phi }+\lambda_{\min}(L_J+ 2 C_\phi)+2\lambda^2_{\min}}{2\lambda^2_{\min}}\beta_t \E \ftwonormsq{\theta_t-\thestar}
\\&+\frac{L_J\lambda_{\min}+B C^3_\phi\brac{C_\phi L_\pi+2L_\phi }}{2\lambda^2_{\min}}\beta_t \E\ftwonormsq{\gradj}
 + \frac{2B C^2_\phi L_J \brac{ 3 C_{\text{gap}}+  \lambda_{\min} }}{\lambda^2_{\min}} m \rho^k\beta_t 
 \\&+\frac{2(k+1) B C^3_\phi U_\delta L_J}{\lambda_{\min}}\beta_t\sum_{j=t-k}^{t-1}\alpha_j  +\frac{2B^3 C_\pi C^4_\phi L_J}{\lambda_{\min}}\beta_t \sum_{j=t-k}^{t-1} \sum_{i=j}^{t-1} \beta_i+\frac{1}{\beta_t}\brac{\frac{C_{\text{gap}}m\rho^k}{\lambda_{\min}} }^2
  \\& + \frac{2 B^2 C^4_\phi}{\lambda^2_{\min}}\brac{2B C_\phi L_J(C_\phi L_\pi+2L_\phi)+ 2 B L_\Theta \lambda_{\min}+C_\Theta \lambda_{\min}  }\beta_t\sum_{j=t-k}^t \beta_j.
}
This bounds term III in \Cref{eq:72}.

Plugging bounds on terms I, II, III in \Cref{eq:72} further implies that 

\eqenv{ 
    E&\ftwonormsq{\theta\tit-\thestara} \leq  \brac{ 1-\bar\lambda_{\min}\alpha_t+ \frac{B C^3_\phi\brac{C_\phi L_\pi+2L_\phi }+\lambda_{\min}(L_J+ 2 C_\phi)+2\lambda^2_{\min}}{\lambda^2_{\min}}\beta_t}\E\ftwonormsq{\theta_t-\thestar}
    \\&+\frac{L_J\lambda_{\min}+B C^3_\phi\brac{C_\phi L_\pi+2L_\phi } }{\lambda^2_{\min}}\beta_t \E\ftwonormsq{\gradj}
     + \frac{(k+1)^2C^2_\phi }{\bar\lambda_{\min}}\alpha_t \E\flnormsq{\eta_t-J(\omega_t) }
   \\& + (k+1)^2 U_\delta^2 C^2_\phi \alpha_t^2+ \frac{4(k+1) B C^3_\phi U_\delta L_J}{\lambda_{\min}}\beta_t\sum_{j=t-k}^{t-1}\alpha_j 
    +2(k+1) C^3_\phi B U_\delta C_\Theta \alpha_t\beta_t 
    \\& + \brac{2B C_\phi L_J(C_\phi L_\pi+2L_\phi)+ 2 B L_\Theta \lambda_{\min}+ C_\Theta \lambda_{\min}  }\frac{4 B^2 C^4_\phi}{\lambda^2_{\min}}\beta_t\sum_{j=t-k}^t \beta_j+ 2 C^4_\phi B^2 C_\Theta^2 \beta_t^2
     \\& +\frac{4B C^2_\phi L_J \brac{ 3 C_{\text{gap}}+  \lambda_{\min} }}{\lambda^2_{\min}} m \rho^k\beta_t  + \frac{4(k+1) C_\phi U_\delta C_{\text{gap}}m\rho^k}{\lambda_{\min}}\alpha_t+\brac{\frac{2}{\beta_t}+8}\brac{\frac{C_{\text{gap}}m\rho^k}{\lambda_{\min}} }^2
     \\& +\frac{4B^3 C_\pi C^4_\phi L_J}{\lambda_{\min}}\beta_t \sum_{j=t-k}^{t-1} \sum_{i=j}^{t-1} \beta_i+ 2 \alpha_t G^\delta_t.
}
Set the stepsize such that $\frac{\bar \lambda_{\min} \alpha_t}{2}\geq \frac{B C^3_\phi\brac{C_\phi L_\pi+2L_\phi }+\lambda_{\min}(L_J+ 2 C_\phi)+2\lambda^2_{\min}}{\lambda^2_{\min}}\beta_t$, and the proof ends.
\end{proof}

\subsection{Sample Complexity of AC}\label{sec:b.4}
We first present our proof of \Cref{thm1}. For convenience, we set $\alpha_t=\alpha$,  
$\beta_t=\beta$, $\gamma_t=\gamma$ in following.

\begin{proposition}(Restatement of \Cref{thm1})
With the constant step sizes such that $k = \mathcal{O}\brac{\log T}$ and $\gamma\geq  \alpha\geq   \beta$, the tracking error of the AC algorithm in  \Cref{alg:AC} can be bounded as follows:
\eqenv{
\frac{1}{T}\sum _{t=0}^{T-1} \E\Fbrac{\twonormsq{\theta^*_t-\theta_t}} 
\leq & \brac{\frac{c_\alpha\beta}{\alpha}+\frac{c_\eta \beta}{\gamma}}\frac{1}{T}\sum _{t=0}^{T-1} \E \ftwonormsq{\gradj}
+\mathcal{O}\brac{\frac{1}{T \alpha}}
+\mathcal{O}\brac{\frac{\log^2 T}{T \gamma}}
\\&+\mathcal{O}\brac{\alpha\log^3 T }
+\mathcal{O}\brac{\beta\log^3 T }
+ \mathcal{O}\brac{\gamma\log^3 T }+\mathcal{O}\brac{\frac{\beta^2\log^2 T}{\alpha}}
+\mathcal{O}\brac{\frac{\beta^2\log^2 T}{\gamma}}
\\ &+\mathcal{O}\brac{\frac{\beta}{\alpha}(m\rho^k)}
+\mathcal{O}\brac{(m\rho^k)\log^2 T}
+\mathcal{O}\brac{\frac{(m\rho^k)^2}{\alpha\beta}},
}
 where $c_\alpha=\frac{2L_J\lambda_{\min}+2B C^3_\phi\brac{C_\phi L_\pi+2L_\phi } }{\lambda^2_{\min} \bar \lambda_{\min}}$ and $c_\eta=\frac{2B^2 (k+1)^2 C^6_\phi}{(\bar\lambda_{\min})^2}  $.
\end{proposition}
\begin{proof}
Recall that in \Cref{lm:actrackingerror}, we showed
\eqenv{
\E \ftwonormsq{ \theta\tit-\thestara}
&\leq \brac{1-\frac{\bar \lambda_{\min} \alpha}{2} }
  \E \ftwonormsq{\theta_t-\thestar}
 +\frac{L_J\lambda_{\min}+B C^3_\phi\brac{C_\phi L_\pi+2L_\phi }}{\lambda^2_{\min}}\beta \E\ftwonormsq{\gradj}
  \\&  \qquad  + \frac{(k+1)^2C^2_\phi }{\bar\lambda_{\min}}\alpha \E\flnormsq{\eta_t-J(\omega_t) }
   +G_t^\theta.
} 

Apply this inequality recursively and we have that
\eqenv{\label{vari:q}
\E &\ftwonormsq{ \theta_t-\thestar}
\\&\leq \brac{1-\frac{\bar \lambda_{\min} \alpha}{2} }^t
\E  \ftwonormsq{\theta_0-\theta^*_0}
  +\frac{L_J\lambda_{\min}+B C^3_\phi\brac{C_\phi L_\pi+2L_\phi } }{\lambda^2_{\min}}\beta \sum_{j=0}^{t-1} \brac{1-\frac{\bar \lambda_{\min} \alpha}{2} }^{t-j-1}\E\ftwonormsq{\nabla J(\omega_j)}
 \\& \quad  +\frac{(k+1)^2 C^2_\phi}{\bar \lambda_{\min}}\alpha\sum_{j=0}^{t-1} \brac{1-\frac{\bar \lambda_{\min} \alpha}{2} }^{t-j-1}\E\flnormsq{J(\omega_{j})-\eta_j} +\sum_{j=0}^{t-1} \brac{1-\frac{\bar \lambda_{\min} \alpha}{2} }^{t-j-1}G^\theta_j
 \\& = \brac{1-q }^t
  \E\ftwonormsq{\theta_0-\theta^*_0}
  +\frac{L_J\lambda_{\min}+B C^3_\phi\brac{C_\phi L_\pi+2L_\phi } }{\lambda^2_{\min}}\beta \sum_{j=0}^{t-1} \brac{1-q }^{t-j-1}\E\ftwonormsq{\nabla J(\omega_j)}
 \\& \quad  +\frac{(k+1)^2 C^2_\phi}{\bar \lambda_{\min}}\alpha\sum_{j=0}^{t-1} \brac{1-q}^{t-j-1}\E\flnormsq{J(\omega_{j})-\eta_j} +\sum_{j=0}^{t-1} \brac{1-q }^{t-j-1}G^\theta_j,
}
where we let $q= \frac{\bar \lambda_{\min} \alpha}{2}$.

Summing the inequality in \Cref{vari:q} above w.r.t. $t$ from $0$ to $T-1$ further implies that
\eqenv{
&\frac{1}{T}\sum_{t=0}^{T-1}
\E \ftwonormsq{ \theta_t-\thestar}
\\& \leq
\frac{1}{T}\sum_{t=0}^{T-1} \brac{1-q }^t
  \E\ftwonormsq{\theta_0-\theta^*_0}
  + \frac{L_J\lambda_{\min}+B C^3_\phi\brac{C_\phi L_\pi+2L_\phi } }{\lambda^2_{\min}}\frac{\beta }{T}\sum_{t=0}^{T-1}\sum_{j=0}^{t-1} \brac{1-q }^{t-j-1}\E\ftwonormsq{\nabla J(\omega_j)}
  \\& \qquad+\frac{(k+1)^2 C^2_\phi}{\bar \lambda_{\min}}\alpha\frac{1}{T}\sum_{t=0}^{T-1}\sum_{j=0}^{t-1} \brac{1-q}^{t-j-1}\E\flnormsq{J(\omega_{j})-\eta_j}
  +\frac{1}{T}\sum_{t=0}^{T-1}\sum_{j=0}^{t-1} \brac{1-q }^{t-j-1}G^\theta_j
  \\& \leq \frac{4B^2 }{Tq}+ \frac{L_J\lambda_{\min}+B C^3_\phi\brac{C_\phi L_\pi+2L_\phi } }{\lambda^2_{\min}}\frac{\beta }{Tq}\sum_{j=0}^{T-1}\E\ftwonormsq{\nabla J(\omega_j)}
  +\frac{(k+1)^2 C^2_\phi}{\bar \lambda_{\min}}\frac{\alpha}{Tq}\sum_{j=0}^{T-1} \E\flnormsq{J(\omega_{j})-\eta_j}
 \\& \qquad  + \frac{1}{Tq}\sum_{t=0}^{T-1} G^\theta_t,\label{eq:ac:thetaav}
}
where the last inequality is from the double-sum trick: $\sum^{T-1}_{t=0}\sum^{t-1}_{j=0} y^{t-j-1}X_j\leq (\sum^{T-1}_{t=0}X_t)(\sum^{T-1}_{t=0}y^t)\leq \frac{\sum^{T-1}_{t=0}X_t}{1-y}$ for $X_j\geq 0,  j=0,1,2,...,T-1$ and $y \in (0,1)$.


Recall that we showed in \Cref{lm:aceta} that
\eqenv{
\E\flnormsq{\eta\tit-J(\omega_{t+1})}
 \leq \brac{1-\gamma}\E\flnormsq{\eta_t-J(\omega_t)}
 +C_\phi^4 B^2{\beta} \E\ftwonormsq{\gradj}+ G_t^\eta.
}

Recursively applying this inequality implies that
\eqenv{\label{cons:gate}
\E&\flnormsq{\eta_t-J(\omega_t)}
 \\&\leq \brac{1-\gamma}^t \lnormsq{\eta_0-J(\omega_0)}
 + C_\phi^4 B^2\beta\sum_{j=0}^{t-1} \brac{1-\gamma}^{t-j-1}\E\ftwonormsq{\nabla J(\omega_j)}
 + \sum_{j=0}^{t-1} \brac{1-\gamma}^{t-j-1} G_j^\eta
 \\&\leq  R^2_{\max} \brac{1-\gamma}^t+ C_\phi^4 B^2 \beta\sum_{j=0}^{t-1} \brac{1-\gamma}^{t-j-1}\E\ftwonormsq{\nabla J(\omega_j)}
 + \sum_{j=0}^{t-1} \brac{1-\gamma}^{t-j-1} G_j^\eta. 
}
We then sum the above inequality w.r.t. $t$ from 0 to $T-1$, and have that
\eqenv{
\frac{1}{T}&\sum_{t=0}^{T-1} \E\flnormsq{\eta_t-J(\omega_t)}
\\& \leq \frac{R^2_{\max}}{T}\sum_{t=0}^{T-1}   \brac{1-\gamma}^t 
 + C_\phi^4 B^2 \frac{\beta}{T} \sum_{t=0}^{T-1}\sum_{j=0}^{t-1} \brac{1-\gamma}^{t-j-1}\E\ftwonormsq{\nabla J(\omega_j)}
+ \frac{1}{T}\sum_{t=0}^{T-1}\sum_{j=0}^{t-1} \brac{1-\gamma}^{t-j-1} G_j^\eta
\\ & =\frac{R^2_{\max}}{T\gamma}+C_\phi^4 B^2 \frac{\beta}{T\gamma}\sum_{t=0}^{T-1}\E\ftwonormsq{\nabla J(\omega_t)}
+ \frac{1}{T\gamma}\sum_{t=0}^{T-1} G_t^\eta,\label{eq:ac:etaav}
}
where we use the double-sum trick below \Cref{eq:ac:thetaav} again. 

Plugging \Cref{eq:ac:etaav} in \Cref{eq:ac:thetaav} further implies that
\eqenv{\label{cons:c}
    \frac{1}{T}\sum_{t=0}^{T-1}
\E \ftwonormsq{ \theta_t-\thestar}
& \myineq{\leq}{a} \frac{8B^2 }{\bar\lambda_{\min}T \alpha}+\frac{2L_J\lambda_{\min}+2B C^3_\phi\brac{C_\phi L_\pi+2L_\phi } }{\lambda^2_{\min} \bar \lambda_{\min}} \frac{\beta}{T\alpha}\sum_{t=0}^{T-1}\E\ftwonormsq{\nabla J(\omega_t)}
   \\& \qquad +\frac{2 (k+1)^2 C^2_\phi}{(\bar\lambda_{\min})^2}\frac{1}{T}\sum_{t=0}^{T-1} \E\flnormsq{J(\omega_t)-\eta_t}
  +\frac{2}{\bar\lambda_{\min}T \alpha}\sum_{t=0}^{T-1} G^\theta_t
  \\&\leq \brac{\frac{2L_J\lambda_{\min}+2B C^3_\phi\brac{C_\phi L_\pi+2L_\phi } }{\lambda^2_{\min} \bar \lambda_{\min}} \frac{ \beta}{\alpha}+\frac{ 2(k+1)^2B^2 C^6_\phi\beta}{(\bar\lambda_{\min})^2\gamma} } \frac{1}{T}\sum_{t=0}^{T-1}\E\ftwonormsq{\nabla J(\omega_t)}\\&\qquad+ \frac{8B^2 }{\bar\lambda_{\min}T \alpha}+ \frac{2}{\bar\lambda_{\min}T \alpha}\sum_{t=0}^{T-1} G^\theta_t
  +\frac{2 (k+1)^2 C^2_\phi R^2_{\max}}{(\bar\lambda_{\min})^2 T\gamma}
  +\frac{2 (k+1)^2 C^2_\phi   }{(\bar\lambda_{\min})^2}\frac{1}{T\gamma}\sum_{t=0}^{T-1} G_t^\eta
  \\& =\brac{\frac{c_\alpha \beta}{\alpha}+\frac{c_\eta \beta}{\gamma} } \frac{1}{T}\sum_{t=0}^{T-1}\E\ftwonormsq{\nabla J(\omega_t)}+ \frac{8B^2 }{\bar\lambda_{\min}T \alpha}+ \frac{2}{\bar\lambda_{\min}T \alpha}\sum_{t=0}^{T-1} G^\theta_t
  \\&\qquad +\frac{2 (k+1)^2 C^2_\phi R^2_{\max}}{(\bar\lambda_{\min})^2 T\gamma}
  +\frac{2 (k+1)^2 C^2_\phi  }{(\bar\lambda_{\min})^2 } \frac{1}{T\gamma}\sum_{t=0}^{T-1} G_t^\eta,
}
where $(a)$ follows from that $q=\frac{\bar\lambda_{\min}\alpha}{2}$, $c_\alpha=\frac{2L_J\lambda_{\min}+2B C^3_\phi\brac{C_\phi L_\pi+2L_\phi } }{\lambda^2_{\min} \bar \lambda_{\min}}$ and $c_\eta=\frac{ 2(k+1)^2B^2 C^6_\phi}{(\bar\lambda_{\min})^2}  $.  
By \Cref{cons:c}, if we set the stepsize $\beta=\min\varbrac{ \frac{\lambda^2_{\min} \bar \lambda_{\min}\alpha}{16 C^4_\phi\brac{L_J\lambda_{\min}+B C^3_\phi\brac{C_\phi L_\pi+2L_\phi } }}, \frac{(\bar\lambda_{\min})^2\gamma}{16(k+1)^2 C^{10}_\phi B^2} }$, then we can get that $\brac{\frac{c_\alpha \beta}{\alpha}+\frac{c_\eta \beta}{\gamma} } \leq \frac{1}{4 C^4_\phi} $. This completes the proof of \Cref{thm1}.
\end{proof}

We are now ready to prove \Cref{thm:ac}.

\begin{theorem}(Restatement of \Cref{thm:ac})
    Consider the AC algorithm in  \Cref{alg:AC} with constant step sizes $\alpha_t=\alpha,\beta_t=\beta,\gamma_t=\gamma$. Then, it holds that
    \begin{align}
    \frac{1}{T} \sum _{t=0}^{T-1} \E \Fbrac{\twonormsq{\gradj}} 
    & \leq 2C_\phi^4 \frac{1}{T} \sum _{t=0}^{T-1} \E \Fbrac{\twonormsq{\thestar-\theta_t}} +\mathcal{O}\brac{\frac{1}{T \beta}}
    +\mathcal{O}\brac{\beta \log^2 T}+\mathcal{O} (m\rho^k) .
    \end{align}
    If we further set $k=\myceil{\frac{\log T}{1-\rho}}, \gamma=\mathcal{O}(\frac{1}{\sqrt{T}}),  \alpha=\mathcal{O}(\frac{1}{\sqrt{T}\log^2 T}),  \beta =\mathcal{O}(\frac{1}{\sqrt{T}\log^2 T})$, then we have that
    \begin{align}
        \frac{1}{T} \sum _{t=0}^{T-1} \E \Fbrac{\twonormsq{\gradj}}
    \leq\mathcal{O}\brac{\frac{{\log^3 T}}{\sqrt{T}}}.
    \end{align}
\end{theorem}

\begin{proof}
Recall \Cref{eq:graditera}, and we have that
  \eqenv{\label{eq:ac:gradav}
\frac{1}{T}\sum_{t=0}^{T-1} \E \ftwonormsq{\gradj}
&\leq \frac{2\brac{\E\Fbrac{J(\omega_{T})}-J(\omega_0)}}{T \beta}
+  \frac{2 C^4_\phi}{T}\sum_{t=0}^{T-1} \E\Fbrac{\twonormsq{\theta_t-\thestar}}
+   \frac{2}{T\beta}\sum_{t=0}^{T-1}G_t^\omega.
} 

 Plug \Cref{eq:ac:gradav} in \Cref{cons:c}, and we have that
\eqenv{
&\frac{1}{T}\sum_{t=0}^{T-1}
\E \ftwonormsq{ \theta_t-\thestar}
  \\
  &\leq \frac{1}{4C^4_\phi} \frac{1}{T}\sum_{t=0}^{T-1}\E\ftwonormsq{\nabla J(\omega_t)}+ \frac{8B^2 }{\bar \lambda_{\min} T\alpha}+ \frac{2}{\bar \lambda_{\min} T\alpha}\sum_{t=0}^{T-1} G^\theta_t
  +\frac{2 (k+1)^2 C^2_\phi R^2_{\max}}{(\bar\lambda_{\min})^2 T\gamma}
  +\frac{2 (k+1)^2 C^2_\phi  }{(\bar\lambda_{\min})^2 } \frac{1}{T\gamma}\sum_{t=0}^{T-1} G_t^\eta\\
\\& = \frac{1}{2 C^4_\phi  T\beta}\brac{\E\Fbrac{J(\omega_{T})}-J(\omega_0)}
+  \frac{1}{2T}\sum_{t=0}^{T-1} \E\Fbrac{\twonormsq{\theta_t-\thestar}}
+ \frac{8B^2 }{\bar \lambda_{\min} T\alpha}+ \frac{1}{2 C^4_\phi T\beta}\sum_{t=0}^{T-1}G^\omega_t
\\&\qquad   +\frac{2 (k+1)^2 C^2_\phi}{(\bar\lambda_{\min})^2}\brac{\frac{R^2_{\max}}{T\gamma}+ \frac{1}{T\gamma}\sum_{t=0}^{T-1} G_t^\eta }+\frac{2}{\bar \lambda_{\min} T\alpha}\sum_{t=0}^{T-1} G^\theta_t
 .
}

This further implies that
\eqenv{\label{eq:ac:trackingerrorfinal}
 \frac{1}{T}\sum_{t=0}^{T-1}
\E &\ftwonormsq{ \theta_t-\thestar}
 \leq 
  \frac{1}{ C^4_\phi T \beta}\brac{\E\Fbrac{J(\omega_{T})}-J(\omega_0)}
\\& + \frac{16B^2 }{\bar \lambda_{\min} T\alpha}+ \frac{1}{ C^4_\phi T\beta} \sum_{t=0}^{T-1}G^\omega_t
+\frac{4 (k+1)^2 C^2_\phi}{(\bar\lambda_{\min})^2}\brac{\frac{ R^2_{\max}}{T\gamma}+ \frac{1}{T\gamma}\sum_{t=0}^{T-1} G_t^\eta }+\frac{4}{\bar \lambda_{\min} T\alpha}\sum_{t=0}^{T-1} G^\theta_t. 
}

Next, we choose the stepsizes to minimize the tracking error and the gradient norm. 
We choose $\gamma=\frac{1}{\sqrt{T}}$ and  $k= \myceil{\frac{\log T}{1-\rho}}$. Consider $T\geq \frac{2C^2_\phi d m}{\lambda_{\min}}$, then we can get that $\bar \lambda_{\min}\geq \frac{\lambda_{\min}}{2}$. 
  We set $\alpha$ and $\beta$ such that $\alpha=  \frac{1}{(k+1)^2  }\gamma$ and $\beta= \min \varbrac{\frac{\lambda^3_{\min} }{4\brac{L_J\lambda_{\min}+B C^3_\phi\brac{C_\phi L_\pi+2L_\phi } +2\lambda^2_{\min}}}\alpha,\frac{\lambda^3_{\min} }{32C^4_\phi\brac{L_J\lambda_{\min}+B C^3_\phi\brac{C_\phi L_\pi+2L_\phi } }}\alpha,\frac{\lambda^2_{\min}}{64 (k+1)^2 C^6_\phi B^2 }\gamma}  $. It holds that
  \eqenv{\label{eq:acostepsize}
  \gamma=\mathcal{O}\brac{\frac{1}{\sqrt{T}}},  \alpha=\mathcal{O}\brac{\frac{1}{\sqrt{T}\log^2 T}},  \beta =\mathcal{O}\brac{\frac{1}{\sqrt{T}\log^2 T}},q =\mathcal{O}\brac{\frac{1}{\sqrt{T}\log^2 T}}.
  }

According to \Cref{sym:Gomegat}, \Cref{eq:33} and \Cref{vari:gthetat}, the orders of the following terms are as follows: 
\eqenv{\label{eq:acog}
&\frac{1}{T}\sum_{t=0}^{T-1}G_t^\omega=\mathcal{O}\brac{ (m\rho^k)\beta+ k^2 \beta^2 }=\mathcal{O}\brac{\frac{1}{T \log^2 T}};
\\&\frac{1}{T}\sum_{t=0}^{T-1}G_t^\eta= \mathcal{O}\brac{(m\rho^k)\gamma +\beta^2+k^2\beta\gamma +k \gamma^2}=\mathcal{O}\brac{\frac{\log T}{T}};
\\&\frac{1}{T}\sum_{t=0}^{T-1} G_t^\theta =\mathcal{O}\brac{k (m\rho^k)\alpha+(m\rho^k)\beta+\frac{(m \rho^k)^2}{\beta}+k^3\alpha^2+k^3 \alpha\beta+k^2\beta^2 }=\mathcal{O}\brac{\frac{1}{T \log T}}.
}
Then \Cref{eq:ac:trackingerrorfinal} can be bounded as follows: 
  \eqenv{\label{eq:ac:ordertk}
   \frac{1}{T}\sum_{t=0}^{T-1}
\E \ftwonormsq{ \theta_t-\thestar}
 = \mathcal{O}\brac{ \frac{\log^3 T}{\sqrt{T}} }.
  }

Plugging \Cref{eq:ac:trackingerrorfinal} in \Cref{eq:ac:gradav} implies that
\eqenv{\label{eq:acfinialgrad}
\frac{1}{T}\sum_{t=0}^{T-1} \E \ftwonormsq{\gradj}
 & \leq 
\frac{32 C^4_\phi B^2 }{\bar \lambda_{\min} T\alpha}+ \frac{8 (k+1)^2 C^6_\phi}{(\bar\lambda_{\min})^2}\brac{\frac{R^2_{\max}}{T\gamma}+ \frac{1}{T\gamma}\sum_{t=0}^{T-1}G_t^\eta }+ \frac{8C^4_\phi}{\bar \lambda_{\min} T\alpha}\sum_{t=0}^{T-1}G_t^\theta
\\& \qquad+ \frac{4}{T\beta }\brac{\E\Fbrac{J(\omega_{T})}-J(\omega_0)}+
\frac{4}{T\beta}\sum_{t=0}^{T-1}G_t^\omega.
 }

Plugging in the step-sizes above, we have that
   \eqenv{
   \frac{1}{T}\sum_{t=0}^{T-1}
\E \ftwonormsq{ \gradj}
 = \mathcal{O}\brac{ \frac{\log^3 T}{\sqrt{T}} },
  }
  which completes the proof.
\end{proof}

\section{NAC Sample Complexity Analysis}\label{appdix:proofnac}
In this section, we provide the sample complexity analysis for NAC. 

\subsection{Bound on Gradient Norm in NAC}\label{sec:nc1}
Recall that in \Cref{alg:AC}, NAC updates the policy parameter as follows:
$\omega\tit-\omega_t=\beta_t\theta_t $, which directly implies that
\eqenv{\label{constant:bl}
\twonorm{\omega\tit-\omega_t }
\leq \beta_t \twonorm{\theta_t }
\leq  B \beta_t.
}

Consider the largest eigenvalue of the matrix 
$F_t=\E_{D_t}\Fbrac{\phi_t(s,a)\phi^\top_t(s,a)} $. Note that for any vector $x\in \mathbb{R}^d$
\eqenv{
 \twonorm{F_t x}= \twonorm{\E_{D_t}\Fbrac{\phi_t(s,a)\phi^\top_t(s,a)x}  }
 \leq \twonorm{\phi_t(s,a) }\twonorm{\phi_t(s,a) }\twonorm{x }
 \leq C^2_\phi \twonorm{x}.
} Thus, $\lambda_{\max}\brac{F_t }\leq C^2_\phi $. 
Then, by \Cref{Prop:lip}, we can show that
\eqenv{\label{eq:84}
J&(\omega\tit)\geq J(\omega_t)+ \vbrac{\gradj, \omega\tit -\omega_t}-\frac{L_J}{2} \twonormsq{\omega\tit-\omega_t}
\\ & \geq J(\omega_t)+ \beta_t\vbrac{\gradj, \theta_t}- \frac{\beta_t^2 L_J}{2} \twonormsq{\theta_t}
\\& \geq J(\omega_t)+\beta_t\vbrac{\gradj, \thebar}+ \beta_t\vbrac{\gradj, \theta_t-\thestar}
+ \beta_t\vbrac{\gradj, \thestar-\thebar}
- \frac{L_J B^2\beta_t^2}{2}
\\ & \myineq{\geq}{a}
J(\omega_t)+\beta_t\vbrac{\gradj, \brac{\E_{D_t}\Fbrac{\phi_t(s,a)\phi^\top_t(s,a)}}^{-1}\gradj}-\frac{\beta_t}{4C^2_\phi}\twonormsq{\gradj}-C^2_\phi\beta_t\twonormsq{\theta_t-\thestar} 
\\& \qquad -\frac{\beta_t}{4C^2_\phi}\twonormsq{\gradj}-C^2_\phi\beta_t\twonormsq{\thestar-\thebar}
- \frac{L_J B^2\beta_t^2}{2}
\\& \myineq{\geq}{b} J(\omega_t)+\frac{\beta_t}{C^2_\phi}\twonormsq{\gradj} -\frac{\beta_t}{2C^2_\phi}\twonormsq{\gradj}-C^2_\phi\beta_t\twonormsq{\theta_t-\thestar}- \frac{C_{\text{gap}}C^2_\phi m \rho^k\beta_t}{\lambda_{\min}}
- \frac{L_J B^2\beta_t^2}{2}
\\&= J(\omega_t)+\frac{\beta_t}{2C^2_\phi}\twonormsq{\gradj}
-C^2_\phi\beta_t\twonormsq{\theta_t-\thestar}- \frac{C_{\text{gap}}C^2_\phi m \rho^k\beta_t}{\lambda_{\min}}- \frac{L_J B^2\beta_t^2}{2} ,
}
where $(a)$ follows from that $\vbrac{\gradj, \theta_t-\thestar} \geq-\frac{1}{4C^2_\phi}\twonormsq{\gradj}-C^2_\phi\twonormsq{\theta_t-\thestar}  $ and $(b)$ follows from that $\lambda_{\max}(F_t)\leq C^2_\phi$, $\vbrac{\gradj, \brac{\E_{D_t}\Fbrac{\phi_t(s,a)\phi^\top_t(s,a)}}^{-1}\gradj}= \gradj^\top \brac{F_t}^{-1}\gradj\geq \frac{1}{C^2_\phi}\twonormsq{\gradj},$  and \Cref{prop:bound}. 

Taking the expectation on both sides of \Cref{eq:84}, 
 we have that
 \eqenv{
 \E\ftwonormsq{\gradj}\leq 2C^2_\phi \frac{\E \Fbrac{J(\omega\tit)}- \E \Fbrac{J(\omega_t)}}{\beta_t}
 + 2 C_\phi^4 \E\ftwonormsq{\theta_t-\thestar}+ \frac{2C^4_\phi C_{\text{gap}} m \rho^k}{\lambda_{\min}}
 + {L_J B^2 C^2_\phi \beta_t}.\label{eq:94}
 }

\subsection{Bound on $|\eta_t-J(\omega_t)|$ in NAC}\label{sec:nc2}
In this section, we bound the term $\eta_t-J(\omega_t)$ for the NAC algorithm.
\begin{lemma}\label{lm:naceta}
If we denote 
\eqenv{
\widetilde G^\eta_t& = 2\gamma_t \brac{  BR_{\max}^2 C_\pi  \sum_{j=t-k}^{t-1} \sum_{i=j}^{t-1} \beta_i 
+  R^2_{\max} m \rho^k+ R_{\max}^2 \sum_{j=t-k}^{t-1}\gamma_j
+  BC_JR_{\max}   \sum_{j=t-k}^{t-1}\beta_j}
\\& + R_{\max}^2 \gamma^2_t + C_J^2  B^2 \beta_t^2
+   2 BC_J R_{\max}  \beta_t\gamma_t
+ R_{\max}L_J  B^2 \beta_t^2,
}
and set $\gamma_t-\gamma^2_t\geq \beta_t$, 
then it holds that
\eqenv{\label{eq:nacetaitera}
\E\flnormsq{\eta\tit-J(\omega_{t+1})}
 \leq \brac{1-\gamma_t}\E\flnormsq{\eta_t-J(\omega_t)}
+{\beta_t}B^2 \E\ftwonormsq{\gradj}+ \widetilde G^\eta_t.
}

\end{lemma}

\begin{proof}
    
Similar to the AC analysis in \Cref{eq:ac:averr} in \Cref{sec:ac:av}, we have that
\eqenv{
\lnormsq{\eta\tit-J(\omega_{t+1}) }
&=\lnormsq{\brac{1-\gamma_t}\brac{\eta_t-J(\omega_t)}
+\gamma_t(R_t-J(\omega_t))+J(\omega_t)-J(\omega_{t+1})  }
\\ &= \brac{1-\gamma_t}^2\lnormsq{\eta_t-J(\omega_t)}
+\gamma_t^2 \lnormsq{R_t-J(\omega_t)} +\lnormsq{J(\omega_t)-J(\omega_{t+1})}
\\&\qquad +2\gamma_t\underbrace{\brac{R_t-J(\omega_t)}\brac{J(\omega_t)-J(\omega_{t+1})}}_{\text{I}}
+ 2\gamma_t\brac{1-\gamma_t}
\underbrace{\brac{\eta_t-J(\omega_t)}\brac{R_t-J(\omega_t)}}_{\text{II}}
\\&\qquad +2\brac{1-\gamma_t}
\underbrace{\brac{\eta_t-J(\omega_t)}\brac{J(\omega_t)-J(\omega_{t+1})}}_{\text{III}}.\label{eq:nac:av1}
}
The term $\lbrac{J(\omega_t)-J(\omega_{t+1})}$ can be bounded using its Lipschitz smoothness in \Cref{Prop:lip}:
\eqenv{
\lbrac{J(\omega_t)-J(\omega\tit)}
 \leq C_J\twonorm{\omega_t-\omega\tit}
 \leq C_J  B \beta_t .
}

Term I in \Cref{eq:nac:av1}  can be bounded as follows: 
\eqenv{
\lbrac{\E\Fbrac{(R_t-J(\omega_t))(J(\omega_t)-J(\omega_{t+1}))}}
 &\leq \E \Fbrac{\lbrac{R_t-J(\omega_t)}\lbrac{J(\omega_t)-J(\omega_{t+1})} }
\leq  BC_J R_{\max}  \beta_t.\label{eq:nac:avii}
}

Term II in \Cref{eq:nac:av1} can be bounded as follows: 
\eqenv{
&\lbrac{\E\Fbrac{(\eta_t-J(\omega_t))(R_t-J(\omega_t))}}
\\& \leq
\lbrac{\E\Fbrac{(\eta_{t-k}-J(\omega_{t-k}))( R_t-J(\omega_t))}}
+\lbrac{\E\Fbrac{(\eta_t-\eta_{t-k}-J(\omega_t)+J(\omega_{t-k}))( R_t-J(\omega_t))}}
\\& \leq \lbrac{\E\Fbrac{\E\Fbrac{{(\eta_{t-k}-J(\omega_{t-k}))(R_t-J(\omega_t))}|\mathcal{F}_{t-k}}}}
+ \E\Fbrac{\lnorm{ \eta_t-\eta_{t-k}-J(\omega_t)+J(\omega_{t-k})}\lnorm{R_t-J(\omega_t)}}
\\&\myineq{\leq }{a}\lbrac{\E\Fbrac{\E\Fbrac{{(\eta_{t-k}-J(\omega_{t-k}))R_t}|\mathcal{F}_{t-k}}-\E_{(s,a)\sim D_t}\Fbrac{{(\eta_{t-k}-J(\omega_{t-k}))R(s,a)}|\mathcal{F}_{t-k}}}}
\\& \qquad+ \E\Fbrac{\lnorm{ \eta_t-\eta_{t-k}-J(\omega_t)+J(\omega_{t-k})}\lnorm{R_t-J(\omega_t)}}
\\& \myineq{\leq}{b}    R^2_{\max}\E\Fbrac{ \tvnorm{\mathbb{P}\brac{(s_t,a_t)|\mathcal{F}_{t-k}},D_t}  }
+ R_{\max}\E\Fbrac{\lnorm{\eta_t-\eta_{t-k} }+\lnorm{J(\omega_t)-J(\omega_{t-k}) }} 
\\ & \myineq{\leq}{c} 
R_{\max}^2\brac{ C_\pi\sum_{j=t-k}^{t-1} \E\ftwonorm{\omega_t-\omega_{j}}+m\rho^k } 
+ R_{\max}\brac{R_{\max}\sum_{j=t-k}^{t-1}\gamma_j+C_J \E\ftwonorm{\omega_t-\omega_{t-k}} }
\\& \leq B R_{\max}^2 C_\pi  \sum_{j=t-k}^{t-1} \sum_{i=j}^{t-1} \beta_i 
+  R^2_{\max} m \rho^k+ R_{\max}^2\sum_{j=t-k}^{t-1}\gamma_j
+  R_{\max} C_J  B \sum_{j=t-k}^{t-1}\beta_j,\label{eq:nac:avi}
}
 where $(a)$ follows from $\E_{(s,a)\sim D_t}\Fbrac{R(s,a)-J(\omega_t)|\mathcal{F}_t}=0 $, $(b)$ follows from  $0\leq\eta_t\leq R_{\max}$, $0\leq J(\omega_t)\leq R_{\max}$,  $0\leq R_t\leq R_{\max}$, and 
 $(c)$ follows from  \Cref{lm:tv}. 


We then bound  term III as follows: 
\eqenv{\label{eq:nac:aviii}
&\lbrac{\E\Fbrac{\brac{\eta_t-J(\omega_t)}\brac{J(\omega_t)-J(\omega_{t+1})}}}
\\& \qquad =\lbrac{\E\Fbrac{(\eta_t-J(\omega_t))\brac{ -\nabla^\top J(\omega_t) (\omega\tit-\omega_t)+ \brac{\omega\tit-\omega_t}^\top\frac{\nabla^2J(\hat \omega_t)}{2}\brac{\omega\tit-\omega_t} }}}
\\& \qquad \leq\lbrac{\E\Fbrac{(\eta_t-J(\omega_t)) \nabla^\top J(\omega_t) (\omega\tit-\omega_t) }}
\\& \qquad\qquad+
\lbrac{\E\Fbrac{(\eta_t-J(\omega_t))\brac{\omega\tit-\omega_t}^\top \frac{\nabla^2J(\hat \omega_t)}{2}\brac{\omega\tit-\omega_t} }}
\\& \qquad =\beta_t \lbrac{\E\Fbrac{(\eta_t-J(\omega_t)) \nabla^\top J(\omega_t) \theta_t  }}
+\beta^2_t\lbrac{\E\Fbrac{\lbrac{\eta_t-J(\omega_t)}  \frac{\twonorm{\nabla^2J(\hat \omega_t)}}{2}\twonormsq{\theta_t} }}
\\ & \qquad \leq \frac{\beta_t}{2}\E\flnormsq{\eta_t-J(\omega_t)}+
\frac{  B^2\beta_t}{2}\E\ftwonormsq{\gradj}
+\frac{ R_{\max}L_J  B^2}{2} \beta_t^2,
}
where the first equation is from the Lagrange's Mean Value theorem for some $\hat{\omega}_t=\lambda\omega_t+(1-\lambda)\omega_{t+1}$, $\lambda\in(0,1)$.

Plug  \Cref{eq:nac:avii}, \Cref{eq:nac:avi}, and \Cref{eq:nac:aviii} in \Cref{eq:nac:av1}, and we have that
\eqenv{\label{eq:nac_eta}
\E&\flnormsq{\eta\tit-J(\omega_{t+1})}
 \leq \brac{(1-\gamma_t)^2+ {\beta_t}}\E\flnormsq{\eta_t-J(\omega_t)}
+ B^2{\beta_t} \E\ftwonormsq{\gradj}
\\ &
+ 2\gamma_t \brac{  B R_{\max}^2 C_\pi  \sum_{j=t-k}^{t-1} \sum_{i=j}^{t-1} \beta_i 
+  R^2_{\max} m \rho^k+ R_{\max}^2 \sum_{j=t-k}^{t-1}\gamma_j
+  R_{\max}C_J  B \sum_{j=t-k}^{t-1}\beta_j}
\\& + R_{\max}^2 \gamma^2_t + C_J^2  B^2 \beta_t^2
+ 2BC_J R_{\max}  \beta_t\gamma_t 
+R_{\max}L_J B^2 \beta_t^2.
} In \Cref{eq:nac_eta}, we use the fact that $1-\gamma_t\leq 1$. 
This completes the proof. 
\end{proof}


\subsection{Tracking Error Analysis of NAC}\label{sec:nc3}

In this section, we bound the tracking error $\theta_t-\thestar$ for NAC.  Define
\eqenv{
    \widetilde G^\theta_t =&
   \brac{ 8  +\frac{ 2 }{\beta_t}+ \beta_t C_\Theta   }\brac{\frac{C_{\text{gap}}m\rho^k}{\lambda_{\min}}}^2
+ (k+1)^2U_\delta^2 C_\phi^2 \alpha_t^2
+\frac{4(k+1) C_\phi U_\delta C_{\text{gap}}m\rho^k}{\lambda_{\min}}\alpha_t
 \\& +2(k+1) C_\phi B U_\delta C_\Theta \alpha_t\beta_t
+ 2\alpha_t G_t^\delta+2 C_\Theta^2  B^2 \beta_t^2
.} 
\begin{lemma}\label{lm:nactheta}
If we set the step size satisfies that $\beta_t\leq \frac{\bar\lambda_{\min}}{4(2C_\Theta+1)}\alpha_t$, then it holds that
\eqenv{\label{eq:nac:thetaitera}
\E \ftwonormsq{ \theta\tit-\theta^*_{t+1}}
&\leq \brac{1-\frac{\bar \lambda_{\min}\alpha_t}{2}  }
  \E \ftwonormsq{\theta_t-\thestar}
 \\& +  \frac{C_\Theta}{\lambda^2_{\min}} \beta_t\E\ftwonormsq{\gradj}
   +\frac{(k+1)^2 C^2_\phi}{\bar \lambda_{\min}}\alpha_t\E\flnormsq{J(\omega_t)-\eta_t} +\widetilde G_t^\theta.
}
\end{lemma}

\begin{proof}
From the update rule of \Cref{alg:AC}, we have that
\eqenv{\label{eq:eq96}
\twonormsq{\theta\tit-\thestara}&=
\twonormsq{\Pi_B\brac{\theta_t+\alpha_t \delta_t z_t}-\thestara}
\\&\myineq{\leq}{a}\twonormsq{ \theta_t+\alpha_t \delta_t z_t-\thestara}
\\ & \leq \twonormsq{\theta_t+\alpha_t \delta_t z_t-\thestar+\thestar-\thestara}
\\ & = \twonormsq{\theta_t-\thestar}+ \alpha_t^2 \twonormsq{\delta_t z_t}+\twonormsq{ \thestar-\thestara}
+ 2 \alpha_t\vbrac{\theta_t-\thestar, \delta_t z_t}
\\&\quad + 2\alpha_t \vbrac{\delta_t z_t, \thestar-\thestara }
+ 2 \vbrac{\theta_t-\thestar, \thestar-\thestara},
}
where $(a)$ follows from the fact $\twonorm{\Pi_B(x)-y}\leq \twonorm{x-y} $ when $\twonorm{y}\leq B $ and $\twonorm{\theta^*_{t+1}}\leq B$.

Taking expectations on both sides of \Cref{eq:eq96}, we have that
\eqenv{ \label{eq:nac:theta1}
\E&\ftwonormsq{\theta\tit-\thestara}  \leq \E\ftwonormsq{\theta_t-\thestar}+ \alpha_t^2 \E\ftwonormsq{\delta_t z_t}+\E\ftwonormsq{ \thestar-\thestara}
\\&\quad+ 2 \alpha_t\underbrace{\E\fvbrac{\theta_t-\thestar, \delta_t z_t}}_{\text{I}}
 + 2\alpha_t \underbrace{\E \fvbrac{\delta_t z_t, \thestar-\thestara }}_{\text{II}}
+ 2 \underbrace{\E\fvbrac{\theta_t-\thestar, \thestar-\thestara}}_{\text{III}}.
}

For the term $\twonorm{\thestar-\thestara}$, we have that
\eqenv{
\twonorm{\thestar-\thestara}
& =\twonorm{\thebar-\thebara+ \thestar-\thebar-\thestara+\thebara}
\\& \leq \twonorm{\thebar-\thebara}+\twonorm{ \thestar-\thebar}+ \twonorm{\thestara-\thebara}
\\& \myineq{\leq}{a} 
 C_\Theta\twonorm{\omega_t-\omega\tit}+\frac{2C_{\text{gap}}m\rho^k}{\lambda_{\min}}
\\ & = \beta_t C_\Theta \twonorm{\theta_t }+\frac{2C_{\text{gap}}m\rho^k}{\lambda_{\min}}
\\ & \leq C_\Theta  B \beta_t+\frac{2C_{\text{gap}}m\rho^k}{\lambda_{\min}},
}
where $(a)$ follows from \Cref{lm:acnacinnersmooth} and  \Cref{prop:bound}. Hence we have that
\eqenv{
\E\ftwonormsq{\thestar-\thestara }
\leq  2 C^2_\Theta  B^2 \beta^2_t+\frac{8C^2_{\text{gap}}m^2\rho^{2k}}{\lambda^2_{\min}}.
}

By Lemma \ref{lm:consterm},  term I in \Cref{eq:nac:theta1} can be bounded as
\eqenv{
\E\fvbrac{\theta_t-\thestar, \delta_t z_t}
\leq& - \frac{\bar \lambda_{\min}}{2} \E \ftwonormsq{\theta_t-\thestar}
+\frac{(k+1)^2 C^2_\phi}{2\bar \lambda_{\min}}\E\flnormsq{J(\omega_t)-\eta_t}+G^\delta_t.
}

For  term II in \Cref{eq:nac:theta1}, we have that
\eqenv{
\E \fvbrac{ \delta_t z_t, \thestar-\thestara}
& \leq \E\Fbrac{\twonorm{\delta_t z_t} \twonorm{ \thestar-\thestara} }
\\& \leq (k+1) C_\phi U_\delta\E\Fbrac{\twonorm{\thebara-\thebar }+\twonorm{ \thestar-\thebar}+\twonorm{\thestara-\thebara } }
\\& \myineq{\leq}{a} (k+1) C_\phi U_\delta C_\Theta \E\ftwonorm{\omega\tit-\omega_t }+ \frac{2(k+1) C_\phi U_\delta C_{\text{gap}}m\rho^k}{\lambda_{\min}}
\\ &\leq (k+1) C_\phi B U_\delta C_\Theta \beta_t
+ \frac{2(k+1) C_\phi U_\delta C_{\text{gap}}m\rho^k}{\lambda_{\min}}, 
} where $(a)$ follows from \Cref{lm:acnacinnersmooth} and  \Cref{prop:bound}.

For term III in \Cref{eq:nac:theta1}, we have  that
\eqenv{
&\E\fvbrac{\theta_t-\thestar, \thestar-\thestara }
 \\& = \E\fvbrac{\theta_t-\thestar, \thebar-\thebara } +\E\fvbrac{\theta_t-\thestar, \thestar-\thebar }
 +\E\fvbrac{\theta_t-\thestar, \thebara-\thestara }
 \\&\leq \E\Fbrac{\twonorm{\theta_t-\thestar}\twonorm{\thebar-\thebara}  } +\E\Fbrac{\twonorm{\theta_t-\thestar }\twonorm{\thebar-\thestar}  }
 +\E\Fbrac{\twonorm{\theta_t-\thestar}\twonorm{\thebara-\thestara }  }
 \\&\myineq{\leq}{a} C_\Theta\E\Fbrac{\twonorm{\theta_t-\thestar} \twonorm{\omega\tit-\omega_t}  }+ \frac{\beta_t}{2}\E\ftwonormsq{\theta_t-\thestar}+ \frac{1}{2\beta_t}\E\ftwonormsq{\thestar-\thebar}
 \\& \qquad + \frac{\beta_t}{2}\E\ftwonormsq{\theta_t-\thestar}+ \frac{1}{2\beta_t}\E\ftwonormsq{\thestara-\thebara}
 \\& \myineq{\leq}{b}C_\Theta\E\Fbrac{\twonorm{\theta_t-\thestar}  \twonorm{\omega\tit-\omega_t} }+ \beta_t \E\ftwonormsq{\theta_t-\thestar} +\frac{1}{\beta_t}\brac{\frac{C_{\text{gap}}m\rho^k}{\lambda_{\min}} }^2
\\& =C_\Theta \beta_t \E\Fbrac{\twonorm{\theta_t-\thestar }\twonorm{ \theta_t}  }
+ \beta_t \E\ftwonormsq{\theta_t-\thestar} +\frac{1}{\beta_t}\brac{\frac{C_{\text{gap}}m\rho^k}{\lambda_{\min}} }^2
\\& \leq  C_\Theta \beta_t  \E\Fbrac{\twonorm{\theta_t-\thestar}\twonorm{\thebar} }
+ C_\Theta \beta_t  \E\ftwonormsq{\theta_t-\thestar}+ C_\Theta \beta_t  \E\Fbrac{\twonorm{\theta_t-\thestar} \twonorm{ \thestar-\thebar }}
\\& \qquad+ \beta_t \E\ftwonormsq{\theta_t-\thestar} +\frac{1}{\beta_t}\brac{\frac{C_{\text{gap}}m\rho^k}{\lambda_{\min}} }^2
\\& \leq \frac{1}{2} C_\Theta \beta_t \E \Fbrac{\twonormsq{\theta_t-\thestar } }+\frac{1}{2} C_\Theta \beta_t \E \Fbrac{\twonormsq{ \thebar } }
+C_\Theta \beta_t \E\ftwonormsq{\theta_t-\thestar}
+\frac{1}{2} C_\Theta \beta_t \E \Fbrac{\twonormsq{\theta_t-\thestar } }
\\& \qquad +\frac{1}{2} C_\Theta \beta_t \E \Fbrac{\twonormsq{\thestar-\thebar } } + \beta_t \E\ftwonormsq{\theta_t-\thestar} +\frac{1}{\beta_t}\brac{\frac{C_{\text{gap}}m\rho^k}{\lambda_{\min}} }^2
\\& \myineq{\leq}{c} \frac{1}{2} C_\Theta \beta_t \E \Fbrac{\twonormsq{\theta_t-\thestar } }
+ \frac{1}{2} C_\Theta\beta_t  \E\ftwonormsq{\brac{ F_t}^{-1}\gradj }
  + C_\Theta\beta_t \E\ftwonormsq{\theta_t-\thestar}
  \\&\qquad+\frac{C_\Theta}{2}\beta_t  \E \Fbrac{\twonormsq{\theta_t-\thestar } } + \frac{C_\Theta\beta_t}{2} \brac{\frac{C_{\text{gap}}m\rho^k}{\lambda_{\min}} }^2+\beta_t \E\ftwonormsq{\theta_t-\thestar} 
 +\frac{1}{\beta_t}\brac{\frac{C_{\text{gap}}m\rho^k}{\lambda_{\min}} }^2
 \\ & \myineq{\leq}{d} \brac{2C_\Theta+1} \beta_t\E \Fbrac{\twonormsq{\theta_t-\thestar } }
 + \frac{C_\Theta}{2 \lambda_{\min}^2}\beta_t
\E\ftwonormsq{ \gradj}
+\brac{\frac{\beta_t C_\Theta}{2} +\frac{1}{\beta_t}}\brac{\frac{C_{\text{gap}}m\rho^k}{\lambda_{\min}} }^2
,
} where $(a)$ follows from \Cref{lm:acnacinnersmooth}; $(b)$ follows from \Cref{prop:bound}; $(c)$ follows from \Cref{prop:pg}, \Cref{prop:npg}, and \Cref{prop:bound}; $(d)$ follows from that $\twonorm{F_t^{-1}}\leq \frac{1}{\lambda_{\min}}$.

Plugging the above bounds   in \Cref{eq:nac:theta1}, we have  that
\eqenv{
\E&\ftwonormsq{\theta\tit-\thestara} 
  \\ & \leq \brac{1-\bar \lambda_{\min}\alpha_t  +\brac{4C_\Theta +2}\beta_t  }
  \E \ftwonormsq{\theta_t-\thestar}
  + \frac{C_\Theta}{\lambda^2_{\min}} \beta_t\E\ftwonormsq{\gradj}
  \\& +\frac{(k+1)^2 C^2_\phi}{\bar \lambda_{\min}}\alpha_t\E\flnormsq{J(\omega_t)-\eta_t}
  +\brac{ 8  +\frac{ 2 }{\beta_t}+ \beta_t C_\Theta   }\brac{\frac{C_{\text{gap}}m\rho^k}{\lambda_{\min}}}^2
+ (k+1)^2U_\delta^2 C_\phi^2 \alpha_t^2
\\&+\frac{4(k+1) C_\phi U_\delta C_{\text{gap}}m\rho^k}{\lambda_{\min}}\alpha_t
  +2(k+1) C_\phi B U_\delta C_\Theta \alpha_t\beta_t
+ 2\alpha_t G_t^\delta+2 C_\Theta^2  B^2 \beta_t^2
.}

This hence completes the proof. 
\end{proof}

\subsection{Sample Complexity of NAC}

In the following, we set $\alpha_t=\alpha$, $\beta_t=\beta$, and $\gamma_t=\gamma$ for any $t\geq 0$. We denote by $\D(\omega_t)=-\E_{D_{\pi^*}}\Fbrac{ \log\frac{\pi_{\omega_t}(a|s)}{\pi^*(a|s)}}$. Recall that $\hat t =\lceil \frac{3\log T}{\bar \lambda_{\min} \alpha}\rceil$. We first have the following lemma. 

\begin{lemma}\label{lemma:13}
Denote by $\widetilde T={\myceil{\frac{T}{\hat t\log T}}}\hat t\geq \frac{T}{\log T}$, then for any $t'\leq T-\widetilde T$, it holds that
\eqenv{
  \min_{t\leq T} &\E\Fbrac{J(\pi^*)-J(\omega_t)}
    \\&\leq \frac{\D(\omega_{t'})-\D(\omega_{t'+\widetilde T}) }{\beta \widetilde T}
    + \frac{C_\phi C_{\text{gap}} m \rho^k }{\lambda_{\min}}
    +2 C_\infty\sqrt{\varepsilon_{\text{actor}} }
     +\sqrt{2}C_\phi \brac{\frac{2e C_\phi^3 C_M  \hat t}{\widetilde T} }+ \sqrt{2}C_\phi \sqrt{\frac{4 B^2+  R_{\max}^2}{\widetilde T} }
   \\& 
     +\sqrt{2}C_\phi \sqrt{ \frac{1}{\widetilde T}\sum_{t=t'}^{t'+\widetilde T-1}  \sum_{j=t-\hat t}^{t-1} (1-q)^{t-j-1}\brac{\widetilde G^\omega_j +\widetilde G_j^\theta+  \widetilde G_j^\eta } }
     +2C_\phi \sqrt{ C_\phi^2 C_M\frac{e \hat t}{\widetilde T}\frac{\D(\omega_{t'})-\D(\omega_{t'+\widetilde T}) }{\beta \widetilde T}
     }
     \\&+2C_\phi  \sqrt{\frac{e \hat t C_\phi^2 C_M }{\widetilde T}  \brac{\frac{C_\phi C_{\text{gap}} m \rho^k }{\lambda_{\min}}
+\frac{B^2L_\phi}{2}\beta} }
     +\frac{B^2L_\phi}{2}\beta
    +\frac{2R_{\max}}{\sqrt T}
    .
}    
\end{lemma}

\begin{proof}
 Recall that $\pi^*=\arg\max_\pi J(\pi)$ and $A^{\pi_t}=Q^{\pi_t}(s,a)-V^{\pi_t}(s)$. We first have that
\eqenv{
\D&(\omega_t)-\D(\omega\tit)
= \E _{D_{\pi^*}}\Fbrac{\log \pi \tit (a|s)-\log\pi_t (a|s)  }
\\& \myineq{\geq}{a} \E _{D_{\pi^*}}\Fbrac{\nabla^\top \log \pi_t (a|s) }
\brac{\omega\tit-\omega_t}-\frac{L_\phi}{2}\twonormsq{\omega\tit-\omega_t }
\\&= \beta \E _{D_{\pi^*}}\Fbrac{\phi_t^\top (s,a)\theta_t}- \frac{L_\phi\beta^2}{2}\twonormsq{\theta_t}
\\ & \geq \beta \E _{D_{\pi^*}}\Fbrac{A^{\pi_t}(s,a)}
 +\beta \E _{D_{\pi^*}} \Fbrac{\phi_t^\top(s,a)\thestar-A^{\pi_t}(s,a) }
 +\beta \E _{D_{\pi^*}} \Fbrac{\phi_t^\top(s,a)(\theta_t-\thestar) }
  -\frac{ B^2L_\phi\beta^2}{2}
 \\& \myineq{=}{b} \beta\brac{J(\pi^*)-J(\pi_t) }
 +\beta \underbrace{\E _{D_{\pi^*}} \Fbrac{\phi_t^\top(s,a)\thestar-A^{\pi_t}(s,a) }}_{\text{I}}
    +\beta \underbrace{\E _{D_{\pi^*}} \Fbrac{\phi_t^\top(s,a)(\theta_t-\thestar)}}_{\text{II}}
 -\frac{ B^2L_\phi\beta^2}{2},\label{eq:115}
}
where $(a)$ follows from that 
\eqenv{
&\twonorm{\nabla\E_{D_{\pi^*}}[\log \pi_{\omega}(a|s)]-\nabla\E_{D_{\pi^*}}[\log \pi_{\omega'}(a|s)]}
\\&\qquad=\twonorm{\E_{D_{\pi^*}}[\phi_{\omega}(s,a)-\phi_{\omega'}(s,a)] }
\leq \E_{D_{\pi^*}}[L_\phi\twonorm{\omega-\omega'}]= L_\phi \twonorm{\omega-\omega'},
}
and $(b)$ follows from the fact that 
\eqenv{
\E _{D_{\pi^*}}\Fbrac{A^{\pi_t}(s,a)}&= \E _{D_{\pi^*}}\Fbrac{Q^{\pi_t}(s,a)-V^{\pi_t}(s) }
\\& =\E _{D_{\pi^*}}\Fbrac{R(s,a)-J(\pi_t) + V^{\pi_t}(s') -V^{\pi_t}(s) }
\\& =\E _{D_{\pi^*}}\Fbrac{R(s,a)}-J(\pi_t)+ \E _{(s,a)\sim D_{\pi^*}, s\sim P(\cdot|s,a)}\Fbrac{V^{\pi_t}(s')}-\E _{D_{\pi^*}}\Fbrac{V^{\pi_t}(s)}
\\& =J(\pi^*)-J(\pi_t)+ \E _{ D_{\pi^*}}\Fbrac{V^{\pi_t}(s')}-\E _{D_{\pi^*}}\Fbrac{V^{\pi_t}(s)}
\\& =J(\pi^*)-J(\pi_t).
}

To bound Term I, we first have that
\eqenv{
    &\lbrac{\E _{D_{\pi^*}} \Fbrac{\phi_t^\top(s,a)\thestar-A^{\pi_t}(s,a) }} \\
    &\leq 
    \lbrac{\E _{D_{\pi^*}} \Fbrac{\phi_t^\top(s,a)\thebar-A^{\pi_t}(s,a) }}+
     \lbrac{\E _{D_{\pi^*}} \Fbrac{\phi_t^\top(s,a)(\thebar-\thestar) }}
     \\& \leq \mynorm{\frac{D_{\pi^*}}{D_t}}_\infty\E_{D_{t} } \Fbrac{\lbrac{\phi_t^\top(s,a)\thebar -A^{\pi_t}(s,a) } }+ C_\phi \twonorm{\thebar-\thestar}
    \\& \myineq{\leq}{a} \mynorm{\frac{D_{\pi^*}}{D_t}}_\infty
    \sqrt{\E_{D_{t}}\flnormsq{\phi_t^\top(s,a)\thebar- A^{\pi_t}(s,a) } } + \frac{C_\phi C_{\text{gap}} m \rho^k }{\lambda_{\min}}
    \\& \myineq{\leq}{b} \mynorm{\frac{D_{\pi^*}}{D_t}}_\infty \sqrt{\varepsilon_{\text{actor}} }
    + \frac{C_\phi C_{\text{gap}} m \rho^k }{\lambda_{\min}},
}
where $(a)$ follows from that \cref{prop:bound}, 
 and $(b)$ follows from the definition of  $\thebar$ in \Cref{eq:thetabar},  the definition of  $\varepsilon_{\text{actor}}$ and the facts that
\eqenv{
    \E_{D_t}\flnormsq{\phi^\top_t(s,a)\thebar-A^{\pi_t}(s,a)}\leq \varepsilon_{\text{actor}}.
}

For term II, we  have that
\eqenv{
    \lbrac{\E _{D_{\pi^*}} \Fbrac{\phi_t^\top(s,a)(\theta_t-\thestar) }}
     \leq C_\phi \twonorm{\theta_t-\thestar }.
}

Plug the two bounds on terms I and II in \Cref{eq:115}, and we have that
\eqenv{
    \E\Fbrac{\D(\omega_t)}-\E\Fbrac{\D(\omega\tit) }
    \geq & \beta\brac{J(\pi^*)-\E\Fbrac{J(\omega_t) } }
    -\beta\mynorm{\frac{D_{\pi^*}}{D_t}}_\infty  \sqrt{\varepsilon_{\text{actor}} }
    -\beta \frac{C_\phi C_{\text{gap}} m \rho^k }{\lambda_{\min}}
    \\& -\beta C_\phi \E\ftwonorm{\theta_t-\thestar }
    -\frac{L_\phi}{2}\beta^2 B^2,
}
which implies 
 \eqenv{ 
    \beta \brac{J(\pi^*)-\E\Fbrac{J(\omega_t) } }
    \leq &\E\Fbrac{\D(\omega_t)}-\E\Fbrac{\D(\omega\tit) }
    +\beta \mynorm{\frac{D_{\pi^*}}{D_t}}_\infty  \sqrt{\varepsilon_{\text{actor}} }
    +\beta  \frac{C_\phi C_{\text{gap}} m \rho^k }{\lambda_{\min}}
    \\& +\beta  C_\phi \E\ftwonorm{\theta_t-\thestar }
    +\frac{B^2L_\phi}{2}\beta ^2 .\label{eq:nac:gap}
 }

Set $M\tit= \E\ftwonormsq{ \theta\tit-\thestara} +   \E \flnormsq{\eta\tit-J(\omega_{t+1})}  $, and we now aim to bound $M_t$. Combine the bounds we obtained in \Cref{eq:nacetaitera} and \Cref{eq:nac:thetaitera}, and we have that
\eqenv{\label{cons:cm}
   &M\tit
    \leq  \brac{1-\frac{1}{2}\bar \lambda_{\min}\alpha    }
  \E \ftwonormsq{\theta_t-\thestar}
  +  \frac{C_\Theta}{\lambda^2_{\min}} \beta\E\ftwonormsq{\gradj}
   +\frac{(k+1)^2 C^2_\phi}{\bar \lambda_{\min}}\alpha\E\flnormsq{J(\omega_t)-\eta_t} 
   \\& \qquad+  \brac{1-\gamma}\E\flnormsq{\eta_t-J(\omega_t)}
+ B^2\beta \E\ftwonormsq{\gradj}+\widetilde G_t^\theta+  \widetilde G_t^\eta
\\& =  \brac{1-\frac{1}{2}\bar \lambda_{\min}\alpha    }
  \E \ftwonormsq{\theta_t-\thestar}
   +\frac{(k+1)^2 C^2_\phi}{\bar \lambda_{\min}}\alpha\E\flnormsq{J(\omega_t)-\eta_t} 
   +  \brac{1-\gamma}\E\flnormsq{J(\omega_t)-\eta_t}
\\& \qquad+\brac{ \frac{C_\Theta}{\lambda^2_{\min}}+ B^2} \beta\E\ftwonormsq{\gradj}+\widetilde G_t^\theta+  \widetilde G_t^\eta
\\&\overset{(a)}{\leq}  \brac{1-\frac{1}{2}\bar \lambda_{\min}\alpha    }
  \E \ftwonormsq{\theta_t-\thestar}
   +\frac{(k+1)^2 C^2_\phi}{\bar \lambda_{\min}}\alpha\E\flnormsq{J(\omega_t)-\eta_t} 
   +  \brac{1-\gamma}\E\flnormsq{J(\omega_t)-\eta_t}
\\& \qquad+\brac{ \frac{C_\Theta}{\lambda^2_{\min}}+ B^2}\beta \brac{ \frac{ 2C_\phi^2 }{\beta}\brac{ \E \Fbrac{J(\omega\tit)}- \E \Fbrac{J(\omega_t)}}
 + 2 C^4_\phi\E\ftwonormsq{\theta_t-\thestar}}
 \\& \qquad +\brac{  \frac{C_\Theta}{\lambda^2_{\min}}+ B^2}\brac{ 2C_{\text{gap}}\frac{ C_\phi^4  m \rho^k\beta}{\lambda_{\min}}
 +{L_J C^2_\phi B^2\beta^2}}+\widetilde G_t^\theta+  \widetilde G_t^\eta
 \\&\leq  
    \brac{1-\frac{1}{2}\bar \lambda_{\min}\alpha + 2 C_\phi^4 C_M\beta }
    \E \ftwonormsq{\theta_t-\thestar}
    + { 2C_\phi^2 C_M }\brac{\E \Fbrac{J(\omega\tit)}- \E \Fbrac{J(\omega_t)} }
    \\&\qquad + \brac{1-\gamma + \frac{(k+1)^2 C^2_\phi}{\bar \lambda_{\min}}\alpha}\E\flnormsq{J(\omega_t)-\eta_t}
    + C_M \brac{ \frac{2C_{\text{gap}} C_\phi^4  m \rho^k\beta}{\lambda_{\min}}
 +{L_J C^2_\phi B^2\beta^2}}
 \\& \qquad+\widetilde G_t^\theta+  \widetilde G_t^\eta,
}
where $(a)$ is obtained by  \Cref{eq:94}, and $C_M= \frac{C_\Theta}{\lambda^2_{\min}}+ B^2 $. 
For convenience, we set 
\eqenv{\label{vari:gomega}
\widetilde G^\omega_t=C_M \brac{ \frac{2C_{\text{gap}} C_\phi^4  m \rho^k\beta}{\lambda_{\min}}
 + {L_J C^2_\phi B^2\beta^2}}.
}
We set $k=\lceil\frac{\log T}{1-\rho}\rceil$ and $T\geq \frac{2C^2_\phi dm}{\lambda_{\min}}$ such that $\bar \lambda_{\min}\geq \frac{1}{2} \lambda_{\min}$. Furthermore, we set the step sizes such that $\frac{1}{6}\bar \lambda_{\min} \alpha \geq 2 C_\phi^4 C_M\beta $ and  $\frac{\gamma}{2}\geq \max\varbrac{\frac{1}{3}\bar \lambda_{\min}\alpha, \frac{(k+1)^2 C^2_\phi}{\bar \lambda_{\min}}\alpha }$. Let $q=\frac{1}{3}\bar \lambda_{\min}\alpha$ and $\frac{\gamma}{2}\geq q$. Then the inequality in \Cref{cons:cm} can be written as
\eqenv{
    M\tit&\leq  
    \brac{1-q } M_t
    + { 2C_\phi^2 C_M }\brac{\E \Fbrac{J(\omega\tit)}- \E \Fbrac{J(\omega_t)} }
    + \widetilde G^\omega_t +\widetilde G_t^\theta+  \widetilde G_t^\eta. \label{eq:124}
}

Recall that $\hat t=\myceil{ \frac{1}{q}\log T}=\myceil{ \frac{3 \log T}{\bar \lambda_{\min}\alpha}}$. 
For $t\geq 2k+ \hat t$, we recursively apply \Cref{eq:124} for $\hat t$ times, and we have that
\eqenv{\label{vari:that}
    M_t\leq& \brac{1-q}^{\hat t} M_{t-\hat t}+
     2C_\phi^2 C_M\sum_{j=t-\hat t}^{t-1} (1-q)^{t-j-1}
    \brac{\E \Fbrac{J(\omega_{j+1})}- \E \Fbrac{J(\omega_j)} }
    \\&\qquad + \sum_{j=t-\hat t}^{t-1} (1-q)^{t-j-1}\brac{\widetilde G^\omega_j +\widetilde G_j^\theta+  \widetilde G_j^\eta }
    \\& \myineq{\leq}{a}  \frac{4 B^2+ R_{\max}^2}{T}+ 2C_\phi^2 C_M\sum_{j=t-\hat t}^{t-1} (1-q)^{t-j-1}
    \brac{\E \Fbrac{J(\omega_{j+1})}- \E \Fbrac{J(\omega_j)} }
    \\& \qquad +\sum_{j=t-\hat t}^{t-1} (1-q)^{t-j-1}\brac{\widetilde G^\omega_j +\widetilde G_j^\theta+  \widetilde G_j^\eta }, 
}
where $(a)$ follows from $(1-q)^{\hat t}\leq e^{-q\hat t}\leq e^{-\log T}= \frac{1}{T} $ and $M_i= \E\ftwonormsq{ \theta_i-\theta_i^*} +   \E \flnormsq{\eta_i-J(\omega_{i})}\leq 4 B^2+R^2_{\max}$ for $i=0,1,...,T$.

Denote the time length $\widetilde T={\myceil{\frac{T}{\hat t\log T}}}\hat t\geq \frac{T}{\log T}$. For any $t'\leq T-\widetilde T$, together with \Cref{eq:nac:gap} we have that
\eqenv{\label{eq:nac:d}
    &\frac{1}{\widetilde T}\sum_{t=t'}^{t'+\widetilde T-1} J(\pi^*)-\E\Fbrac{J(\omega_t) }
    \\& \leq
    \frac{\D(\omega_{t'})-\D(\omega_{t'+\widetilde T}) }{\beta \widetilde T}
    + \frac{C_\phi C_{\text{gap}} m \rho^k }{\lambda_{\min}}
    +\frac{1}{\widetilde T}\sum_{t=t'}^{t'+\widetilde T-1}\mynorm{\frac{D_{\pi^*}}{D_t}}_\infty  \sqrt{\varepsilon_{\text{actor}} }
    \\& \qquad +C_\phi \frac{1}{\widetilde T}\sum_{t=t'}^{t'+\widetilde T-1} \E\ftwonorm{\theta_t-\thestar }
    +\frac{B^2L_\phi}{2}\beta
    \\&  \myineq{\leq}{a}
    \frac{\D(\omega_{t'})-\D(\omega_{t'+\widetilde T}) }{\beta \widetilde T}
    + \frac{C_\phi C_{\text{gap}} m \rho^k }{\lambda_{\min}}
    +\frac{1}{\widetilde T}\sum_{t=t'}^{t'+\widetilde T-1}\mynorm{\frac{D_{\pi^*}}{D_t}}_\infty  \sqrt{\varepsilon_{\text{actor}} }
    \\& \qquad +C_\phi \sqrt{\frac{1}{\widetilde T}\sum_{t=t'}^{t'+\widetilde T-1} \E\ftwonormsq{\theta_t-\thestar}}
    +\frac{B^2L_\phi}{2}\beta
    \\&   \myineq{\leq}{b}
    \frac{\D(\omega_{t'})-\D(\omega_{t'+\widetilde T}) }{\beta \widetilde T}
    + \frac{C_\phi C_{\text{gap}} m \rho^k }{\lambda_{\min}}
    +  C_\infty \sqrt{\varepsilon_{\text{actor}} }
    \\& \qquad +C_\phi \sqrt{\frac{1}{\widetilde T}\sum_{t=t'}^{t'+\widetilde T-1} M_t}
    +\frac{B^2L_\phi}{2}\beta
    ,
}where $(a)$ follows from the rearrangement inequality and the fact for any random variable $X$, $\twonormsq{\E\Fbrac{X} }\leq \E \ftwonormsq{X} $ and $(b)$ follows from Assumption \ref{asm:explore}.

Moreover, for $2k+ \hat t \leq t'\leq T-\widetilde T$, summing \Cref{vari:that} w.r.t. $t$ from $t'$ to $t'+\tilde{T}-1$ implies  
\eqenv{
    &\frac{1}{\widetilde T}\sum_{t=t'}^{t'+\widetilde T-1}
    M_t\\
    & \leq  2C_\phi^2 C_M\frac{1}{\widetilde T}\sum_{t=t'}^{t'+\widetilde T-1} \sum_{j=t-\hat t}^{t-1} (1-q)^{t-j-1}
    \brac{\E \Fbrac{J(\omega_{j+1})}- \E \Fbrac{J(\omega_j)} }
    \\&\qquad
    + \frac{4 B^2+  R_{\max}^2}{\widetilde T}+   \frac{1}{\widetilde T}\sum_{t=t'}^{t'+\widetilde T-1}  \sum_{j=t-\hat t}^{t-1} (1-q)^{t-j-1}\brac{\widetilde G^\omega_j +\widetilde G_j^\theta+  \widetilde G_j^\eta }
    \\ & \myineq{=}{a} 
     2C_\phi^2 C_M\frac{1}{\widetilde T}\sum_{i=0}^{\hat t-1} \sum_{t=\hat t-i }^{t'+\widetilde T-1} (1-q)^{i}
    \brac{\E \Fbrac{J(\omega_{t-i})}- \E \Fbrac{J(\omega_{t-i-1})} }
    \\&\qquad
    + \frac{4 B^2+  R_{\max}^2}{\widetilde T}+\frac{1}{\widetilde T}\sum_{t=t'}^{t'+\widetilde T-1}  \sum_{j=t-\hat t}^{t-1} (1-q)^{t-j-1}\brac{\widetilde G^\omega_j +\widetilde G_j^\theta+  \widetilde G_j^\eta }
    \\& \myineq{=}{b} 
     2 C_\phi^2 C_M\frac{1}{\widetilde T}\sum_{i=0}^{\hat t-1}  (1-q)^{i}
    \brac{\E \Fbrac{J(\omega_{t'+\widetilde T-i-1})}- \E \Fbrac{J(\omega_{t'-i-1})} }
    \\&\qquad
    + \frac{4 B^2+  R_{\max}^2}{\widetilde T}+\frac{1}{\widetilde T}\sum_{t=t'}^{t'+\widetilde T-1}  \sum_{j=t-\hat t}^{t-1} (1-q)^{t-j-1}\brac{\widetilde G^\omega_j +\widetilde G_j^\theta+  \widetilde G_j^\eta }
    \\& \myineq{\leq}{c} 
     2C_\phi^2 C_M\frac{1}{\widetilde T}\sum_{i=0}^{\hat t-1}  (1-q)^{i}
    \brac{J(\pi^*)- \E \Fbrac{J(\omega_{t'-i-1})} }
    \\&\qquad
    + \frac{4 B^2+  R_{\max}^2}{\widetilde T}+\frac{1}{\widetilde T}\sum_{t=t'}^{t'+\widetilde T-1}  \sum_{j=t-\hat t}^{t-1} (1-q)^{t-j-1}\brac{\widetilde G^\omega_j +\widetilde G_j^\theta+  \widetilde G_j^\eta }
    \\& \leq 
     2C_\phi^2 C_M\frac{1}{\widetilde T}\sum_{i=0}^{\hat t-1}  
    \brac{J(\pi^*)- \E \Fbrac{J(\omega_{t'-i-1})} }
    \\&\qquad
    + \frac{4 B^2+  R_{\max}^2}{\widetilde T}+\frac{1}{\widetilde T}\sum_{t=t'}^{t'+\widetilde T-1}  \sum_{j=t-\hat t}^{t-1} (1-q)^{t-j-1}\brac{\widetilde G^\omega_j +\widetilde G_j^\theta+  \widetilde G_j^\eta }, \label{eq:x}
} where $(a)$ follows from that we set $i=t-j-1$, $(b)$ 
follows from the fact that $\sum_{t=\hat t-i }^{t'+\widetilde T-i-1}  \brac{\E \Fbrac{J(\omega_{t-i})}- \E \Fbrac{J(\omega_{t-i-1})} }= \E\Fbrac{J(\omega_{t'+\widetilde T-i-1}) }- \E \Fbrac{J(\omega_{t'-i-1})}$ and $(c)$ follows from $J(\pi^*)\geq J(\pi_{t'+\widetilde T-i-1}) $. 

\begin{lemma}\label{lm:lm14}
    We denote by  $X_{t'}= \frac{1}{\hat t}\sum_{t=t'-\hat t}^{t'-1} \brac{J(\pi^*)- \E \Fbrac{J(\omega_t)} } $ and  $Y_{t'}=\frac{1}{\widetilde T}\sum_{t=t'}^{t'+\widetilde T-1} \brac{J(\pi^*)- \E \Fbrac{J(\omega_t)} } $.  
We have that
  there must exist $t'+\hat{t}\leq t''\leq t'+\widetilde T$ s.t. 
$X_{t''}\leq Y_{t'} $.
\end{lemma}

\begin{proof}
   We prove the lemma by contraction.  Assume that there does not exist $t'+\hat{t}\leq t''\leq t'+\widetilde T$ s.t. 
$X_{t''}\leq Y_{t'} $. Then, for any $t'+\hat{t}\leq t''\leq t'+\widetilde T$, $X_{t''}> Y_{t'} $.

Next, recall $\widetilde T={\myceil{\frac{T}{\hat t\log T}}}\hat t$, then we have that 
\eqenv{
Y_{t'}&= \frac{1}{\widetilde T}\sum_{t=t'}^{t'+\widetilde T-1} \brac{J(\pi^*)- \E \Fbrac{J(\omega_t)} }
= \frac{1}{\myceil{\frac{T}{\hat t\log T}}}\sum_{\tau=1}^{\myceil{\frac{T}{\hat t\log T}}}\frac{1}{\hat t}\sum_{t=t'+(\tau-1) \hat t}^{t'+\tau \hat t-1}\brac{J(\pi^*)- \E \Fbrac{J(\omega_t)} }
\\& = \frac{1}{\myceil{\frac{T}{\hat t\log T}}}\sum_{\tau=1}^{\myceil{\frac{T}{\hat t\log T}}} X_{t'+\tau \hat t}\myineq{>}{a} \frac{1}{\myceil{\frac{T}{\hat t\log T}}}\sum_{\tau=1}^{\myceil{\frac{T}{\hat t\log T}}} Y_{t'}=Y_{t'},
} where $(a)$ follows from $t'+\hat{t}\leq t'+\tau \hat t\leq t'+\widetilde T$, $X_{t+\tau \hat t}> Y_{t'}  $ for $\tau=1,...,\myceil{\frac{T}{\hat t\log T}}$. This hence results in a contradiction, which completes the proof.
\end{proof}
We then discuss the following two cases.

\subsubsection*{Case 1:}[For any $2k+ \hat t \leq  t' \leq   T -\widetilde T$, it holds that $ e Y _{ t'}<  X_{ t'}$.] 
     
    Then, for $\widetilde t_0=2k+ \hat t$, we have
    $X_{\widetilde t_0}> e Y _{\widetilde t_0} $.  Recall that $\widetilde T={\myceil{\frac{T}{\hat t\log T}}}\hat t\leq \brac{\frac{T}{\hat t\log T}+1}\hat t\leq \frac{T}{ \log T}+\hat t\leq \frac{2T}{ \log T}$ for large $T$. Then, $\frac{T}{\widetilde T}\geq \frac{\log T}{2} \geq \lfloor \frac{\log T}{2} \rfloor$. 
    Thus, there  exist $\widetilde t_0+\hat{t} \leq  \widetilde t_1\leq \widetilde t_0 +\widetilde T$, s.t.
    $X_{\widetilde t_1}\myineq{\leq}{a} Y _{\widetilde t_0}\myineq{<}{b} \frac{1}{e} X _{\widetilde t_0}$, where $(a)$ follows from \Cref{lm:lm14} and $(b)$ follows from the condition of case 1. 
    Recursively applying the above argument, we have that for $j=0,1,...,\lfloor \frac{\log T}{2} \rfloor$, there exists $\widetilde t_j+\hat t\leq \widetilde t_{j+1}\leq \widetilde t_j +\widetilde T$, such that $X_{\widetilde t_{j+1}}\leq Y _{\widetilde t_j}< \frac{1}{e} X _{\widetilde t_j}$. This further implies that 
    \eqenv{
    X _{\widetilde t_0}> e Y _{\widetilde t_0}
    \geq e X _{\widetilde t_1}> e^2
    Y _{\widetilde t_1}\geq ...\geq e^j Y _{\widetilde t_j}> e^{j+1}Y _{\widetilde t_j}\geq...\geq e ^{{\lfloor{\frac{\log T}{2} }\rfloor}}X_{\widetilde t_{\lfloor{\frac{\log T}{2} }\rfloor}} > e^{{\lfloor{\frac{\log T}{2} }\rfloor}+1}Y_{\widetilde t_{\lfloor{\frac{\log T}{2} }\rfloor}}.\label{eq:nac:151}
    }
    Then, by \Cref{eq:nac:151}, we can conclude that
    \eqenv{Y_{\widetilde t_{\lfloor{\frac{\log T}{2} }\rfloor}}:
    =\frac{1}{\widetilde T}\sum_{t=\widetilde t_{\lfloor{\frac{\log T}{2} }\rfloor}}^{\widetilde t_{\lfloor{\frac{\log T}{2} }\rfloor}+\widetilde T-1} \brac{J(\pi^*)- \E \Fbrac{J(\omega_t)} }
    \leq \frac{1}{e^{{\lfloor{\frac{\log T}{2}}\rfloor}+1}}X_{\widetilde t_0}
    \leq \frac{1}{\sqrt T}X _{\widetilde t_0}.}

     Note that $X_{\widetilde t}\leq J(\pi ^*)\leq R_{\max}$, hence we have that
     \eqenv{\label{eq:case1}
     Y_{\widetilde t_{\lfloor{\frac{\log T}{2}}\rfloor}}
     = \frac{1}{\widetilde T}\sum_{t=\widetilde t_{\lfloor{\frac{\log T}{2}}\rfloor}}^{\widetilde t_{\lfloor{\frac{\log T}{2}}\rfloor}+\widetilde T-1} \brac{J(\pi^*)- \E \Fbrac{J(\omega_t)} }
     \leq \frac{R_{\max}}{\sqrt T}. 
     }
     This further implies that
     \eqenv{
        \min_{t\leq T} \E \Fbrac{J(\pi^*)-J(\omega_t) } \leq \frac{R_{\max}}{\sqrt T}. 
     }
     This hence completes the proof of \Cref{thm:nac} under Case 1.

\subsubsection*{Case 2:} [There exists some $2k+ \hat t \leq  t' \leq T -\widetilde T $ s.t. $ X _{ t '}\leq e Y_{ t'}$.]
    
Define $Z_{t'}=\frac{1}{\widetilde T}\sum_{t=t'}^{t'+\widetilde T-1} M_t$. From \Cref{eq:x}, we obtain that
\eqenv{\label{eq:eq1271}
   Z_{t'}&= \frac{1}{\widetilde T}\sum_{t=t'}^{t'+\widetilde T-1}
    M_t  \leq 
    2 C_\phi^2 C_M\frac{1}{\widetilde T}\sum_{i=0}^{\hat t-1} 
    \brac{J(\pi^*)- \E \Fbrac{J(\omega_{t'-i-1})} }
    \\&\qquad
    + \frac{4 B^2+  R_{\max}^2}{\widetilde T}+\frac{1}{\widetilde T}\sum_{t=t'}^{t'+\widetilde T-1}  \sum_{j=t-\hat t}^{t-1} (1-q)^{t-j-1}\brac{\widetilde G^\omega_j +\widetilde G_j^\theta+  \widetilde G_j^\eta }
    \\& =
     2C_\phi^2 C_M\frac{\hat t}{\widetilde T}X_{t'}
    + \frac{4 B^2+  R_{\max}^2}{\widetilde T}+\frac{1}{\widetilde T}\sum_{t=t'}^{t'+\widetilde T-1}  \sum_{j=t-\hat t}^{t-1} (1-q)^{t-j-1}\brac{\widetilde G^\omega_j +\widetilde G_j^\theta+  \widetilde G_j^\eta }
    \\& \myineq{\leq}{a} 
     2C_\phi^2 C_M\frac{e \hat t}{\widetilde T}Y_{t'}
    + \frac{4 B^2+  R_{\max}^2}{\widetilde T}+\frac{1}{\widetilde T}\sum_{t=t'}^{t'+\widetilde T-1}  \sum_{j=t-\hat t}^{t-1} (1-q)^{t-j-1}\brac{\widetilde G^\omega_j +\widetilde G_j^\theta+  \widetilde G_j^\eta }, 
} where $(a)$ follows from the condition of Case 2.

Next, from \Cref{eq:nac:d}, we have that
\eqenv{\label{eq:eq1281}
    Y_{t'}&=\frac{1}{\widetilde T}\sum_{t=t'}^{t'+\widetilde T-1} J(\pi^*)-\E\Fbrac{J(\omega_t) } 
    \\& \leq \frac{\D(\omega_{t'})-\D(\omega_{t'+\widetilde T}) }{\beta \widetilde T}
    + \frac{C_\phi C_{\text{gap}} m \rho^k }{\lambda_{\min}}
    + C_\infty\sqrt{\varepsilon_{\text{actor}} }
     +C_\phi \sqrt{\frac{1}{\widetilde T}\sum_{t=t'}^{t'+\widetilde T-1} M_t}
    +\frac{B^2L_\phi}{2}\beta
    \\& =\frac{\D(\omega_{t'})-\D(\omega_{t'+\widetilde T}) }{\beta \widetilde T}
    + \frac{C_\phi C_{\text{gap}} m \rho^k }{\lambda_{\min}}
    + C_\infty\sqrt{\varepsilon_{\text{actor}} }
      +C_\phi \sqrt{Z_{t'}}
    +\frac{B^2L_\phi}{2}\beta.
}

Plug \Cref{eq:eq1281} in \Cref{eq:eq1271}, and we have that
\eqenv{
   Z_{t'} &=\frac{1}{\widetilde T}\sum_{t=t'}^{t'+\widetilde T-1}
    M_t   
     \\& \leq  \frac{4 B^2+  R_{\max}^2}{\widetilde T}+\frac{1}{\widetilde T}\sum_{t=t'}^{t'+\widetilde T-1}  \sum_{j=t-\hat t}^{t-1} (1-q)^{t-j-1}\brac{\widetilde G^\omega_j +\widetilde G_j^\theta+  \widetilde G_j^\eta }+\frac{2e  C_\phi^2 C_M\hat t}{\widetilde T}  \brac{\frac{C_\phi C_{\text{gap}} m \rho^k }{\lambda_{\min}}+\frac{B^2L_\phi}{2}\beta}
    \\& \quad+ 2C_\phi^2 C_M\frac{e \hat t}{\widetilde T}\brac{\frac{\D(\omega_{t'})-\D(\omega_{t'+\widetilde T}) }{\beta \widetilde T}
    + C_\infty\sqrt{\varepsilon_{\text{actor}} }
     +C_\phi \sqrt{Z_{t'}}
    }
    \\& \myineq{\leq}{a}  \frac{4 B^2+  R_{\max}^2}{\widetilde T} + \frac{1}{\widetilde T}\sum_{t=t'}^{t'+\widetilde T-1}  \sum_{j=t-\hat t}^{t-1} (1-q)^{t-j-1}\brac{\widetilde G^\omega_j +\widetilde G_j^\theta+  \widetilde G_j^\eta }+ \frac{2e  C_\phi^2 C_M \hat t}{\widetilde T}  \brac{\frac{C_\phi C_{\text{gap}} m \rho^k }{\lambda_{\min}}
+\frac{B^2L_\phi}{2}\beta}
\\&\quad + \brac{\frac{2e C_\phi^3 C_M  \hat t}{\widetilde T}}^2+\frac{1}{2} Z_{t'}
   +\frac{ C^2_\infty \varepsilon_{\text{actor}}}{2 C^2_\phi}
    + 2C_\phi^2 C_M\frac{e \hat t}{\widetilde T}\brac{\frac{\D(\omega_{t'})-\D(\omega_{t'+\widetilde T}) }{\beta \widetilde T}
    },
} where $(a)$ follows from $ xy\leq \frac{x^2+y^2}{2}$

Thus, it follows that
\eqenv{
     &Z_{t'}
     \leq \frac{8 B^2+  2R_{\max}^2}{\widetilde T} + \frac{2}{\widetilde T}\sum_{t=t'}^{t'+\widetilde T-1}  \sum_{j=t-\hat t}^{t-1} (1-q)^{t-j-1}\brac{\widetilde G^\omega_j +\widetilde G_j^\theta+  \widetilde G_j^\eta }
   + 2\brac{\frac{2e C_\phi^3 C_M  \hat t}{\widetilde T}}^2+ \frac{ C^2_\infty \varepsilon_{\text{actor}}}{C^2_\phi}
    \\&\quad  + 4C_\phi^2 C_M\frac{e \hat t}{\widetilde T}\brac{\frac{\D(\omega_{t'})-\D(\omega_{t'+\widetilde T}) }{\beta \widetilde T}
    }+ \frac{4e  C_\phi^2 C_M \hat t}{\widetilde T}  \brac{\frac{C_\phi C_{\text{gap}} m \rho^k }{\lambda_{\min}}
+\frac{B^2L_\phi}{2}\beta}.
}
Then, we get that
\eqenv{\label{eq:case2}
    Y_{t'}&\leq \frac{\D(\omega_{t'})-\D(\omega_{t'+\widetilde T}) }{\beta \widetilde T}
    + \frac{C_\phi C_{\text{gap}} m \rho^k }{\lambda_{\min}}
    + C_\infty\sqrt{\varepsilon_{\text{actor}} }
     +C_\phi \sqrt{\frac{1}{\widetilde T}\sum_{t=t'}^{t'+\widetilde T-1} M_t}
    +\frac{B^2L_\phi}{2}\beta
    \\& \leq 
    \frac{\D(\omega_{t'})-\D(\omega_{t'+\widetilde T}) }{\beta \widetilde T}
    + \frac{C_\phi C_{\text{gap}} m \rho^k }{\lambda_{\min}}
    +2 C_\infty\sqrt{\varepsilon_{\text{actor}} }
     +C_\phi \sqrt{2}\brac{\frac{2 e C_\phi^3 C_M  \hat t}{\widetilde T} }
   \\& 
     +C_\phi \sqrt{2}\sqrt{ \frac{1}{\widetilde T}\sum_{t=t'}^{t'+\widetilde T-1}  \sum_{j=t-\hat t}^{t-1} (1-q)^{t-j-1}\brac{\widetilde G^\omega_j +\widetilde G_j^\theta+  \widetilde G_j^\eta } }
    + C_\phi \sqrt{2}\sqrt{\frac{4 B^2+  R_{\max}^2}{\widetilde T} }
    \\& +2C_\phi  \sqrt{ C_\phi^2 C_M\frac{e \hat t}{\widetilde T}\frac{\D(\omega_{t'})-\D(\omega_{t'+\widetilde T}) }{\beta \widetilde T}
     }
     +2C_\phi \sqrt{\frac{e \hat t C_\phi^2 C_M }{\widetilde T}  \brac{\frac{C_\phi C_{\text{gap}} m \rho^k }{\lambda_{\min}}
+\frac{B^2L_\phi}{2}\beta} }
     +\frac{B^2L_\phi}{2}\beta,
}
 which proves the claim under Case 2.

Thus, combine the Case 1 result in  \Cref{eq:case1} and the Case 2 result in \Cref{eq:case2},  we have that
\eqenv{\label{eq:nac:final}
    \min_{t\leq T} &\E\Fbrac{J(\pi^*)-J(\omega_t)}
    \\&\leq \frac{\D(\omega_{t'})-\D(\omega_{t'+\widetilde T}) }{\beta \widetilde T}
    + \frac{C_\phi C_{\text{gap}} m \rho^k }{\lambda_{\min}}
    +2 C_\infty\sqrt{\varepsilon_{\text{actor}} }
     +C_\phi \sqrt{2}\brac{\frac{2e C_\phi^3 C_M  \hat t}{\widetilde T} }+ C_\phi \sqrt{2}\sqrt{\frac{4 B^2+  R_{\max}^2}{\widetilde T} }
   \\& 
     +C_\phi \sqrt{2}\sqrt{ \frac{1}{\widetilde T}\sum_{t=t'}^{t'+\widetilde T-1}  \sum_{j=t-\hat t}^{t-1} (1-q)^{t-j-1}\brac{\widetilde G^\omega_j +\widetilde G_j^\theta+  \widetilde G_j^\eta } }
     +2 C_\phi \sqrt{ C_\phi^2 C_M\frac{e \hat t}{\widetilde T}\frac{\D(\omega_{t'})-\D(\omega_{t'+\widetilde T}) }{\beta \widetilde T}
     }
     \\&+2C_\phi\sqrt{\frac{e \hat t C_\phi^2 C_M }{\widetilde T}  \brac{\frac{C_\phi C_{\text{gap}} m \rho^k }{\lambda_{\min}}
+\frac{B^2L_\phi}{2}\beta} }
     +\frac{B^2L_\phi}{2}\beta
    +\frac{2R_{\max}}{\sqrt T}
    .
}
which completes the proof.
\end{proof}

Next, we prove \Cref{thm:nac}. 
\begin{theorem}(Restatement of \Cref{thm:nac})
     Consider the NAC algorithm in  \Cref{alg:AC} with constant step sizes that $\gamma\geq\alpha\geq\beta $, then it holds that
   \eqenvnonum{
   \min_{t\leq T }& \E\Fbrac{J(\pi^*)-J(\omega_t) }
     \leq \mathcal{O}\brac{\frac{\log ^2T}{T\alpha}}
     +\mathcal{O}\brac{\frac{\log T}{T\beta}}
    + \mathcal{O}\brac{\frac{\sqrt{\log ^3T}}{T\sqrt{\alpha\beta}}}
    +\mathcal{O}\brac{\sqrt{\frac{\beta \log T}{T\alpha}}}
    +\mathcal{O}\brac{\sqrt{\frac{\log T}{T}} }
    \\ & + \mathcal{O}\brac{\frac{\sqrt{\gamma\beta\log ^2T} }{\sqrt{\alpha}}}
    + \mathcal{O}\brac{\frac{\gamma\sqrt{\log T} }{\sqrt{\alpha}}}
     +\mathcal{O}\brac{\frac{\beta }{\sqrt{\alpha}}}
    + \mathcal{O}\brac{\sqrt{\alpha\log ^{3} T }}
     + \mathcal{O}\brac{\sqrt{\beta\log ^{3} T} }
    \\
     &+\mathcal{O}\brac{\sqrt{(m\rho^k) \log T} }
     + \mathcal{O}\brac{\sqrt{\frac{(m\rho^k)\beta}{\alpha}} }
      + \mathcal{O}\brac{\sqrt{\frac{(m\rho^k)\gamma}{\alpha}} }
     + \mathcal{O}\brac{\frac{m\rho^k}{\sqrt{\alpha\beta}} }+
     \mathcal{O}\brac{ 
    \sqrt{\varepsilon_{\text{actor}}}}.
   }
   If we set $\gamma=\mathcal{O}(T^{-\frac{2}{3}}\log T),  \alpha=\mathcal{O}(T^{-\frac{2}{3}}\log^{-1}T),  \beta =\mathcal{O}(T^{-\frac{2}{3}}\log^{-1}T)$, we have 
    \eqenv{
    \min_{t\leq T } J(\pi^*)-J(\omega_t)\leq \mathcal{O}\brac{{T^{-\frac{1}{3}}}\log ^3T}
    + \mathcal{O}\brac{\sqrt{\varepsilon_{\text{actor}}}}.
    }
\end{theorem}

\begin{proof}

From \Cref{lemma:13}, we have that 
\eqenv{
     \min_{t\leq T} &\E\Fbrac{J(\pi^*)-J(\omega_t)}
    \\&\leq \frac{\D(\omega_{t'})-\D(\omega_{t'+\widetilde T}) }{\beta \widetilde T}
    + \frac{C_\phi C_{\text{gap}} m \rho^k }{\lambda_{\min}}
    +2 C_\infty\sqrt{\varepsilon_{\text{actor}} }
     +C_\phi \sqrt{2}\brac{\frac{2e C_\phi^3 C_M  \hat t}{\widetilde T} }+ C_\phi \sqrt{2}\sqrt{\frac{4 B^2+  R_{\max}^2}{\widetilde T} }
   \\& 
     +C_\phi \sqrt{2}\sqrt{ \frac{1}{\widetilde T}\sum_{t=t'}^{t'+\widetilde T-1}  \sum_{j=t-\hat t}^{t-1} (1-q)^{t-j-1}\brac{\widetilde G^\omega_j +\widetilde G_j^\theta+  \widetilde G_j^\eta } }
     +2C_\phi \sqrt{ C_\phi^2 C_M\frac{e \hat t}{\widetilde T}\frac{\D(\omega_{t'})-\D(\omega_{t'+\widetilde T}) }{\beta \widetilde T}
     }
     \\&+2C_\phi \sqrt{\frac{e \hat t C_\phi^2 C_M }{\widetilde T}  \brac{\frac{C_\phi C_{\text{gap}} m \rho^k }{\lambda_{\min}}
+\frac{B^2L_\phi}{2}\beta} }
     +\frac{B^2L_\phi}{2}\beta
    +\frac{2R_{\max}}{\sqrt T}
    .
}    

We then set the stepsize as follows: 
\eqenv{\label{eq:nac:stpesize}
    \gamma &=T^{-\frac{2}{3}}\log T; \\
    \alpha &= \min \varbrac{\frac{\bar \lambda_{\min}}{(k+1)^2 C^2_\phi}\gamma,\frac{3}{2\bar\lambda_{\min}}}=\mathcal{O}\brac{T^{-\frac{2}{3}}\log^{-1} T };\\
    \beta &=\min\varbrac{\frac{\bar\lambda_{\min}}{4(2C_\Theta+1)}\alpha,  \frac{\bar \lambda_{\min}}{12 C_\phi^4 C_M}\alpha} =\mathcal{O}\brac{T^{-\frac{2}{3}}\log^{-1} T }.
}

Recall that 
\eqenv{
&k=\lceil \frac{\log T}{1-\rho}\rceil;\\
&q= \frac{\bar \lambda_{\min} \alpha}{2} =\mathcal{O}\brac{\alpha };
\\& \hat t=\myceil{ \frac{1}{q}\log T}=\mathcal{O}\brac{\frac{\log T}{\alpha} }; \\& \widetilde T={\myceil{\frac{T}{\hat t\log T}}}\hat t=\mathcal{O}\brac{\frac{T}{\log T} }.
}

Applying the above stepsizes in \Cref{eq:nac:stpesize}, for $t\leq T $, we can have that
\eqenv{
 &\widetilde G^\eta_t=\mathcal{O}\brac{(m\rho^k)\gamma +\beta^2+k^2\beta\gamma+k \gamma^2 } =\mathcal{O}\brac{T^{-\frac{4}{3}} \log^3 T };
 \\& \widetilde G_t^\theta=\mathcal{O}\brac{k (m\rho^k)\alpha +\frac{(m \rho^k)^2}{\beta}+k^3\alpha^2+k^3 \alpha\beta+\beta^2 }=\mathcal{O}\brac{T^{-\frac{4}{3}} \log T };
 \\& \widetilde G_t^\omega=\mathcal{O}\brac{(m\rho^k)\beta+\beta^2}=\mathcal{O}\brac{T^{-\frac{4}{3}}\log^{-2} T}.
}

Besides,  we have that $\D (\omega_i)=\E_{D_{\pi^*}}\Fbrac{\log\brac{\frac{\pi^*(a|s)}{\pi_{\omega_i}(a|s)}
}}\leq \E_{D_{\pi^*}}\Fbrac{\log\brac{\frac{D_{\pi^*}(s,a)}{D_{\pi_{\omega_i}}(s,a)}}}\myineq{\leq}{a}\log C_\infty$,
 where $(a)$ follows from Assumption \ref{asm:explore}. 

Thus, it holds that
\eqenv{
\sqrt{ \frac{1}{\widetilde T}\sum_{t=t'}^{t'+\widetilde T-1}  \sum_{j=t-\hat t}^{t-1} (1-q)^{t-j-1}\brac{\widetilde G^\omega_j +\widetilde G_j^\theta+  \widetilde G_j^\eta } }=\mathcal{O}\brac{T^{-\frac{1}{3}} \log^2 T}.
}

Plugging the above equations to \Cref{eq:nac:final}
, we have that
\eqenv{
    \min_{t\leq T} \E\Fbrac{J(\pi^*)-J(\omega_t)}
    \leq\brac{T^{-\frac{1}{3}}\log^{3} T}+\mathcal{O}\brac{\sqrt{\varepsilon_{\text{actor}}} }. 
}
This concludes the proof.
\end{proof}

\section{Proof of Lemmas}\label{sec:secd}
\begin{proof}[Proof of Lemma \ref{Prop:lip}]
    Recall the definition of $\nabla J(\omega)$ in \Cref{eq:gradient}:
    \eqenv{
    \nabla J(\omega)= \E_{D_{\pi_\omega}}\Fbrac{ Q^{\pi_\omega}(s,a) \phi_\omega(s,a)},
    }
which implies that
    \eqenv{
        \twonorm{\nabla J(\omega)}
        & = \twonorm{  \E_{D_{\pi_\omega}}\Fbrac{ Q^{\pi_\omega}(s,a)\phi_\omega(s,a) }}
        \\& \myineq{=}{a}
        \twonorm{\E_{D_{\pi_\omega}}\Fbrac{\phi^\top_\omega(s,a) \bar \theta^*_\omega \phi_\omega(s,a) }} 
        \\ &\leq C_\phi^2 \brac{\twonorm{ \theta^*_\omega}+\twonorm{\bar \theta^*_\omega-\theta^*_\omega}}
        \\& \myineq{\leq}{b}   C^2_\phi \brac{B+ C_{\text{gap}}\frac{ m \rho^k}{\lambda_{\min}}}=C_J,
    }
    where $(a)$ follows from \Cref{eq:eq5} and $(b)$ follows from \Cref{prop:bound}.
    It hence proves the first claim. 
The second claim is proved in Lemma A.1 in \citep{wu2020finite}.
\end{proof}

\begin{proof}[Proof of Lemma\ref{lm:acnacinnersmooth}]
    Recall that $F_{\omega} = \E_{D_{\pi_\omega}}\Fbrac{\phi_{\omega}(s,a) \phi^\top_{\omega}(s,a) }$. From the definition of $ \bar \theta^*_\omega$ in \Cref{eq:thetabar}, it can be verified that
    \eqenv{
         \bar \theta^*_{\omega}=F_{\omega}^{-1} \nabla J(\omega). 
    }
Hence
    \eqenv{\label{eq:21}
       &\twonorm{ \bar\theta^*_{\omega}- \bar\theta^*_{\omega'} }
        \\ & \qquad = \twonorm{F_{\omega}^{-1} \nabla J(\omega)- F_{\omega'}^{-1} \nabla J({\omega'}) }
        \\ & \qquad \leq\twonorm{F_{\omega}^{-1}\nabla J(\omega)- F_{\omega'}^{-1} \nabla J({\omega}) }+
 \twonorm{F_{\omega'}^{-1} \nabla J({\omega})- F_{\omega'}^{-1} \nabla J({\omega'}) }
        \\ & \qquad \myineq{\leq }{a}\twonorm{F_{\omega}^{-1} } \twonorm{F_{\omega'}^{-1} } \twonorm{F_{\omega}-F_{\omega'} }\twonorm{\nabla J(\omega )}
        + \twonorm{(F_{\omega'})^{-1} }\twonorm{\nabla J({\omega})-\nabla J({\omega'}) },
    }
    where $(a)$ follows from \Cref{eq:invinequ}. 
   Note that  $F_{\omega}$ can be shown to be Lipschitz as follows: 
    \eqenv{\label{eq:22}
        \twonorm{ F_{\omega}-F_{\omega'}}& =
        \twonorm{\E_{D_{\pi_{\omega}}}\Fbrac{\phi_{{\omega}}(s,a) \phi^\top_{{\omega}}(s,a) }-\E_{D_{\pi_{\omega'}}}\Fbrac{\phi_{{\omega'}}(s,a) \phi^\top_{{\omega'}}(s,a) } }
        \\& \leq 
        \twonorm{\E_{D_{\pi_{\omega}}}\Fbrac{\phi_{{\omega}}(s,a) \phi^\top_{{\omega}}(s,a) }-\E_{D_{\pi_{\omega}}}\Fbrac{\phi_{{\omega'}}(s,a) \phi^\top_{{\omega'}}(s,a) } }
        \\&\quad+
        \twonorm{\E_{D_{\pi_{\omega}}}\Fbrac{\phi_{{\omega'}}(s,a) \phi^\top_{{\omega'}}(s,a) }-\E_{D_{\pi_{\omega'}}}\Fbrac{\phi_{{\omega'}}(s,a) \phi^\top_{{\omega'}}(s,a) } }
        \\& \leq 2 C_\phi L_\phi\twonorm{\omega'-\omega }
        + C^2_\phi \tvnorm{D_{\pi_\omega}-D_{\pi_{\omega'}} }
        \\& \myineq{\leq}{a}
        2 C_\phi L_\phi\twonorm{\omega'-\omega }
        + C^2_\phi L_\pi \twonorm{\omega-\omega' },
    }
    where $(a)$ follows from \citep{zou2019finite} and Theorem 1 in \citep{li2021faster}, that 
    \eqenv{
        \tvnorm{D_{\pi_\omega}-D_{\pi_{\omega'}} }\leq
        L_\pi \twonorm{\omega-\omega' }.
    }


Hence combining \Cref{eq:21}, \Cref{eq:22} and Lemma \ref{Prop:lip}, we obtain that
    \eqenv{
         \twonorm{\bar \theta^*_\omega-\bar \theta^*_{\omega'}}& \leq 
         \brac{\frac{C_J}{\lambda^2_{\min}}
         \brac{2 C_\phi L_\phi+C^2_\phi L_\pi }+ \frac{L_J}{\lambda_{\min}}}\twonorm{\omega-\omega'}.
    }

This completes the proof.
  
\end{proof}
\begin{proof}[Proof of Lemma \ref{lm:eig}]
From the definition of $H_\omega$ in \Cref{eq:aw}, we first have that
    \eqenv{
    H_{\omega}&=
    \E_{D_{\pi_\omega}} \Fbrac{\E \Fbrac{\phi_{\omega}(s_0,a_0)\brac{\phi_{\omega}(s_k,a_k)-\phi_{\omega}(s_0,a_0) }^\top 
    |s_0=s,a_0=a,\pi_\omega }}
    \\& \myineq{=}{a} 
    \E_{D_{\pi_\omega}} \Fbrac{\phi_\omega(s,a)\brac{\E \Fbrac{ \phi_\omega^\top(s_k,a_k)|s_0=s,a_0=a,\pi_\omega}-\E_{D_{\pi_\omega}}\Fbrac{\phi_\omega^\top(s,a)} }}
    \\& \quad-\E_{D_{\pi_\omega}} \Fbrac{\phi_\omega(s,a)\phi_\omega^\top(s,a)},
    }
where $(a)$ follows from $ \phi_\omega(s,a)=\nabla \log \pi_\omega(a|s)$  and $\E_{D_{\pi_\omega}}\Fbrac{\phi_\omega^\top(s,a) f(s)}=0 $, where $f(s)$ is the function which is not determined by action $a$.        

Define $\Delta H_\omega= \E_{D_{\pi_\omega}} \Fbrac{\phi_\omega(s,a)\brac{\E \Fbrac{ \phi_\omega^\top(s_k,a_k)|s_0=s,a_0=a,\pi_\omega}-\E_{D_{\pi_\omega}}\Fbrac{\phi_\omega^\top(s,a)} }}$. Thus, 
\eqenv{
  \frac{H_\omega+H_\omega^\top}{2}=\frac{\Delta H_\omega+ \Delta H_\omega^\top}{2}-\E_{D_{\pi_\omega}} \Fbrac{\phi_\omega(s,a)\phi_\omega^\top(s,a)}.
}

For any symmetric matrices $X$ and $Y$, $ \lambda_{\max}\brac{X+Y }\leq \lambda_{\max}\brac{X}+\lambda_{\max}\brac{Y} $.
 Thus,   we have that
\eqenv{
\lambda_{\max}\brac{\frac{H_\omega+H_\omega^\top}{2} } & \leq 
\lambda_{\max}\brac{\frac{\Delta H_\omega+ \Delta H_\omega^\top}{2} }
 +\lambda_{\max}\brac{-\E_{D_{\pi_\omega}} \Fbrac{\phi_\omega(s,a)\phi_\omega^\top(s,a)} } 
\\& \leq C_\phi^2
\E\Fbrac{\tvnorm{P(s_k,a_k|s_0=s,a_0=s,\pi_\omega), D_{\pi_\omega}}}-\lambda_{\min}
\\& \leq d C_\phi^2 m \rho^k -\lambda_{\min}=-\bar\lambda_{\min} ,
   } where last inequality follows from Assumption \ref{asm:ergodic}.
\end{proof}

\begin{proof}[Proof of Lemma \ref{lm:tv}]

Conditioned on $(s_{t-k},a_{t-k})$, the sample trajectory in \Cref{alg:AC} is generated according to the following Markov chain:
\eqenv{
(s_{t-k},a_{t-k})\myarrow{\pi_{t-k} \times {P}}(s_{t-k+1},a_{t-k+1})\myarrow{\pi_{t-k+1} \times {P}}
...(s_t,a_t)\myarrow{\pi_t \times {P}}(s\tit,a\tit).}

Using the technique in  \citep{zou2019finite}, we construct an auxiliary Markov chain as follows. Before time $t-k$, the states and actions are generated according to \Cref{alg:AC}; and after time $t-k$, all the subsequent state-action pairs, denoted by $(\widetilde s_l,\widetilde a_l)$, are generated according to a fixed policy $\pi_t$ and transition kernel $P$:
\eqenv{\label{eq:auxillary}
(s_{t-k},a_{t-k})\myarrow{\pi_{t} \times {P}}(\widetilde s_{t-k+1},\widetilde a_{t-k+1})\myarrow{\pi_{t} \times {P}}...(\widetilde s_t,\widetilde a_t)\myarrow{\pi_t \times {P}}(\widetilde s\tit,\widetilde a\tit).}

Denote by $\widetilde {\mathcal{F}}_{t}$ the filtration corresponding to the auxiliary Markov chain designed in \Cref{eq:auxillary}. 

Then follow steps similar to those in \citep[Appendix B]{zou2019finite} and \citep[Lemma 6]{li2021faster}, it can be shown that
     \eqenv{
          {\tvnorm{\Prob\brac{s_t,a_t|\mathcal{F}_{t-k}}-D_t}}
          &\leq m\rho^k+ \sum_{j=t-k}^t C_\pi\twonorm{\omega_t-\omega_{j}}.
     }
\end{proof}

\begin{proof}[Proof of Lemma \ref{lm:consterm}]
Define the sum of the feature along the trajectory as follows:
\begin{align}
     z_t= \sum_{j=t-k}^t \phi_j(s_j,a_j),\quad
\hat{z}_t=\sum_{j=t-k}^t \phi_t(s_j,a_j)
\text{ and } \widetilde{z_t}=\sum_{j=t-k}^t \phi_t(\widetilde s_j,\widetilde a_j).
\end{align}

    For every policy $\pi_t$, we construct another auxiliary Markov chain, denoted by $\{\brac{\bar s_j,\bar a_j} \}_{j=0}^\infty$, which is under the stationary distribution induced by policy $\pi_t$ and transition kernel $P$, i.e., 
\eqenv{
\brac{\bar s_0,\bar a_0}\sim D_t,}
and all the subsequent actions are generated by $\pi_t$. Define \eqenv{\bar z_t=\sum_{j=t-k}^t \phi_t(\bar s_j,\bar a_j).}

Denote by $\bar \delta_t(\bar s_t,\bar a_t;\theta_t,\omega_t)=R(\bar s_t,\bar a_t)-J(\omega_t)+\phi_t^\top(\bar s\tit,\bar a\tit)\theta_t-\phi_t^\top(\bar s_t,\bar a_t)\theta_t $.

\begin{lemma}\label{lm:constract}
It holds that
    \eqenv{
    \E\Fbrac{\bar  z_t\bar \delta_t(\bar s_t,\bar a_t;\theta,\omega_t)|\pi_t}
    =H_{\omega_t}\theta+b_{\omega_t}.
    }
\end{lemma}

From the definition in \Cref{eq:kstep}, $\thestar$ is the fixed point of the $k$-step TD operator $\mathcal{T}_{\pi_t}^{(k)}$.  Then, it follows that
 \eqenv{\label{eq:A_theta}
 H_{\omega_t}\thestar+b_{\omega_t}=
 \E_{ D_t}\Fbrac{ \phi^\top_t(s,a)\brac{\mathcal{T}_{\pi_t}^{(k)}\brac{\phi_t^\top(s,a) \thestar}-\phi_t^\top(s,a) \thestar}}
 =\mathbf 0.
 }
 
Together with Lemma \ref{lm:constract}, we have that
 \eqenv{\label{eq:thestar0}
 \E\Fbrac{ \bar z_t \bar \delta_t(\bar s_t,\bar a_t;\thestar,\omega_t)}=0.
 }
Thus we have that
    \eqenv{
     \E\fvbrac{\theta_t-\thestar, \bar z_t \bar \delta_t(\bar s_t,\bar a_t;\theta_t, \omega_t) }
     &=\E\fvbrac{\theta_t-\thestar, \bar z_t \bar \delta_t(\bar s_t,\bar a_t;\theta_t, \omega_t)-\bar z_t \bar \delta_t(\bar s_t,\bar a_t;\thestar, \omega_t)}
     \\& =
     \E\fvbrac{\theta_t-\thestar,H_{\omega_t} \brac{\theta_t- \thestar}}
     \\& \leq  \lambda_{\max}\brac{ \frac{H_{\omega_t}+H_{\omega_t}^\top}{2}} \E\ftwonormsq{\theta_t-\thestar}
     \\& \myineq{\leq}{a}  -\bar \lambda_{\min}  \E\ftwonormsq{\theta_t-\thestar},\label{eq:164}
     }
     where $(a)$ follows from \Cref{lm:eig}. 

Then, recall $\hat{z}_t=\sum_{j=t-k}^t \phi_t(s_j,a_j)$. Denote by $\hat \delta_t= R(s_t,a_t)-J(\omega_t)+ \phi^\top_t(s\tit,a\tit)\theta_t -\phi^\top_t(s_t,a_t)\theta_t$, we have that
\eqenv{
\E\fvbrac{\theta_t-\thestar, \delta_t z_t} 
& = 
\E\fvbrac{ \theta_t-\thestar, \bar z_t\bar \delta_t(\bar s_t,\bar a_t;\theta_t, \omega_t) }
+\E\fvbrac{\theta_t-\thestar, z_t \hat \delta_t- \bar z_t\bar \delta_t(\bar s_t,\bar a_t;\theta_t, \omega_t) }
\\& \qquad\qquad+
\E\fvbrac{\theta_t-\thestar,z_t  \delta_t- \hat z_t\hat \delta_t }
\\& \myineq{\leq}{a}
-\bar \lambda_{\min}\E \ftwonormsq{\theta_t-\thestar }
+ \E\fvbrac{\theta_t-\thestar, z_t\hat \delta_t- \bar z_t\bar \delta_t(\bar s_t,\bar a_t;\theta_t) }
\\& \qquad\qquad +\E\fvbrac{\theta_t-\thestar,z_t  \brac{J(\omega_t)-\eta_t } }
\\& \leq -\bar \lambda_{\min}\E \ftwonormsq{\theta_t-\thestar }
+ \E\fvbrac{\theta_t-\thestar, z_t\hat \delta_t- \bar z_t\bar \delta_t(\bar s_t,\bar a_t;\theta_t) }
\\& \qquad\qquad +\frac{\bar \lambda_{\min}}{2}\E \ftwonormsq{\theta_t-\thestar}+ \frac{(k+1)^2 C^2_\phi}{2\bar \lambda_{\min}}\E\flnormsq{J(\omega_t)-\eta_t},
}
where  $(a)$ follows from \Cref{eq:164}. 

Consider the term 
$\E\fvbrac{\theta_t-\thestar,z_t \hat \delta_t- \bar z_t\bar \delta_t(\bar s_t,\bar a_t;\theta_t, \omega_t)} $, and we have that
 \eqenv{
 &\E\fvbrac{\theta_t-\thestar,z_t\hat \delta_t- \bar z_t\bar \delta_t(\bar s_t,\bar a_t;\theta_t, \omega_t) }
\\& = \E\fvbrac{\theta_t-\theta_{t-2k}-\thestar+\theta^*_{t-2k},z_t\hat \delta_t- \bar z_t\bar \delta_t(\bar s_t,\bar a_t;\theta_t, \omega_t) }
  \\& \qquad\qquad+\E\fvbrac{\theta_{t-2k}-\theta^*_{t-2k},z_t\hat \delta_t- \bar z_t\bar \delta_t(\bar s_t,\bar a_t;\theta_t, \omega_t)}
 \\ & \leq \underbrace{\E\Fbrac{\twonorm{ \theta_t-\theta_{t-2k}}+\twonorm{\thestar-\theta^*_{t-2k} } \brac{\twonorm{z_t}\twonorm{\hat \delta_t}+\twonorm{\bar z_t}\twonorm{\bar \delta_t(\bar s_t,\bar a_t;\theta_t, \omega_t)} }}}_{(i)}
 \\&\quad+ \underbrace{\E\fvbrac{\theta_{t-2k}-\theta^*_{t-2k},z_t\hat \delta_t- \hat z_t\hat \delta_t }}_{(ii)}
 \\& \quad +\underbrace{\E\fvbrac{\theta_{t-2k}-\theta^*_{t-2k},\hat z_t\hat \delta_t- \bar z_t \bar \delta_t(\bar s_t,\bar a_t;\theta_t, \omega_t) }}_{(iii)}.
 \label{eq:ac:tke1}
 }

 In AC algorithm, recall $U_\delta =R_{\max}+2 C_\phi B$. Then, consider  $ \twonorm{\theta_t-\theta_{t-2k}}$ and $ \twonorm{\thestar-\theta^*_{t-2k}}$,  we have that
  \eqenv{\label{eq:kgaptheta}
    \twonorm{\theta_t-\theta_{t-2k}}\leq
    \twonorm{\sum_{j=t-2k}^{t-1} \alpha_j \delta_j z_j}
    \leq \sum_{j=t-2k}^{t-1} \alpha_j \lbrac{\delta_j} \twonorm{z_j}
    \myineq{\leq}{a} (k+1) C_\phi U_\delta \sum_{j=t-2k}^{t-1} \alpha_j ,
  } where $(a)$ follows from the fact that $\twonorm{z_t}\leq (k+1) C_\phi$.

  Then, it can be shown that
  \eqenv{\label{eq:kgapthestar}
  \twonorm{\thestar-\theta^*_{t-2k}}
  & = \twonorm{\thebar-\bar\theta^*_{t-2k}+ \thestar-\thebar+\bar\theta^*_{t-2k}-\theta^*_{t-2k}  }
  \\& \leq  \twonorm{\thebar-\bar\theta^*_{t-2k} }+\twonorm{ \thestar-\thebar}+\twonorm{\bar\theta^*_{t-2k}-\theta^*_{t-2k} }
  \\& \leq C_\Theta\twonorm{{\omega_t-\omega_{t-2k}} }+  \frac{C_{\text{gap}}m \rho^k}{\lambda_{\min}}
  +\frac{ C_{\text{gap}}m \rho^k}{\lambda_{\min}}
  \\& \leq C_\Theta \twonorm{ \sum_{j=t-2k}^{t-1} \beta_j \phi_j^\top(s_j,a_j)\theta_j \phi_j(s_j,a_j) }+ \frac{2C_{\text{gap}}m \rho^k}{\lambda_{\min}}
  \\& \leq C_\Theta C^2_\phi B \sum_{j=t-2k}^{t-1} \beta_j+ 
   \frac{2C_{\text{gap}}m \rho^k}{\lambda_{\min}}.
  }
 Thus, from \Cref{eq:kgapthestar}, the term $(i)$ in \Cref{eq:ac:tke1} can be bounded as follows: 
 \eqenv{
 (i)\leq 2(k+1) C_\phi U_\delta\brac{(k+1) C_\phi U_\delta \sum_{j=t-2k}^{t-1} \alpha_j+ C_\Theta C^2_\phi B \sum_{j=t-2k}^{t-1} \beta_j +2C_{\text{gap}}\frac{m \rho^k}{\lambda_{\min}}}.\label{eq:ac:tka1}
 }
Then, for term $(ii)$ in \Cref{eq:ac:tke1}, it can be bounded as follows: 
\eqenv{
(ii)&\leq \E\Fbrac{ \twonorm{\theta_{t-2k}-\theta^*_{t-2k}}
\twonorm{z_t-\hat z_t}\lbrac{\delta_t(\theta_t)}}
\leq 2B U_\delta \twonorm{\sum_{j=t-k}^t \phi_j(s_j,a_j)-\phi_t(s_j,a_j) }
\\& \leq 2BU_\delta \sum_{j=t-k}^t C_\pi \E\ftwonorm{\omega_t-\omega_j}
 \leq 2B U_\delta C_\pi \sum_{j=t-k}^t \sum_{i=j}^t \E\ftwonorm{\beta_i \phi_i^\top(s_i,a_i)\theta_i \phi_i(s_i,a_i)}
 \\& \leq 2B^2 C^2_\phi U_\delta C_\pi \sum_{j=t-k}^t \sum_{i=j}^t \beta_i .\label{eq:ac:tka2}
}

Next, for term $(iii)$ in \Cref{eq:ac:tke1}, we can show that
\eqenv{
(iii) &=\E \fvbrac{\theta_{t-2k}-\theta^*_{t-2k}, \hat z_t \hat \delta_t-\bar z_t\bar\delta_t(\bar s_t,\bar a_t;\theta_t) }
\\& =\E \Fbrac{\E \Fbrac{\vbrac{\theta_{t-2k}-\theta^*_{t-2k},\hat z_t \hat \delta_t-\bar z_t\bar\delta_t(\bar s_t,\bar a_t;\theta_t)}|\mathcal{F}_{t-2k}} }
\\& \leq 4B   C_\phi U_\delta \sum_{j=t-k}^t\E\Fbrac{\tvnorm{\Prob\brac{ s_j,a_j|\mathcal{F}_{t-2k}}-D_t}}
\\ & \leq 4  B C_\phi U_\delta \sum_{j=t-k}^t\E\Fbrac{\tvnorm{\Prob\brac{ s_j,a_j|\mathcal{F}_{t-2k}}-D_j}+ \tvnorm{D_j-D_t}}
\\& \leq 4  B C_\phi U_\delta \sum_{j=t-k}^t\brac{ 
 C_\pi \sum_{i=j-k}^{j-1} \E\ftwonorm{\omega_i-\omega_j}+ m\rho^k+L_\pi\E\ftwonorm{\omega_t-\omega_j} }
 \\& \leq 4  B C_\phi U_\delta \sum_{j=t-k}^t\brac{ 
 C_\pi \sum_{i=j-k}^{j-1} \sum_{\iota=i}^{j-1} \beta_\iota C^2_\phi B + m\rho^k+L_\pi\sum_{i=j}^{t-1} \beta_i C^2_\phi B }.\label{eq:ac:tka3}
}
Thus, combining \Cref{eq:ac:tka1}, \Cref{eq:ac:tka2} and \Cref{eq:ac:tka3}, we can bound  term  as follows: 
\eqenv{
&\E\fvbrac{\theta_t-\thestar, \delta_t z_t}
\leq - \frac{\bar \lambda_{\min}}{2} \E \ftwonormsq{\theta_t-\thestar}
+\frac{(k+1)^2 C^2_\phi}{2\bar \lambda_{\min}}\E\flnormsq{J(\omega_t)-\eta_t}
\\&+2B^2 C^2_\phi U_\delta C_\pi \sum_{j=t-k}^t \sum_{i=j}^t \beta_i
+4  B C_\phi U_\delta \sum_{j=t-k}^t\brac{ 
 B C^2_\phi  C_\pi \sum_{i=j-k}^{j-1} \sum_{\iota=i}^{j-1} \beta_\iota  + m\rho^k+ BC^2_\phi L_\pi\sum_{i=j}^{t-1} \beta_i  }
 \\& +  2(k+1) C_\phi U_\delta\brac{(k+1) C_\phi U_\delta \sum_{j=t-2k}^{t-1} \alpha_j+ C_\Theta C^2_\phi B \sum_{j=t-2k}^{t-1} \beta_j+\frac{2C_{\text{gap}}m \rho^k}{\lambda_{\min}}}.
}

In NAC algorithm, 
terms $ \twonorm{\theta_t-\theta_{t-2k}}$ and $ \twonorm{\thestar-\theta^*_{t-2k}}$ can be bounded as follows: 
  \eqenv{\label{eq:nac:kgaptheta}
    \twonorm{\theta_t-\theta_{t-2k}}\leq
    \twonorm{\sum_{j=t-2k}^{t-1} \alpha_j \delta_j z_j}
    \leq (k+1) C_\phi U_\delta \sum_{j=t-2k}^{t-1} \alpha_j ,
  } and 
  \eqenv{\label{eq:nac:kgapthestar}
  \twonorm{\thestar-\theta^*_{t-2k}}
  & = \twonorm{\thebar-\bar\theta^*_{t-2k}+ \thestar-\thebar+\bar\theta^*_{t-2k}-\theta^*_{t-2k}  }
  \\& \leq  \twonorm{\thebar-\bar\theta^*_{t-2k} }+\twonorm{ \thestar-\thebar}+\twonorm{\bar\theta^*_{t-2k}-\theta^*_{t-2k} }
  \\& \leq C_\Theta\twonorm{\brac{\omega_t-\omega_{t-2k}} }+  \frac{C_{\text{gap}}m \rho^k}{\lambda_{\min}}
  + \frac{C_{\text{gap}}m \rho^k}{\lambda_{\min}}
  \\& \leq C_\Theta \twonorm{ \sum_{j=t-2k}^{t-1} \beta_j \theta_j  }+\frac{ 2C_{\text{gap}}m \rho^k}{\lambda_{\min}}
  \\& \leq C_\Theta  B \sum_{j=t-2k}^{t-1} \beta_j+ 
   \frac{2C_{\text{gap}}m \rho^k}{\lambda_{\min}}.
  }
 Thus, using \Cref{eq:nac:kgapthestar} and \Cref{eq:nac:kgaptheta}, term $(i)$ in \Cref{eq:ac:tke1} can be bounded as
 \eqenv{
 (i)\leq 2(k+1) C_\phi U_\delta\brac{(k+1) C_\phi U_\delta \sum_{j=t-2k}^{t-1} \alpha_j+ C_\Theta  B \sum_{j=t-2k}^{t-1} \beta_j +2C_{\text{gap}}\frac{m \rho^k}{\lambda_{\min}}}.
 }
Next, we bound the term $(ii)$ in \Cref{eq:ac:tke1} as follows: 
\eqenv{
(ii)&\leq \E\Fbrac{\twonorm{\theta_{t-2k}-\theta^*_{t-2k}}
\twonorm{z_t-\hat z_t}\lbrac{\hat \delta_t}}
\leq 2B U_\delta \E\ftwonorm{\sum_{j=t-k}^t \phi_j(s_j,a_j)-\phi_t(s_j,a_j) }
\\& \leq 2BU_\delta \sum_{j=t-k}^t C_\pi \E\ftwonorm{\omega_t-\omega_j}
 \leq 2B U_\delta C_\pi \sum_{j=t-k}^t \sum_{i=j}^{t-1} \twonorm{\beta_i \theta_i }
  \leq 2B^2  U_\delta C_\pi \sum_{j=t-k}^t \sum_{i=j}^{t-1} \beta_i .
}

Term $(iii)$ in \Cref{eq:ac:tke1} can be bounded as
\eqenv{
(iii) &=\E \fvbrac{\theta_{t-2k}-\theta^*_{t-2k}, \hat z_t \delta_t(\theta_t)-z'_t\delta'_t(\theta_t) }
\\& =\E \Fbrac{\E \Fbrac{\vbrac{\theta_{t-2k}-\theta^*_{t-2k}, \hat z_t \delta_t(\theta_t)-z'_t\delta'_t(\theta_t) }|\mathcal{F}_{t-2k}} }
\\& \leq 2B 2  C_\phi U_\delta \sum_{j=t-k}^t\E\Fbrac{\tvnorm{\Prob\brac{ s_j,a_j|\mathcal{F}_{t-2k}}-D_t}}
\\ & \leq 4  B C_\phi U_\delta \sum_{j=t-k}^t\E\Fbrac{\tvnorm{\Prob\brac{ s_j,a_j|\mathcal{F}_{t-2k}}-D_j}+ \tvnorm{D_j-D_t}}
\\& \leq 4  B C_\phi U_\delta \sum_{j=t-k}^t\brac{ 
 C_\pi \sum_{i=j-k}^{j-1} \E\ftwonorm{\omega_i-\omega_j}+ m\rho^k+L_\pi\E\ftwonorm{\omega_t-\omega_j} }
 \\& \leq 4  B^2 C_\phi U_\delta \sum_{j=t-k}^t\brac{ 
 C_\pi \sum_{i=j-k}^{j-1} \sum_{\iota=i}^{j-1} \beta_\iota   + m\rho^k+L_\pi\sum_{i=j}^{t-1} \beta_i  }.
}

Combining the above bounds on terms $(i),(ii),(iii)$,  term can be bounded as
\eqenv{
\E&\fvbrac{\theta_t-\thestar, \delta_t z_t}
\\ \leq& - \frac{\bar \lambda_{\min}}{2} \E \ftwonormsq{\theta_t-\thestar}
+\frac{(k+1)^2 C^2_\phi}{2\bar \lambda_{\min}}\E\flnormsq{J(\omega_t)-\eta_t}
+2B^2 U_\delta C_\pi \sum_{j=t-k}^t \sum_{i=j}^{t-1} \beta_i
\\& + 2(k+1) C_\phi U_\delta\brac{(k+1) C_\phi U_\delta \sum_{j=t-2k}^{t-1} \alpha_j+ C_\Theta  B \sum_{j=t-2k}^{t-1} \beta_j+\frac{2C_{\text{gap}}m \rho^k}{\lambda_{\min}}}
\\&
+4  B C_\phi U_\delta \sum_{j=t-k}^t\brac{ 
 B C_\pi \sum_{i=j-k}^{j-1} \sum_{\iota=i}^{j-1} \beta_\iota   + m\rho^k+B L_\pi\sum_{i=j}^{t-1} \beta_i  }.
}
This completes the proof.
\end{proof}

\begin{proof}[Proof of Lemma \ref{lm:acinnersmooth}]

From $\nabla^2 J(\omega)= \sum_{s,a}\nabla^2 D_{\pi_\omega}(s,a) R(s,a) $, we can get that for any $\omega,\omega'\in \mathbb R^d$
\eqenv{
\twonorm{ \nabla^2 J(\omega)-\nabla^2 J(\omega')}& = \twonorm{ \sum_{s,a}\brac{\nabla^2 D_{\pi_\omega}(s,a) R(s,a)-\nabla^2 D_{\pi_{\omega'}}(s,a) R(s,a)}}
. 
}

Denote by $\boldsymbol{\sigma}(A)$ the spectral radius of matrix $A\in \mathbb R^{n\times n}$.
Recall the fact that 
$
\boldsymbol{\sigma}(A)\leq \mynorm{A}_\infty= \max_i\varbrac{\sum_j\lbrac{a_{ij}}}.
$
 If $A$ is symmetric matrix, $\twonorm{A}=\boldsymbol{\sigma}(A)$, and thus
$
    \twonorm{A}\leq \max_i\varbrac{\sum_j\lbrac{a_{ij}}}.
$

It is clear that $\nabla^2 D_{\pi_\omega}(s,a)$ is symmetric, and therefore $\sum_{s,a}(\nabla^2 D_{\pi_\omega}(s,a)-\nabla^2 D_{\pi_{\omega'}}(s,a))R(s,a)$ is also symmetric. It then follows that
 \eqenv{
   &\twonorm{\sum_{s,a}\brac{\nabla^2 D_{\pi_\omega}(s,a)-\nabla^2 D_{\pi_{\omega'}}(s,a)}R(s,a)}
    \\&\leq \max_i\varbrac{\sum_j \lbrac{\sum_{s,a}\brac{ \partial_{\omega_i}\partial_{\omega_j}D_{\pi_{\omega}}(s,a) -\partial_{\omega_i}\partial_{\omega_j}D_{\pi_{\omega'}}(s,a) }R(s,a)}}
   \\& \myineq{=}{a}\max_i\varbrac{\sum_j \lbrac{\brac{\nabla_\omega \sum_{s,a} \partial_{\omega_i}\partial_{\omega_j}D_{\pi_{\hat\omega}}(s,a)R(s,a)}^\top(\omega-\omega')}}
    \\& \leq\max_i\varbrac{\sum_j \twonorm{\nabla_\omega \sum_{s,a} \partial_{\omega_i}\partial_{\omega_j}D_{\pi_{\hat\omega}}(s,a)R(s,a)}\twonorm{\omega-\omega'}}
   \\& \myineq{\leq}{b} \max_i\varbrac{\sum_j \sum_l \lbrac{ \sum_{s,a} \partial_{\omega_i}\partial_{\omega_j}\partial_{\omega_l}D_{\pi_{\hat\omega}}(s,a)R(s,a)}\twonorm{\omega-\omega'}}
   \\& \leq \max_{i,j,l} \varbrac{d^2\lbrac{ \sum_{s,a}\partial_{\omega_i}\partial_{\omega_j}\partial_{\omega_l}D_{\pi_{\hat\omega}}(s,a)R(s,a) }}\twonorm{\omega-\omega'},
 } where $(a)$ follows from the fact that $D_{\pi_\omega}$ is $n$ times differentiable as long as the Theorem 4 in \cite{heidergott2003taylor} and the Lagrange's Mean Value Theorem for some $\hat \omega =\lambda_{ij} \omega+(1-\lambda_{ij})\omega'$ with $\lambda_{ij}\in [0,1]$, and $(b)$ follows from that for a vector $a$, $\twonorm{a}\leq \mynorm{a}_1$.

Define a function $v: \mathcal{S}\times\mathcal{A}\to \mathbb R$, denote by $\mynorm{f}_v=\sup_{s,a}\frac{\lbrac{f(s,a)}}{\lbrac{v(s,a)}} $ the finite $v$-norm of function $f$. Set $v(s,a)=e^{R(s,a)}$ and we can get that
 \eqenv{
    \sup_{s,a}\frac{\lbrac{R(s,a)} }{\lbrac{v(s,a) }}\leq 1.
 } Moreover,  $v(s,a)=e^{R(s,a)}\leq e^{R_{\max}}$ and $v(s,a)\geq 1$. If $\mynorm{f}_v\leq 1 $, it implies that 
 \eqenv{\label{eq:boundf}
 \sup_{s,a} |f(s,a)|\leq  e^{R_{\max}}.
 }

 For a (signed) measure $\mu$, the associated norm is 
\eqenv{
\mynorm{\mu}_v=\sup_{\mynorm{f}_v\leq 1} \lbrac{\sum_{s,a} \mu(s,a) f(s,a) }.
} 

For a kernel $P(s',a'|s,a)$, its associated norm is 
\eqenv{
    \mynorm{P}_v=\sup_{s,a}\sup_{\mynorm{f}_v\leq 1}
    \frac{\lbrac{\sum_{s',a'}f(s',a') P(s',a'|s,a) } }{\lbrac{v(s,a)}}.
}

From the fact that  $\sup_{s,a}\frac{\lbrac{R(s,a)} }{\lbrac{v(s,a) }}\leq 1$, we can get that $||R||_v\leq 1$. This further implies that 
\eqenv{
    \lbrac{\sum_{s,a} \partial_{\omega_i}\partial_{\omega_j}\partial_{\omega_l}D_{\pi_{\hat\omega}}(s,a)R(s,a)}&\leq \sup_{\mynorm{f}_v\leq 1}\lbrac{\sum_{s,a} \partial_{\omega_i}\partial_{\omega_j}\partial_{\omega_l}D_{\pi_{\hat\omega}}(s,a)f(s,a)}\\&=\mynorm{\partial_{\omega_i}\partial_{\omega_j}\partial_{\omega_l}D_{\pi_{\hat\omega}} }_v. 
}


Furthermore, we define $\mathcal{K}^{(1)}_{\omega_i}(s,a)$, $\mathcal{K}^{(2)}_{\omega_{i,j}}(s,a)$ and $\mathcal{K}^{(3)}_{\omega_{i,j,l} }(s,a)$ as follows:
\eqenv{
&\mathcal{K}^{(1)}_{\omega_i}(s,a|s',a')= \sum_{\iota =0}^\infty \partial_{\omega_i}\pi_{\omega}(a|s) \brac{\Prob\brac{s_\iota =s,a_\iota =a|s_0=s',a_0=a',\pi_\omega}-D_{\pi_\omega}(s,a)};
\\& \mathcal{K}^{(2)}_{\omega_{i,j}}(s,a|s',a')= \sum_{\iota =0}^\infty \partial_{\omega_i}\partial_{\omega_j}\pi_{\omega}(a|s) \brac{\Prob\brac{s_\iota =s,a_\iota =a|s_0=s',a_0=a',\pi_\omega}-D_{\pi_\omega}(s,a)};
\\& \mathcal{K}^{(3)}_{\omega_{i,j,l}}(s,a|s',a')= \sum_{\iota =0}^\infty \partial_{\omega_i}\partial_{\omega_j}\partial_{\omega_l}\pi_{\omega}(a|s) \brac{\Prob\brac{s_\iota =s,a_\iota =a|s_0=s',a_0=a',\pi_\omega}-D_{\pi_\omega}(s,a)}.
}

Then, define the kernel $\Gamma_\omega(s',a'|s,a)$, s.t.,
$
  \Gamma_\omega(s',a'|s,a)=D_{\pi_\omega}(s',a')
$ for any $s\in \mathcal{S}$,  $a\in \mathcal{A}$. It then follows that 
\eqenv{
    \mynorm{\Gamma_\omega}_v& = \sup_{s,a}\sup_{\mynorm{f}_v\leq 1}
    \frac{\lbrac{\sum_{s',a'}f(s',a') \Gamma_\omega(s',a'|s,a) } }{\lbrac{v(s,a)}}
    \\& =\sup_{s,a}\sup_{\mynorm{f}_v\leq 1}
    \frac{\lbrac{\sum_{s',a'}f(s',a') D_{\pi_\omega}(s',a' ) } }{\lbrac{v(s,a)}}
      =\sup_{s,a}\frac{\mynorm{ D_{\pi_\omega} }_v }{\lbrac{v(s,a)}}. 
}

By Theorem 3 and proof in \citep{heidergott2003taylor}, 
we can get
\eqenv{
    &\partial_{\omega_j} \mathcal{K}^{(1)}_{\omega_i}
    = \mathcal{K}^{(1)}_{\omega_i}\mathcal{K}^{(1)}_{\omega_j}
    + \mathcal{K}^{(2)}_{\omega_{ij}};
    & \partial_{\omega_l} \mathcal{K}^{(2)}_{\omega_{ij}}
    = \mathcal{K}^{(2)}_{\omega_{ij}} \mathcal{K}^{(1)}_{\omega_{l}}+ \mathcal{K}^{(3)}_{\omega_{ijl}}.
}

Combining with Theorem 4 and Section 4 in \citep{heidergott2003taylor}, 
we can have 
\eqenv{
 \partial_{\omega_i}\Gamma_\omega &= \Gamma_\omega \mathcal{K}^{(1)}_{\omega_{i}};
 \\ \partial_{\omega_j}\partial_{\omega_i} \Gamma_\omega &= \brac{\partial_{\omega_j} \Gamma_\omega }\mathcal{K}^{(1)}_{\omega_{i}}+ \Gamma_\omega \partial_{\omega_i}\mathcal{K}^{(1)}_{\omega_{j}} =\Gamma_\omega \mathcal{K}^{(1)}_{\omega_{i}}\mathcal{K}^{(1)}_{\omega_{j}}+ \Gamma_\omega\brac{\mathcal{K}^{(1)}_{\omega_i}\mathcal{K}^{(1)}_{\omega_j} + \mathcal{K}^{(2)}_{\omega_{ij}}}
 \\& = 2\Gamma_\omega \mathcal{K}^{(1)}_{\omega_{i}}\mathcal{K}^{(1)}_{\omega_{j}}+ \Gamma_\omega  \mathcal{K}^{(2)}_{\omega_{ij}};
 \\ \partial_{\omega_l}\partial_{\omega_j}\partial_{\omega_i}\Gamma_\omega &= \partial_{\omega_l}\brac{ 2\Gamma_\omega \mathcal{K}^{(1)}_{\omega_{i}}\mathcal{K}^{(1)}_{\omega_{j}}+ \Gamma_\omega  \mathcal{K}^{(2)}_{\omega_{ij}} }
 \\& =  2\brac{\partial_{\omega_l}\Gamma_\omega }\mathcal{K}^{(1)}_{\omega_{i}}\mathcal{K}^{(1)}_{\omega_{j}}
 + 2\Gamma_\omega \brac{\partial_{\omega_l}\mathcal{K}^{(1)}_{\omega_{i}}}\mathcal{K}^{(1)}_{\omega_{j}}
 + 2\Gamma_\omega \mathcal{K}^{(1)}_{\omega_{i}}\brac{\partial_{\omega_l}\mathcal{K}^{(1)}_{\omega_{j}}}
 \\&\qquad + \brac{\partial_{\omega_l}\Gamma_\omega } \mathcal{K}^{(2)}_{\omega_{ij}}
 +\Gamma_\omega  \brac{\partial_{\omega_l}\mathcal{K}^{(2)}_{\omega_{ij}}}
 \\& = 2\Gamma_\omega \mathcal{K}^{(1)}_{\omega_{l}}\mathcal{K}^{(1)}_{\omega_{i}}\mathcal{K}^{(1)}_{\omega_{j}}
 + 2\Gamma_\omega \brac{\mathcal{K}^{(1)}_{\omega_i}\mathcal{K}^{(1)}_{\omega_l}
    + \mathcal{K}^{(2)}_{\omega_{il}} }\mathcal{K}^{(1)}_{\omega_{j}}
    + 2\Gamma_\omega \mathcal{K}^{(1)}_{\omega_{i}}\brac{\mathcal{K}^{(1)}_{\omega_j}\mathcal{K}^{(1)}_{\omega_l}
    + \mathcal{K}^{(2)}_{\omega_{jl} }}
    \\&\qquad +\Gamma_\omega \mathcal{K}^{(1)}_{\omega_l}
    \mathcal{K}^{(2)}_{\omega_{ij} }+ \Gamma_\omega \brac{\mathcal{K}^{(2)}_{\omega_{ij}} \mathcal{K}^{(1)}_{\omega_{l}}+ \mathcal{K}^{(3)}_{\omega_{ijl}} }
    \\&=6 \Gamma_\omega \mathcal{K}^{(1)}_{\omega_{i}}\mathcal{K}^{(1)}_{\omega_{j}}\mathcal{K}^{(1)}_{\omega_{l}}
    + 2\Gamma_\omega \mathcal{K}^{(1)}_{\omega_{i}}\mathcal{K}^{(2)}_{\omega_{jl}}+ 2\Gamma_\omega \mathcal{K}^{(1)}_{\omega_{j}}\mathcal{K}^{(2)}_{\omega_{il}}+ 2\Gamma_\omega \mathcal{K}^{(1)}_{\omega_{l}}\mathcal{K}^{(2)}_{\omega_{ij}}+ \Gamma_\omega  \mathcal{K}^{(3)}_{\omega_{ijl}}.
}

Then, according to the discussion in Section 4 in \citep{heidergott2003taylor}, it can be shown that  
\eqenv{
\sup_{s,a} \frac{\mynorm{D_{\pi_{\hat \omega}}}_v}{\lbrac{v(s,a) }}&=\mynorm{\partial_{\omega_i} \partial_{\omega_j} \partial_{\omega_l}\Gamma_{\pi_{\hat\omega}}}_v  
\\&\leq \mynorm{\Gamma_{\pi_{\hat\omega}}}_v\mynorm{ \mathcal{K}^{(3)}_{{\hat \omega}_{ijl}}  }_v   + 2 \mynorm{\Gamma_{\pi_{\hat\omega}} }_v \mynorm{\mathcal{K}^{(2)}_{{\hat \omega}_{ij}}  }_v \mynorm{\mathcal{K}^{(1)}_{{\hat \omega}_{l}}  }_v   
    \\& \quad   + 2 \mynorm{\Gamma_{\pi_{\hat\omega}} }_v\mynorm{\mathcal{K}^{(2)}_{{\hat \omega}_{il}}  }_v \mynorm{\mathcal{K}^{(1)}_{{\hat \omega}_{j}}  }_v   
    + 2 \mynorm{\Gamma_{\pi_{\hat\omega}} }_v \mynorm{\mathcal{K}^{(2)}_{{\hat \omega}_{lj}}  }_v \mynorm{\mathcal{K}^{(1)}_{{\hat \omega}_{i}}  }_v   
    \\& \quad + 6  \mynorm{\Gamma_{\pi_{\hat\omega}} }_v \mynorm{\mathcal{K}^{(1)}_{{\hat \omega}_{i}}  }_v 
    \mynorm{\mathcal{K}^{(1)}_{{\hat \omega}_{j}}  }_v \mynorm{\mathcal{K}^{(1)}_{{\hat \omega}_{l}}  }_v
    \\&= \sup_{s,a} \frac{1}{\lbrac{v(s,a) }}\big(
    \mynorm{D_{\pi_{\hat\omega}}}_v\mynorm{ \mathcal{K}^{(3)}_{{\hat \omega}_{ijl}}  }_v   + 2 \mynorm{D_{\pi_{\hat\omega}} }_v \mynorm{\mathcal{K}^{(2)}_{{\hat \omega}_{ij}}  }_v \mynorm{\mathcal{K}^{(1)}_{{\hat \omega}_{l}}  }_v   
    \\& \quad   + 2 \mynorm{D_{\pi_{\hat\omega}} }_v\mynorm{\mathcal{K}^{(2)}_{{\hat \omega}_{il}}  }_v \mynorm{\mathcal{K}^{(1)}_{{\hat \omega}_{j}}  }_v   
    + 2 \mynorm{D_{\pi_{\hat\omega}} }_v \mynorm{\mathcal{K}^{(2)}_{{\hat \omega}_{lj}}  }_v \mynorm{\mathcal{K}^{(1)}_{{\hat \omega}_{i}}  }_v   
    \\& \quad + 6  \mynorm{D_{\pi_{\hat\omega}} }_v \mynorm{\mathcal{K}^{(1)}_{{\hat \omega}_{i}}  }_v 
    \mynorm{\mathcal{K}^{(1)}_{{\hat \omega}_{j}}  }_v \mynorm{\mathcal{K}^{(1)}_{{\hat \omega}_{l}}  }_v
    \big).
}
Note that $R_{\max}\geq R(s,a)\geq 0$ and $\sup_{s,a} \frac{1}{\lbrac{v(s,a)}}\geq e^{-R_{\max}}$. Then we have that 
\eqenv{
      \mynorm{\partial_{\omega_i} \partial_{\omega_j} \partial_{\omega_l}D_{\pi_{\hat\omega}}}_v   &\leq \mynorm{D_{\pi_{\hat\omega}}}_v\mynorm{ \mathcal{K}^{(3)}_{{\hat \omega}_{ijl}}  }_v   + 2 \mynorm{D_{\pi_{\hat\omega}} }_v \mynorm{\mathcal{K}^{(2)}_{{\hat \omega}_{ij}}  }_v \mynorm{\mathcal{K}^{(1)}_{{\hat \omega}_{l}}  }_v   
    \\& \quad   + 2 \mynorm{D_{\pi_{\hat\omega}} }_v\mynorm{\mathcal{K}^{(2)}_{{\hat \omega}_{il}}  }_v \mynorm{\mathcal{K}^{(1)}_{{\hat \omega}_{j}}  }_v   
    + 2 \mynorm{D_{\pi_{\hat\omega}} }_v \mynorm{\mathcal{K}^{(2)}_{{\hat \omega}_{lj}}  }_v \mynorm{\mathcal{K}^{(1)}_{{\hat \omega}_{i}}  }_v   
    \\& \quad + 6  \mynorm{D_{\pi_{\hat\omega}} }_v \mynorm{\mathcal{K}^{(1)}_{{\hat \omega}_{i}}  }_v 
    \mynorm{\mathcal{K}^{(1)}_{{\hat \omega}_{j}}  }_v \mynorm{\mathcal{K}^{(1)}_{(\hat \omega)_{l}}  }_v    .\label{eq:eq188}
}

Next, we bound the term $\mynorm{D_{\pi_{\hat\omega}} }_v  $ as follows: 
\eqenv{\label{eq:eq179a}
  \mynorm{D_{\pi_{\hat\omega}} }_v &=\sup_{\mynorm{f}_v\leq 1} 
  \lbrac{\sum_{s,a}f(s,a)D_{\pi_{\hat\omega}}(s,a)  }
   \leq \sup_{\mynorm{f}_v\leq 1} 
  \lbrac{\sum_{s,a}|f(s,a)| \lbrac{D_{\pi_{\hat\omega}}(s,a)}}
  \\& \myineq{\leq}{a}  
  \lbrac{\sum_{s,a}e^{R_{\max}}D_{\pi_{\hat\omega}}(s,a)  }
  \leq e^{R_{\max}},
} where $(a)$ follows from the \Cref{eq:boundf}.

Then, we bound the terms in \Cref{eq:eq188} as follows:
\eqenv{\label{eq:186eq}
&\mynorm{\mathcal{K}^{(1)}_{(\hat \omega)_{i}}  }_v 
\\&=\sup_{s',a'}\sup_{\mynorm{f}_v\leq 1} \frac{\lnorm{\sum_{\iota =0}^\infty \sum_{s,a} \partial_{\omega_i} \pi_{\omega}(a|s) \brac{\Prob\brac{s_\iota =s,a_\iota =a|s_0=s',a_0=a',\pi_\omega}-D_{\pi_\omega}(s,a)}f(s,a) }}{\lbrac{v(s',a')}}
    \\& \myineq{\leq}{a} \sup_{s',a'}\sup_{\mynorm{f}_v\leq 1}\lbrac{\sum_{\iota =0}^\infty \sum_{s,a} \partial_{\omega_i} \pi_{\omega}(a|s) \brac{\Prob\brac{s_\iota =s,a_\iota =a|s_0=s',a_0=a',\pi_\omega}-D_{\pi_\omega}(s,a)}f(s,a) }
    \\& \leq \sup_{s',a'}\sup_{\mynorm{f}_v\leq 1}\lbrac{\sum_{\iota =0}^\infty \sum_{s,a} \lbrac{\partial_{\omega_i} \pi_{\omega}(a|s) }\lbrac{\Prob\brac{s_\iota =s,a_\iota =a|s_0=s',a_0=a',\pi_\omega}-D_{\pi_\omega}(s,a)}|f(s,a) |}
    \\& \myineq{\leq}{b}\sup_{s',a'} \sum_{\iota =0}^\infty \max_{s,a}\varbrac{\lnorm{\partial_{\omega_i} \pi_{\omega}(a|s)}} \tvnorm{\Prob\brac{s_\iota,a_\iota|s_0=s',a_0=a',\pi_\omega}-D_{\pi_\omega} }e^{R_{\max}}
    \\& \myineq{\leq}{c}\max_{s,a}\varbrac{\lnorm{\partial_{\omega_i} \pi_{\omega}(a|s)}} \frac{m e^{R_{\max}}}{1-\rho}, 
} where $(a)$ follows from that $\lbrac{v(s,a)}\geq 1 $ for all $s\in \mathcal{S},a \in\mathcal{A}$, $(b)$ follows from \Cref{eq:boundf} and $(c)$ follows from Assumption \ref{asm:ergodic}.

Similar to \Cref{eq:186eq}, by Assumption \ref{asm:ergodic} and Assumption \ref{asm:smooth}, we can further bound $\mynorm{ \mathcal{K}^{(1) }_{\omega_i} }_v$, $\mynorm{ \mathcal{K}^{(2) }_{\omega_{ij}} }_v$ 
and $\mynorm{ \mathcal{K}^{(3) }_{\omega_{ijl} } }_v $ as follows: 
\eqenv{
    \mynorm{ \mathcal{K}^{(1) }_{\omega_i} }_v \leq \frac{ m C_\phi e^{R_{\max}}}{1-\rho};
      \mynorm{ \mathcal{K}^{(2) }_{\omega_{ij}} }_v
    \leq  \frac{m C_\delta e^{R_{\max}}}{1-\rho};
    \mynorm{ \mathcal{K}^{(3) }_{\omega_{ijl} } }_v \leq 
    \frac{m L_\delta e^{R_{\max}}}{1-\rho}.\label{eq:eq190}
}

Plug \Cref{eq:eq179a} and \Cref{eq:eq190} in \Cref{eq:eq188}, and we have that 
\eqenv{
  \mynorm{\partial_{\omega_i} \partial_{\omega_j} \partial_{\omega_l}D_{\pi_{\hat\omega}}}_v \leq \frac{6 C^3_\phi m^3 e^{4R_{\max}}}{(1-\rho)^3}+\frac{6 m^2 C_\phi C_\delta e^{3R_{\max}}}{(1-\rho)^2}+\frac{m L_\delta e^{2R_{\max}}}{1-\rho}.
}

    Thus, we can further get that
    \eqenv{
        \twonorm{ \nabla^2 J(\omega)-\nabla^2 J(\omega')}
        &\leq d^2 \brac{ \frac{6 C^3_\phi m^3 e^{4R_{\max}}}{(1-\rho)^3}+\frac{6 m^2 C_\phi C_\delta e^{3R_{\max}}}{(1-\rho)^2}+\frac{m L_\delta e^{2R_{\max}}}{1-\rho}} \twonorm{\omega-\omega'}.
    }
\end{proof}

\begin{proof}[Proof of Lemma \ref{lm:constract}]

    Consider the probability $\Prob \brac{\bar s_j,\bar a_j,\bar s_{t-k},\bar a_{t-k} }$
     and term 
    $ \E\Fbrac{\bar z_t\bar \delta_t(\bar s_t,\bar a_t;\theta,\omega_t)|\pi_t}$.
    We  have that 
    \eqenv{
    \E& \Fbrac{\sum_{j=t-k}^t \phi_t^\top(\bar s_j,\bar a_j) \brac{R(\bar s_t,\bar a_t)-J(\omega_t)+\phi^\top_t(\bar s\tit,\bar a\tit)\theta-\phi_t^\top(\bar s_t,\bar a_t)\theta} \Big |\pi_t}
    \\& = 
    \E \Fbrac{\sum_{j=t-k}^{t}\phi_t^\top(\bar s_j,\bar a_j)\bar \delta_t(\bar s_t,\bar a_t;\theta,\omega_t)\Big |\pi_t }
    \\& = \sum_{s,a} \Fbrac{ \sum_{j=t-k}^{t} \Prob\brac{\bar s_j,\bar a_j,\bar s_t,\bar a_t }\phi^\top_t(\bar s_j,\bar a_j)\delta_t(\bar s_t,\bar a_t;\theta,\omega_t)\Big |\pi_t}
    \\& = 
    \sum_{s,a} \Fbrac{ \sum_{j=t-k}^{t} \Prob\brac{\bar s_t,\bar a_t,\bar s_{2t-j},\bar a_{2t-j} }\phi^\top_t(\bar s_t,\bar a_t)\delta_t(\bar s_{2t-j},\bar a_{2t-j};\theta,\omega_t)\Big |\pi_t}
    \\& \myineq{=}{a} \sum_{s,a} \Fbrac{ \sum_{i=0}^{k} \Prob\brac{\bar s_t,\bar a_t,\bar s_{t+i},\bar a_{t+i}) }\phi^\top_t(\bar s_t,\bar a_t)\delta_t(\bar s_{t+i},\bar a_{t+i};\theta,\omega_t)\Big |\pi_t}
    \\& = 
    \E_{(\bar s_t,\bar a_t)\sim D_t}\Fbrac{ \phi^\top_t(\bar s_t,\bar a_t)\sum_{i=0}^{k} \delta_t(\bar s_{t+i},\bar a_{t+i};\theta,\omega_t)\Big |\pi_t }
    \\& = 
    \E_{ D_t}\Fbrac{ \phi^\top_t(s,a)\brac{\mathcal{T}^{(k)}\brac{\phi_t^\top(s,a) \theta} -\phi_t^\top(s,a) \theta}}
    \\& = H_{\omega_t}\theta+b_{\omega_t},
    } where $(a)$ follows from the fact that $\Prob\brac{\bar s_j,\bar a_j}=\Prob\brac{\bar s_{t},\bar a_{t} }\sim D_t $, and thus,
    \eqenv{
        \Prob\brac{\bar s_j,\bar a_j,\bar s_t,\bar a_t }
        &= \Prob\brac{\bar s_t,\bar a_t |\bar s_j,\bar a_j} \Prob\brac{\bar s_j,\bar a_j }
        \\& = \Prob\brac{\bar s_{2t-j},\bar a_{2t-j}|\bar s_t,\bar a_t }\Prob\brac{\bar s_t,\bar a_t }= \Prob\brac{\bar s_t,\bar a_t,\bar s_{2t-j},\bar a_{2t-j} }.
    }
\end{proof}

\begin{proof}[Proof of \Cref{lm:lm6}]
    Recall  $b_\omega=\E\Fbrac{\sum_{j=0}^{k}\phi^\top_\omega(s_0,a_0)(R(s_j,a_j)-J(\omega))|(s_0,a_0)\sim D_{\pi_\omega},\pi_\omega }$. Recall the definition of $H_\omega$ in \Cref{eq:aw}. Then, the solution to \Cref{eq:kstep} can be written as
\eqenv{
   \theta^*_\omega=- A^{-1}_\omega b_\omega.
}
First, $b_\omega$ can be bounded as follows:
\eqenv{\label{eq:b}
   \twonorm{b_\omega}&= \twonorm{\E\Fbrac{\sum_{j=0}^{k}\phi^\top_\omega(s_0,a_0)(R(s_j,a_j)-J(\omega))|(s_0,a_0)\sim D_{\pi_\omega} ,\pi_\omega }}
   \\& = \twonorm{{\sum_{j=0}^{k}\E\Fbrac{\phi^\top_\omega(s_0,a_0)(R(s_j,a_j)-J(\omega))|(s_0,a_0)\sim D_{\pi_\omega} ,\pi_\omega }}}
   \\& = \twonorm{{\sum_{j=0}^{k}\E\Fbrac{\phi^\top_\omega(s_0,a_0)\brac{R(s_j,a_j)-\E_{D_{\pi_\omega}}\Fbrac{R(s,a)}}|(s_0,a_0)\sim D_{\pi_\omega} ,\pi_\omega }}}
   \\& \myineq{\leq}{a} {\sum_{j=0}^{k}C_\phi R_{\max}\E \Fbrac{\tvnorm{D_{\pi_\omega}-\Prob\brac{s_j,a_j|s_0,a_0, \pi_\omega}}|(s_0,a_0)\sim D_{\pi_\omega}}}
   \\& \myineq{\leq}{b} C_\phi R_{\max} \sum_{j=0}^{k} m \rho^k
   \leq \frac{C_\phi R_{\max}m}{1-\rho},
} where $(a)$ follows from the triangular inequality and the fact that for any probability distribution $P_1 $ and $P_2$, and any random variable $X$, s.t. 
$\lbrac{X}\leq X_{\max}$,  $ \lbrac{\E_{P_1}\Fbrac{X}-\E_{P_2}\Fbrac{X}}\leq X_{\max} \tvnorm{ P_1-P_2}$, $(b)$ follows from Assumption \ref{asm:ergodic}.

From the following equation:
\eqenv{
  \theta^{*\top}_\omega H_\omega \theta^{*}_\omega
  = \theta^{*\top}_\omega b_\omega
  =\brac{\theta^{*\top}_\omega b_\omega}^\top
  = \theta^{*\top}_\omega A^\top_\omega \theta^{*}_\omega,
}
it holds that
\eqenv{\label{eq:eq194}
\lambda_{\max}\brac{\frac{H_\omega+H_\omega^\top}{2} }\twonormsq{\theta^*_\omega} \geq \theta^{*\top}_\omega \frac{H_\omega+A^\top_\omega }{2} \theta^{*}_\omega
  = \theta^{*\top}_\omega b_\omega
  \geq -\twonorm{ \theta_\omega^*}\twonorm{b_\omega}.
}
Thus, we can bound $\theta^*_\omega$ as follows:
\eqenv{
\twonorm{\theta^*_\omega }
& \myineq{\leq}{a} -\frac{1}{\lambda_{\max}\brac{\frac{H_\omega+H_\omega^\top}{2} }}
\twonorm{b_\omega}
 \myineq{\leq }{b}-\frac{1}{\lambda_{\max}\brac{\frac{H_\omega+H_\omega^\top}{2} }}
\frac{m C_\phi R_{\max}}{1-\rho}
\\& \myineq{\leq }{c}  \frac{1}{\lambda_{\min}-d C_\phi^2 m \rho^k }\frac{m C_\phi R_{\max}}{1-\rho}=\frac{m C_\phi R_{\max}}{\bar \lambda_{\min}\brac{1-\rho}},
}
where $(a)$ follows from \Cref{eq:eq194},  $(b)$ follows from \Cref{eq:b} and $(c)$ follows from \Cref{lm:eig}.
\end{proof}

\newpage
\section{Experiments}

In this section, we conduct experiments to numerically verify our AC/NAC with compatible function
approximation. We test our algorithms in the Acrobot environment \cite{sutton1995generalization}. The environment involves a two-link linear chain with one end anchored and a joint that can be actuated. The goal is to apply torques at this joint to swing the unanchored end of the chain to a certain height from an initial position of hanging down. We parameterize our policy using a neural network  and  use compatible function approximation in the  critic part. We compare the performance between our AC/NAC with compatible function approximation and the standard AC/NAC with linear function approximation.

We first compare the performance of AC algorithms. We run vanilla AC, $1$-step AC with compatible function approximation, and $k$-step AC with compatible function approximation (shortened as AC, 1-step CAC and $k$-step CAC); And then we compare three NAC algorithms: vanilla NAC, $1$-step NAC with compatible function approximation and $k$-step NAC with compatible function approximation (shortened as NAC, 1-step CNAC and $k$-step CNAC). 

In our experiment setup, we set $k=128$, and design a 2-layer neural network with 16 hidden neurons to represent the policy, which contains 163 parameters. We run the algorithms for 20 times. At each time step, after we obtain the policies from the algorithms, we evaluate them and plot the average reward in \Cref{Fig:Data1} and \Cref{Fig:Data2} among 20 runs. 
We also plot the 90 and 10 percentiles of the 20 curves as the upper and lower envelopes of the curves.

\begin{figure}[ht]
\vskip 0.2in
\begin{minipage}{0.48\textwidth}
     \centering
     \includegraphics[width=.9\linewidth]{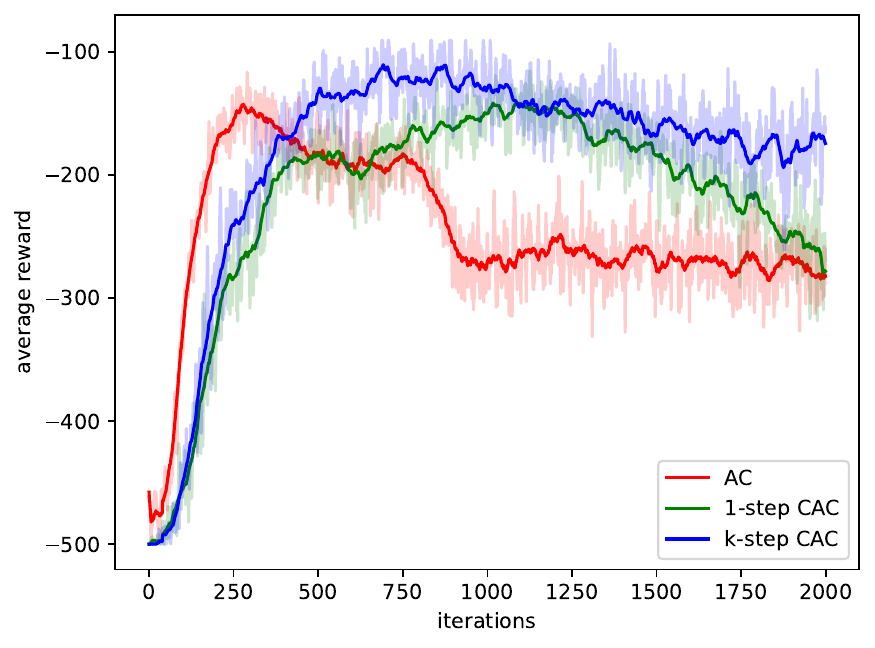}
     \caption{Vanilla AC with fixed feature function v.s. One-step AC with compatible feature function v.s. $128$-step AC with compatible feature function.  }\label{Fig:Data1}
   \end{minipage}\hfill
   \begin{minipage}{0.48\textwidth}
     \centering
     \includegraphics[width=.9\linewidth]{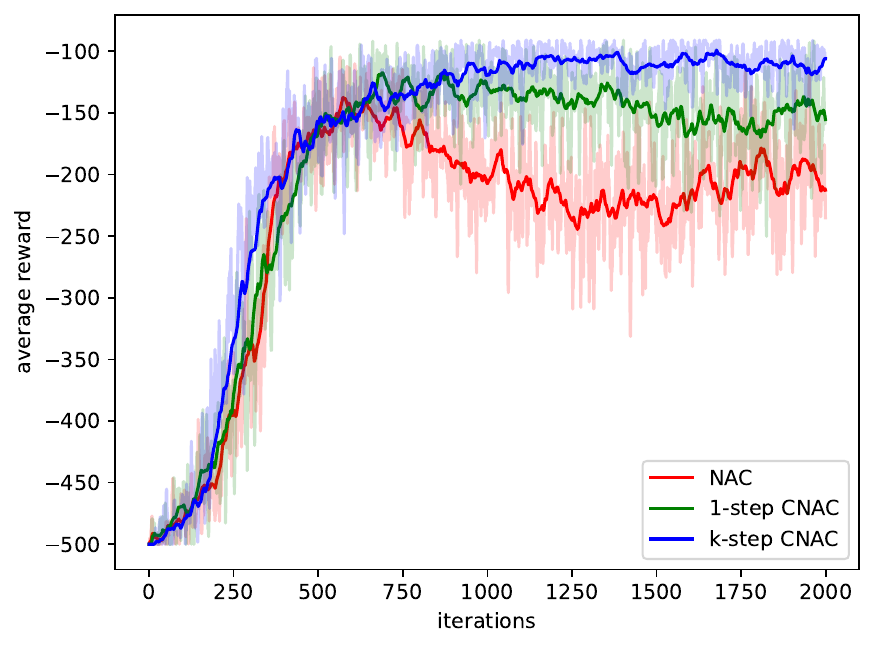}
     \caption{Vanilla NAC with fixed feature function v.s. One-step NAC with compatible feature function v.s. $128$-step NAC with compatible feature function.}\label{Fig:Data2}
   \end{minipage}
\end{figure}

We observe that for both AC and NAC, our algorithms with compatible function approximation lead to better performances. As our theoretical results showed, this performance improvement is due to the fact that the compatible function approximation can avoid critic error, whereas the linear function approximation results in an inaccurate estimation of the value functions in the critic part. 
\newpage
\section{Symbol Reference}
\begin{table}[!h]
    \begin{tabular}{l  l l}
         Constants& First Appearance \\
         $B=\frac{R_{\max}C_\phi}{(1-\rho)\brac{\lambda_{\min}-C_\phi^2 d m \rho^k }}$&  \Cref{sec:3.2}\\
         $c_\alpha=\frac{2L_J\lambda_{\min}+2B C^3_\phi\brac{C_\phi L_\pi+2L_\phi } }{\lambda^2_{\min} \bar \lambda_{\min}} $& \Cref{thm1}& expression in \Cref{cons:c}\\
         $c_\eta=\frac{ 2(k+1)^2B^2 C^6_\phi}{(\bar\lambda_{\min})^2} $& \Cref{thm1}& expression in \Cref{cons:c}\\
         $C_{\text{gap}}=C^2_\phi B  +\frac{ C_\phi R_{\max}}{1-\rho}$& \Cref{lm:cgap}\\
         $C_\infty$&Assumption \ref{asm:explore}\\
         $C_\pi$& Assumption \ref{asm:smooth}\\
         $C_\phi$& Assumption \ref{asm:smooth} \\
         $C_\delta$&  Assumption \ref{asm:smooth}\\
         $L_\phi$& Assumption \ref{asm:smooth} \\
         $L_\delta$&  Assumption \ref{asm:smooth}\\
         $\lambda_{\min}$&\Cref{sec:3.1}\\
         &&\\&&\\
         \end{tabular}

    \begin{tabular}{l  l l}
                 Constants& First Appearance\\  
         $C_J=   C^2_\phi \brac{B+ C_{\text{gap}}\frac{ m \rho^k}{\lambda_{\min}}}$& \Cref{Prop:lip}  \\
         $C_M= \frac{C_\Theta}{\lambda^2_{\min}}+ \frac{1}{2}B^2 $& \Cref{cons:cm}\\
         $C_\Theta=\frac{C_J}{\lambda^2_{\min}}
         \brac{2 C_\phi L_\phi+C^2_\phi L_\pi }+ \frac{L_J}{\lambda_{\min}}$&\Cref{lm:acnacinnersmooth}\qquad\qquad\qquad\qquad\qquad\quad\qquad\\
         $L_J= \frac{m R_{\max}}{1-\rho}\brac{4L_\pi C_\phi+L_\phi}$& \Cref{Prop:lip} \\
         $L_\pi=
\frac{1}{2}C_\pi 
\brac{1+\myceil{\log m^{-1}}+\frac{1}{1-\rho}}$& \Cref{Prop:lip} \\
         $L_\Theta=  d^2 \brac{ \frac{6 C^3_\phi m^3 e^{4R_{\max}}}{(1-\rho)^3}+\frac{6 m^2 C_\phi C_\delta e^{3R_{\max}}}{(1-\rho)^2}+\frac{m L_\delta e^{2R_{\max}}}{1-\rho}}$& \Cref{lm:acinnersmooth}\\
         $ \bar \lambda_{\min}=\lambda_{\min}-d C_\phi^2 m \rho^k$& \Cref{lm:actrackingerror}  \\
         $U_\delta =R_{\max}+2C_\phi B$& \Cref{lm:consterm}
         \\&&\\&&\\&&\\
    \end{tabular}
    \begin{tabular}{l  l l}
         Variable& Appearance &Order (set $\alpha_t\equiv \alpha$, $\beta_t\equiv \beta$,  $\gamma_t\equiv \gamma$)\\
         $G^\delta_t$&\Cref{lm:consterm} 
         & $\mathcal{O}(k^3 \alpha+k^3\beta+ k(m \rho^k))$\\
         $ G_t^\omega$& \Cref{lm:acgrad}&$\mathcal{O}\brac{ (m\rho^k)\beta+ k^2 \beta^2 } $\\
         $G_t^\eta $& \Cref{lm:aceta}&$\mathcal{O}\brac{(m\rho^k)\gamma +\beta^2+k^2\beta\gamma+k \gamma^2 } $\\
         $G_t^\theta$& \Cref{lm:actrackingerror} &$\mathcal{O}\brac{k (m\rho^k)\alpha+(m\rho^k)\beta+\frac{(m \rho^k)^2}{\beta}+k^3\alpha^2+k^3 \alpha\beta+k^2\beta^2 } $\\
         $\widetilde G^\eta_t$& \Cref{lm:naceta}&$\mathcal{O}\brac{(m\rho^k)\gamma +\beta^2+k^2\beta\gamma+k \gamma^2 } $\\
         $\widetilde G_t^\theta$& \Cref{lm:nactheta} &$\mathcal{O}\brac{k (m\rho^k)\alpha+ \frac{(m \rho^k)^2}{\beta}+k^3\alpha^2+k^3 \alpha\beta+ \beta^2 } $\\
         $\widetilde G_t^\omega$&\Cref{vari:gomega}& $\mathcal{O}\brac{(m\rho^k)\beta+\beta^2)}$\\
         $q $&\Cref{vari:q}&$\mathcal{O}\brac{\alpha } $\\
         $\hat t=\myceil{ \frac{1}{q}\log T}$&\Cref{vari:that}&$\mathcal{O}\brac{\frac{\log T}{\alpha} } $\\
         $\widetilde T={\myceil{\frac{T}{\hat t\log T}}}\hat t$&\Cref{eq:nac:d}&$\mathcal{O}\brac{\frac{T}{\log T} } $\\
         $M_t  $&\Cref{cons:cm}&-\\
    \end{tabular}
    \label{tab:variapp}
\end{table}

\end{document}